\documentclass[twoside,11pt]{article}

\usepackage{jmlr2e}
\usepackage{soul}
\usepackage{hyperref}
\usepackage{float}
\usepackage{url}
\usepackage{caption}
\usepackage{subcaption}
\usepackage{xcolor}
\usepackage{comment}
\usepackage{amsmath}
\usepackage{physics}
\usepackage{multirow}
\usepackage{apptools}
\usepackage{lastpage}
\AtAppendix{\counterwithin{theorem}{section}}

\usepackage{amsmath,amsfonts,bm}

\def\eqref#1{equation~\ref{#1}}

\def\1{\mathbb{I}}

\DeclareMathAlphabet{\mathsfit}{\encodingdefault}{\sfdefault}{m}{sl}
\SetMathAlphabet{\mathsfit}{bold}{\encodingdefault}{\sfdefault}{bx}{n}

\newcommand{\E}{\mathbb{E}}

\newcommand{\R}{\mathbb{R}}

\newcommand{\Var}{\mathrm{Var}}

\newcommand{\Cov}{\mathrm{Cov}}
\newcommand{\Span}{\mathrm{Span}}
\renewcommand{\mod}{\,\mathrm{mod}\,}

\DeclareMathOperator*{\argmax}{arg\,max}
\DeclareMathOperator*{\argmin}{arg\,min}

\DeclareMathOperator{\sign}{sign}
\DeclareMathOperator{\diag}{diag}
\DeclareMathOperator{\unif}{Uniform}

\newcommand{\wl}[1]{\textcolor{orange}{(Wenlong: #1)}}

\newcommand{\lm}[1]{\textcolor{blue}{(Ryumei: #1)}}

\newcommand{\AEN}{\text{AE}}
\newcommand{\DE}{\text{DE}}
\newcommand{\CL}{\text{CL}}
\newcommand{\SCL}{\text{SCL}}
\newcommand{\SVD}{\text{SVD}}

\newcommand{\DAE}{\text{DAE}}
\def\eqref#1{(\ref{#1})}

\newcommand{\dataset}{{\cal D}}
\newcommand{\fracpartial}[2]{\frac{\partial #1}{\partial  #2}}

\jmlrheading{24}{2023}{1-\pageref{LastPage}}{12/21; Revised
11/23}{12/23}{21-1501}{Wenlong Ji, Zhun Deng, Ryumei Nakada, James Zou, Linjun Zhang}

\ShortHeadings{The Power of Contrast for Feature Learning: A Theoretical Analysis}{Ji, Deng, Nakada, Zou, and Zhang}
\firstpageno{1}

\begin{document}

\title{The Power of Contrast for Feature Learning:\\A Theoretical Analysis}

\author{\name Wenlong Ji \email jwl2000@stanford.edu \\
       \addr Department of Statistics\\
       Stanford University\\
       Stanford, CA 94305, USA
       \AND
       \name Zhun Deng \email zhundeng@g.harvard.edu \\
       \addr Department of Computer Science
\\
       Columbia University\\
       New York, NY 10027, USA
       \AND
       \name Ryumei Nakada \email ryumei.n@rutgers.edu \\
       \addr Department of Statistics\\
       Rutgers University\\
       Piscataway, NJ 08854, USA
       \AND
       \name James Zou \email jamesz@stanford.edu \\
       \addr
       Department of Biomedical Data Science\\
       Stanford University\\
       Stanford, CA 94305, USA
       \AND
       \name Linjun Zhang \email linjun.zhang@rutgers.edu \\
       \addr Department of Statistics\\
       Rutgers University\\
       Piscataway, NJ 08854, USA}

\editor{Aapo Hyvarinen}

\maketitle

\begin{abstract}%
Contrastive learning has achieved state-of-the-art performance in various self-supervised learning tasks and even outperforms its supervised counterpart. Despite its empirical success, theoretical understanding of the superiority of contrastive learning is still limited. In this paper, under linear representation settings, 
(i) we provably show that contrastive learning outperforms the standard autoencoders and generative adversarial networks, two classical generative unsupervised learning methods, for both feature recovery and in-domain downstream tasks; (ii) we also illustrate the impact of labeled data in supervised contrastive learning. This provides theoretical support for recent findings that contrastive learning with labels improves the performance of learned representations in the in-domain downstream task,  but it can harm the performance in transfer learning. We verify our theory with numerical experiments.

\end{abstract}


\section{Introduction}

Deep supervised learning has achieved great success in various applications, including computer vision \citep{krizhevsky2012imagenet}, natural language processing \citep{vaswani2017attention}, and scientific computing \citep{han2018solving}. However, its dependence on manually assigned labels, which is usually difficult and costly, has motivated research into alternative approaches to exploit unlabeled data. Self-supervised learning is a promising approach that leverages the unlabeled data itself as supervision and learns representations that are beneficial to potential in-domain downstream tasks. 


At a high level, there are two common approaches for feature extraction in self-supervised learning: generative and contrastive \citep{liu2021self,jaiswal2021survey}.
Both approaches aim to learn latent representations of the original data, while the difference is that the generative approach focused on minimizing the reconstruction error from latent representations, and the contrastive approach targets to decrease the similarity between the representations of contrastive pairs constructed by data augmentation.
Recent works have shown the benefits of contrastive learning in practice \citep{chen2020simple,he2020momentum,chen2020big,chen2020improved}.  However, these works did not explain the popularity of contrastive learning --- \textit{what is the advantage of contrastive learning and where does it come from?} 


Additionally, recent works aim to further improve contrastive learning by introducing label information. Specifically, \citet{khosla2020supervised} proposed the {supervised contrastive learning}, where the contrasting procedures are performed across different classes rather than different instances. With the help of label information, their proposed method outperforms self-supervised contrastive learning and classical cross-entropy-based supervised learning. 
However, despite this improvement in in-domain downstream tasks, \citet{Islam2021ABS} found that such improvement in transfer learning is limited and even negative for such supervised contrastive learning. This phenomenon motivates us to rethink \textit{the impact of labeled data} in the contrastive learning framework. 

In this paper, we first establish a theoretical framework to study contrastive learning under the linear representation setting. 
Under this framework, we provide a theoretical analysis of the feature learning performance of the contrastive learning on the spiked covariance model \citep{bai2012sample,yao2015sample,zhang2018heteroskedastic} and theoretically justify why contrastive learning outperforms standard autoencoders and generative adversarial networks (GANs) \citep{goodfellow2014generative} ---contrastive learning is able to remove more noise by constructing contrastive samples via augmentations. 
Moreover, we investigate the impact of label information in the contrastive learning framework and provide a theoretical justification of why labeled data help to gain accuracy in in-domain regression and classification while can hurt multi-task transfer learning.

\subsection{Related Works}
The idea of contrastive learning was firstly proposed in \citet{hadsell2006dimensionality} as an effective method to perform dimensional reduction. 
Following this line of research, \citet{dosovitskiy2014discriminative} proposed to perform instance discrimination by creating surrogate classes for each instance and \citet{wu2018unsupervised} further proposed to preserve a memory bank as a dictionary of negative samples. Other extensions based on this memory bank approach include \citet{he2020momentum, misra2020self, tian2020contrastive, chen2020improved}. Rather than keeping a costly memory bank, another line of work exploits the benefit of mini-batch training where different samples are treated as negative to each other \citep{ye2019unsupervised, chen2020simple}. Moreover, \citet{khosla2020supervised} explores the supervised version of contrastive learning where pairs are generated based on label information.

Despite its success in practice, the theoretical understanding of contrastive learning is still limited. Previous works provide provable guarantees for contrastive learning under conditional independence assumption (or its variants) \citep{arora2019theoretical,lee2021predicting,tosh2021contrastive,tsai2020self}. Specifically, they assume the two contrastive views are independent conditioned on the label and show that contrastive learning can provably learn representations beneficial for in-domain downstream tasks. 
In addition to this line of research, there exist several alternative perspectives for studying the theoretical properties of contrastive learning. To name a few, \cite{wang2020understanding,graf2021dissecting} explored the representation geometry, \cite{haochen2021provable} analyzed the augmentation graph, \cite{tian2022deep} proposed a two-player game theory framework, \cite{zimmermann2021contrastive} demonstrated the connection between contrastive learning and nonlinear Independent Component Analysis \citep{hyvarinen2009independent}, \cite{saunshi2022understanding} showed that the importance of inductive bias in contrastive learning, and \cite{jing2021understanding} investigated the dimensional collapse phenomenon. Furthermore, \citet{tian2021understanding,wang2021towards} have also explored the ability of self-supervised learning to learn features even without contrastive pairs, specifically in the context of linear representation settings.

More relevant to this paper, \citet{wen2021toward} considered representation learning under the sparse coding model and studied the optimization properties in shallow ReLU neural networks. However, the assumptions that features are extremely sparse and signals follow Gaussian distribution seem strong for real data. \citet{garg2020functional} studied the combination of supervised learning and self-supervised learning. They derived sample complexity bounds in a PAC-learning style for various settings.
Specifically, the authors assume that there is a ground-truth representation such that it can keep both self-supervised loss and supervised loss at a very low threshold. 
However, as the authors admit, it is hard to determine such a threshold in practical settings. For example, since the unlabeled data and labeled data come from different domains, such as Image-Net and CIFAR-10, domain-specific features may have a much lower loss compared with domain-transferable features. 

While the aforementioned previous works aim to demonstrate that contrastive learning is capable of learning meaningful representations, it was left untouched why contrastive learning \textit{outperforms} other representation learning methods. We also shed light on the impact of labeled data in a contrastive learning framework, which is underexplored in prior works. A detailed comparison with existing literature is deferred to Appendix \ref{section: comparison}.

\subsection{Outline}
This paper is organized as follows. Section \ref{sec: prelim} provides the setup for the data-generating process and the loss function.
In Section \ref{sec: unsupervised}, we review the connection between PCA and autoencoders/GANs. We also establish a theoretical framework to study contrastive learning in the linear representation setting. 
Under this framework, we evaluate the feature recovery performance and in-domain downstream task performance of contrastive learning and autoencoders.
In Section \ref{sec: labeled data}, we analyze the supervised contrastive learning.
In Section \ref{sec: simulation}, we verify our theoretical results given in Sections \ref{sec: unsupervised} and \ref{sec: labeled data}. 
Finally, we summarize our analysis and provide future directions in Section \ref{sec: conclusion}.

\subsection{Notations} In this paper, we use $O,\Omega,\Theta$ to hide universal constants and we write $a_k\lesssim b_k$ for two sequences of positive numbers $\{a_k\} $ and $\{b_k\}$ if and only if there exists a universal constant $C>0$ such that $a_k<Cb_k$ for any $k$. We write $a_k \asymp b_k$ when $a_k \lesssim b_k$ and $a_k \gtrsim b_k$ holds simultaneously. We use $\|\cdot\|,\|\cdot\|_2,\|\cdot\|_F$ to represent the $\ell_2$ norm of vectors, the spectral norm of matrices, and Frobenius norm of matrices respectively. 
Let $\mathbb{O}_{d,r}$ be a set of $d\times r$ orthogonal matrices. Namely, $\mathbb{O}_{d,r} \triangleq \{U \in \mathbb{R}^{d\times r} : U^\top U = I_r\}$.
We write $n \gg d$ when there exists a sufficiently small constant $c$ depending on the constant and independent of $n$, $d$ and $r$ such that $d/n < c$ holds. $d \gg r$ is defined similarly.
We use $|A|$ to denote the cardinality of a set $A$. For any $n\in\mathbb{N}^+$, let $[n]=\{1,2,\cdots,n\}$. We use $\|\sin\Theta(U_1,U_2)\|_F$ to refer to the sine distance between two orthogonal matrices $U_1,U_2\in\mathbb{O}_{d, r}$, which is defined by: $
    \left\|\sin \Theta\left(U_{1}, U_{2}\right)\right\|_F \triangleq\left\|U_{1 \perp}^\top U_{2}\right\|_F$,
where $U_{1\perp} \in \mathbb{O}_{d-r,r}$ is any orthogonal complement of $U_1$. 
More properties of sine distance can be found in Section \ref{sec: distance}. We use $\{e_i\}_{i=1}^d$ to denote the canonical basis in $d$-dimensional Euclidean space $\mathbb{R}^d$, that is, $e_i$ is the vector whose $i$-th coordinate is $1$ and all the other coordinates are $0$.
Let $\1\{A\}$ be an indicator function that takes $1$ when $A$ is true, otherwise takes $0$.
We write $a \vee b$ and $a \wedge b$ to denote $\max(a, b)$ and $\min(a, b)$, respectively.

\section{Setup}\label{sec: prelim}
Here we introduce loss and data-generative models that will be used for the theoretical analysis later.

\subsection{Linear Representation Settings for Contrastive Learning}
\label{sec: linear representation}
Given an input $x\in \mathbb{R}^d$, contrastive learning aims to learn a low-dimensional representation $h=f(x;\theta)\in\mathbb{R}^r$ by contrasting different samples, that is, maximizing the agreement between positive pairs, and minimizing the agreement between negative pairs. Suppose we have $n$ data points $X=[x_1,x_2,\cdots,x_n]\in\mathbb{R}^{d\times n}$ from the population distribution $\mathcal{D}$. The contrastive learning task can be formulated to 
the following optimization problem:
\begin{equation}
	\label{contrastive task}
	\min_{\theta}\mathcal{L}(\theta)=\min_{\theta}\frac{1}{n}\sum_{i=1}^n\ell(x_i,\mathcal{B}_i^{Pos},\mathcal{B}_i^{Neg};f(\cdot; \theta))+\lambda R(\theta),
\end{equation}
where $\ell(\cdot)$ is a contrastive loss and $\lambda R(\theta)$ is a regularization term; $\mathcal{B}_i^{Pos},\mathcal{B}_i^{Neg}$ are the sets of positive samples and negative samples corresponding to $x_i$, the details of which are described below.

\paragraph{Linear Representation and Regularization Term}
We consider the linear representation function $f(x;W)=Wx$, where the parameter $\theta$ is a matrix $W\in\mathbb{R}^{r\times d}$. This linear representation setting has been widely adopted in other theory papers to understand self-supervised contrastive learning \citep{jing2021understanding,wang2021towards,tian2021understanding} and shed light upon other complex machine learning phenomena such as in \citet{tripuraneni2021provable}. Moreover, since regularization techniques have been widely adopted in contrastive learning practice \citep{chen2020simple,he2020momentum,grill2020bootstrap}, we further consider penalizing the representation by a regularization term $R(W)=\|WW^\top\|_F^2/2$ to encourage the orthogonality of $W$ and therefore  
promote the diversity of $w_i$ to learn different representations. The reason we use such quadratic regularization instead of a standard $\ell_2$ regularization is to encourage a diverse representation in the linear representation setting by penalizing on the similarity $\langle w_i, w_j\rangle^2$, we defer a formal discussion and numerical experiments about this regularization in the Appendix \ref{section: regularization}.


\paragraph{Linear Contrastive Loss}
The contrastive loss is set to be the average similarity (measured by the inner product) between positive pairs minus that between negative pairs: 
\begin{equation}
\label{triplet loss}
	\ell(x,\mathcal{B}_x^{Pos},\mathcal{B}_x^{Neg},f(\cdot; \theta))=-\sum_{x^{Pos}\in\mathcal{B}_x^{Pos}}\frac{\langle f(x,\theta),f(x^{Pos},\theta) \rangle}{|\mathcal{B}_x^{Pos}|}+\sum_{x^{Neg}\in\mathcal{B}_x^{Neg}}\frac{\langle f(x,\theta),f(x^{Neg},\theta) \rangle}{|\mathcal{B}_x^{Neg}|},
\end{equation}
where $\mathcal{B}_x^{Pos},\mathcal{B}_x^{Neg}$ are sets of positive samples and negative samples corresponding to $x$. 
This loss function has been commonly used in contrastive learning \citep{hadsell2006dimensionality} and metric learning \citep{Schroff2015FaceNetAU,He2018TripletCenterLF}. In \citet{khosla2020supervised}, the authors show that the inner-product based linear loss \eqref{triplet loss} is an approximation of the NT-Xent contrastive loss when one positive and one negative are used, which has been highlighted in recent contrastive learning practice \citep{Sohn2016ImprovedDM,wu2018unsupervised,oord2018representation,chen2020simple}. In \citet{li2021self}, the authors proposed the SSL-HSIC contrastive loss, which can be reduced to this linear loss when the kernel $k(\cdot,\cdot)$ is chosen to be a simple inner product. Following \citet{li2021self}, we provide the results in Table \ref{tab: simclr loss}, which shows that linear contrastive loss can also work well with some additional training techniques.
\begin{table}[H]
    \centering
    \begin{tabular}{ l||ccc  }
        \hline
        Testing Accuracy & InfoNCE & Linear contrastive loss\\
        \hline
        \hline
        CIFAR10    & $65.11\pm 0.51$ & $\mathbf{66.07\pm 0.46}$ \\
        \hline
        STL10   & $\mathbf{71.02\pm 0.47}$ & $70.30\pm 0.31$ \\
        \hline
    \end{tabular}
    \caption{\textbf{InfoNCE loss v.s. Linear contrastive loss}. We train a ResNet-18 encoder on CIFAR-10 and STL-10 datasets with different contrastive loss functions. To train the linear contrastive loss, we follow the HSIC regularization techniques used in \citet{li2021self}, which helps linear contrastive loss yield comparable performance to standard InfoNCE. We repeat each experiment for ten runs and report the mean and standard deviation of accuracy. Detailed experimental settings can be found in Section \ref{sec: real data}.}
    \label{tab: simclr loss}
\end{table}

\subsection{Generation of Positive and Negative Pairs}
There are two common approaches to generating positive and negative pairs, depending on whether or not label information is available. 
When the label information is not available, the typical strategy is to generate different views of the original data via augmentation \citep{hadsell2006dimensionality,chen2020simple}. Two views of the same data point serve as the positive pair for each other, while those of different data serve as negative pairs.

\begin{definition}[Augmented Pairs Generation in the Self-supervised Setting]
	\label{pair: augmented}
	Given two augmentation functions $g_1,g_2:\mathbb{R}^d\rightarrow\mathbb{R}^d$ and $n$ training samples $\mathcal{B}=\{x_i\}_{i\in[n]}$, the augmented views are given by: $\{(g_1(x_i)$, $g_2(x_i))\}_{i\in [n]}$. 
	Then for each view $g_v(x_i)$, $v=1,2$, the corresponding positive samples and negative samples are defined by: $\mathcal{B}_{i,v}^{Pos}=\{g_s(x_i):s\in[2]\setminus \{v\}\}$ and $\mathcal{B}_{i,v}^{Neg} = \{g_{s}(x_j):s\in [2],j\in [n]\setminus \{i\}\}$.
\end{definition}
The loss function of the self-supervised contrastive learning problem can then be written as:
\begin{small}
\begin{equation}
\label{loss: self contrastive}
    \mathcal{L}_{\text{SelfCon}}(W)\!=-\!\frac{1}{2n}\sum_{i=1}^n\sum_{v=1}^2\biggl[\langle Wg_v(x_i),Wg_{[2]\setminus \{v\}}(x_i) \rangle\!-\!\sum_{j\neq i}\sum_{s=1}^2\frac{\langle Wg_v(x_i),Wg_s(x_j) \rangle}{2n-2}\biggr]\!+\!\frac{\lambda}{2}\|WW^\top\|_F^2.
\end{equation}
\end{small}

In particular, we adopt the following augmentation in our analysis.
\begin{definition}[Random Masking Augmentation]
	\label{aug: random masking}
	The two views of the original data are generated by randomly dividing its dimensions into two sets, that is,
		$g_1(x_i)=Ax_i,\text{~and~} g_2(x_i) = (I-A)x_i$, 
	where $A=\diag(a_1,\cdots,a_d)\in\mathbb{R}^{d\times d}$ is the diagonal masking matrix with $\{a_i\}_{i=1}^{d}$ being $i.i.d.$ random variables sampled from a Bernoulli distribution with mean $1/2$. 
\end{definition}
\begin{remark}
    In this paper, we focus on random masking augmentation, which has also been used in other works on the theoretical understanding of contrastive learning, eg. \citet{wen2021toward}. However, our primary interest lies in comparing the performance of contrastive learning with autoencoders and analyzing the impact of labeled data, while their work focuses on understanding the training process of neural networks in contrastive learning. Random masking augmentation is an analog of the random cropping augmentation used in practice. As shown in \citet{chen2020simple}, cropping augmentation achieves overwhelming performance on linear evaluation (ImageNet top-1 accuracy) compared with other augmentation methods, please see Figure 5 in \citet{chen2020simple} for details. 
\end{remark}

When the label information is available, \citet{khosla2020supervised} proposed the following approach to generate positive and negative pairs.
\begin{definition}[Pairs Generation in the Supervised Setting]
	\label{pair: supervised}
	In a $K$-class classification problem, given $n_k$ samples for each class $k\in[K]$: $\{x_i^k:i\in[n_k]\}_{k=1}^K$ and let $n=\sum_{k=1}^K n_k$, the corresponding positive samples and negative samples for $x_i^k$ are defined by $\mathcal{B}_{i,k}^{Pos}=\{x_j^k: j\in [n_k]\setminus i\}$ and $\mathcal{B}_{i,k}^{Neg} = \{x_j^s: s\in[K]\setminus k, j\in [n_s]\}$. That is, the positive samples are the remaining ones in the same class with $x^k_i$ and the negative samples are the samples from different classes.
\end{definition}
Correspondingly, the loss function of the supervised contrastive learning problem can be written as:
\begin{equation}
    \mathcal{L}_{\text{SupCon}}(W)=-\frac{1}{nK}\sum_{k=1}^K\sum_{i=1}^{n}\biggl[\sum_{j\neq i}\frac{\langle Wx_i^k,Wx_j^k\rangle}{n-1}-\sum_{j=1}^n\sum_{s\neq k}\frac{\langle Wx_i^k,Wx_j^s \rangle}{n(K-1)}\biggr]+\frac{\lambda}{2}\|WW^\top\|_F^2.
\end{equation}

\subsection{Data Generating Process}
In real-world scenarios, data often comprises both signal (relevant information) and noise (irrelevant distractions). For instance, in image classification, the signal might be the primary subject of interest, while the noise could represent background elements. Self-supervised learning methods, without predefined tasks, aim to extract generalized patterns from data, ideally capturing as much of the signal as possible. It is commonly understood that signals tend to exhibit specific low-complexity structures, often being low-rank and showing higher correlations across coordinates. In contrast, background noise might lack a distinct structure, potentially being dense (or full rank) with lower coordinate correlations. To delve into this structural difference more rigorously, we consider an additive data-generating model. Here, the observed data emerges as a combination of a low-rank signal and dense noise.
\begin{equation}
	\label{model: spiked covariance}
	x = U^\star z+\xi,\quad \Cov(z)=\nu^2I_r,\quad \Cov(\xi)=\Sigma,
\end{equation}
where $z\in\mathbb{R}^{r}$ and $\xi\in\mathbb{R}^{d}$ are both zero mean sub-Gaussian independent random variables, and $\nu\in\R$ is a constant represents the signal strength.
In particular, $U^\star \in \mathbb{O}_{d, r}$ and $\Sigma = \diag(\sigma_1^2,\cdots,\sigma_d^2)$.
The first term $U^\star z$ represents the signal of interest residing in a low-dimensional subspace spanned by the columns of $U^\star$. The second term $\xi$ is the dense noise with heteroskedastic noise. Given that, the ideal low-dimensional representation is to compress the observed $x$ into a low-dimensional representation spanned by the columns of $U^\star$. This model is known as the spiked covariance model \citep{johnstone2001distribution,bai2012sample,yao2015sample,zhang2018heteroskedastic}. It was proposed from the empirical observation that the eigenvalues of the sample covariance matrix of phoneme data have few "spikes", which corresponds to the low-dimensional structure of data generation. The model has been used in the literature of PCA \citep{johnstone2001distribution,deshpande2014information,zhang2018heteroskedastic} and Contrastive Learning \citep{wen2021toward}.


In this paper, we aim to \textit{learn a good projection $W \in \R^{r\times d}$} onto a lower-dimensional subspace from the observation $x$. 
Since the information of $W$ is invariant with the transformation $W \leftarrow O W$ for any $O \in \mathbb{O}_{r,r}$, the essential information of $W$ is contained in the right eigenvector of $W$. Thus, we quantify the goodness of the representation $W$ using the sine distance $\|\sin\Theta(U,U^\star)\|_F$, where $U$ is the top-$r$ right eigenspace of $W$. It is notable that we only assume that noise and signal follow a sub-Gaussian distribution. This includes bounded noise/signals such as images, sound data, or text data.



\section{Comparison of Self-Supervised Contrastive Learning and Autoencoders/GANs}\label{sec: unsupervised}
Generative and contrastive learning are two popular approaches of self-supervised learning. Recent experiments have highlighted the improved performance of contrastive learning compared with the generative approach. For example, in Figure 1 of \citet{chen2020simple} and Figure 7 of \citet{liu2021self}, it is observed that state-of-the-art contrastive self-supervised learning has more than 10 percent improvement over state-of-the-art generative self-supervised learning, with the same number of parameters. In this section, we rigorously demonstrate the advantage of contrastive learning over autoencoders/GANs, the representative methods in generative self-supervised learning, by investigating the linear representation settings under the spiked covariance model \eqref{model: spiked covariance}. The investigation is conducted for both feature recovery and in-domain downstream tasks.


Hereafter, we focus on the linear representation settings. This section is organized as follows: in Section \ref{sec: PCA and autoencoder} we first review the connection between principal component analysis (PCA) and autoencoders/GANs, which are two representative methods in generative approaches in self-supervised learning, under linear representation settings. Then we establish the connection between contrastive learning and PCA in Section \ref{sec: connection to cl}. 
Based on these connections, we make the comparison between contrastive learning and autoencoder on feature recovery ability (Section \ref{sec: recover}) and in-domain downstream performance (Section \ref{sec: downstream}). 

\subsection{Autoencoders, GANs and PCA}\label{sec: connection to autoencoders}

\label{sec: PCA and autoencoder}
Autoencoders are popular unsupervised learning methods to perform dimensional reduction.
Autoencoders learn two functions: encoder $f:\mathbb{R}^d\rightarrow\mathbb{R}^r$ and decoder $g:\mathbb{R}^r\rightarrow\mathbb{R}^d$. While the encoder $f$ compresses the original data into low-dimensional features, and the decoder $g$ recovers the original data from those features. It can be formulated to be the following optimization problem for samples $\{x_i\}_{i=1}^n$ \citep{ballard1987modular,fan2019selective}:
\begin{equation}
\label{autoencoder}
    \min_{f, g}\mathbb{E}_x\mathcal{L}(x,g(f(x))).
\end{equation}
By minimizing this loss, autoencoders try to preserve the essential features to recover the original data in the low-dimensional representation. 
In our setting, we consider the class of linear functions for $f$ and $g$. The loss function is set as the mean squared error. Write $f(x)=W_{\AEN}x$ and $g(x)=W_{\DE}x$. Namely, we consider the following problem.
\begin{equation*}
    \min_{W_{\AEN}, W_{\DE}}\frac{1}{n}\|X-W_{\DE}W_{\AEN}X\|_F^2.
\end{equation*}
Let $X = (x_1, \dots, x_n) \in \R^{d \times n}$. By Theorem 2.4.8 in \citet{van1996matrix}, the optimal solution is given by the eigenspace of $X X^\top$, which exactly corresponds to the result of PCA. Thus, in linear representation settings, autoencoders are equivalent to PCA, which is also often known as \textit{undercomplete linear autoencoders} \citep{bourlard1988auto,plaut2018principal,fan2019selective}. 
We write the obtained low-rank representation by autoencoders as
\begin{equation}
\label{opt: AE}
    W_{\AEN}=(U_{\AEN}\Sigma_{\AEN}V_{\AEN}^\top)^\top,
\end{equation}
where $U_{\AEN}$ is the top-$r$ eigenvectors of matrix $X X^\top$, $\Sigma_{\AEN}$ is a diagonal matrix of spectral values and $V_{\AEN}=[v_1,\cdots,v_r]\in\mathbb{R}^{r\times r}$ can be any orthonormal matrix.

We also note that GANs \citep{goodfellow2014generative} is related to PCA. Namely, \citet{feizi2020understanding} showed that the global solution for GANs recovers the empirical PCA solution as the generative model. 

To see this, let $\mathcal{W}_2$ be the second-order Wasserstein distance. Also let $\mathcal{G}$ be the set of linear generator functions from $\R^r \to \R^d$.
Consider the following $\mathcal{W}_2$ GAN optimization problem:
\begin{align}\label{eq: gan}
    \min_{g \in \mathcal{G}} \mathcal{W}_2^2(\mathbb{P}_{n}, \mathbb{P}_{g(Z)}),
\end{align}
where $\mathbb{P}_{n}$ denotes the empirical distribution of i.i.d. data $x_1, \dots, x_n \in \R^d$ and $\mathbb{P}_{g(Z)}$ is the generated distribution with generator $g$ and $Z \sim N(0, I_r)$.
Note that the optimization problem Equation \eqref{eq: gan} can be written as $\min_{\mathbb{P}_{n, Z}} \min_{g \in \mathcal{G}} \mathbb{E}[\|X - g(Z)\|^2]$, where the first minimization is over probability distributions which have marginals $\mathbb{P}_n$ and $\mathbb{P}_Z$.
By Theorem 2 in \citet{feizi2020understanding}, the optimizer of problem Equation \eqref{eq: gan} is obtained as $\hat g: Z \mapsto \hat G Z$, where $\hat G$ satisfies $\hat G \hat G^\top = U_{\AEN}\Sigma_{\AEN}^2 U_{\AEN}^\top$. This implies $W_{\AEN}^\top: \R^r \to \R^d$ is also a solution to the optimization problem Equation \eqref{eq: gan}. 
Hence GANs learn the PCA solution as a generator.

From this equivalence among ordinary PCA, autoencoders, and GANs, we \textbf{only focus on autoencoders hereafter} for brevity.

\subsection{Contrastive Learning and Diagonal-Deletion PCA}\label{sec: connection to cl}

Here we bridge PCA and contrastive learning with certain augmentations under the linear representation setting.
Recall that the optimization problem for self-supervised contrastive learning is formulated as:
\begin{small}
\begin{equation}
\label{opt: CL ap}
  \min_{W\in\mathbb{R}^{r\times d}}\mathcal{L}_{\text{SelfCon}}(W)\!:=-\!\frac{1}{2n}\sum_{i=1}^n\sum_{v=1}^2\biggl[\langle Wg_v(x_i),Wg_{[2]\setminus \{v\}}(x_i) \rangle\!-\!\sum_{j\neq i}\sum_{s=1}^2\frac{\langle Wg_v(x_i),Wg_s(x_j) \rangle}{2n-2}\biggr]\!+\!\frac{\lambda}{2}\|WW^\top\|_F^2.
\end{equation}
\end{small}

To compare contrastive learning with autoencoders, we now derive the solution of the optimization problem \eqref{opt: CL ap}. We start with the general result for self-supervised contrastive learning with augmented pairs generation in Definition \ref{pair: augmented}, and then turn to the special case of random masking augmentation (Definition \ref{aug: random masking}).

\begin{proposition}
	\label{prop: augment}
	For two fixed augmentation functions $g_1,g_2:\mathbb{R}^d\rightarrow\mathbb{R}^d$, denote the augmented data matrices as $X_1=[g_1(x_1),\cdots,g_1(x_n)]\in\mathbb{R}^{d\times n}$ and $X_2=[g_2(x_1),\cdots,g_2(x_n)]\in\mathbb{R}^{d\times n}$, when the augmented pairs are generated as in Definition \ref{pair: augmented}, all the optimal solutions of contrastive learning problem \eqref{opt: CL ap} are given by:
	\begin{equation*}
		W_{\CL} = C\left(\sum_{i=1}^{r}u_i\sigma_i v_i^\top\right)^\top,
	\end{equation*}
	where $C>0$ is a positive constant, $\sigma_i$ is the $i$-th largest eigenvalue of the following matrix:
	\begin{equation}
		X_1X_2^\top+X_2X_1^\top-\frac{1}{2(n-1)}(X_1+X_2)(1_n 1_n^\top-I_n)(X_1+X_2)^\top,
	\end{equation}
	$u_i$ is the corresponding eigenvector and $V=[v_1,\cdots,v_r]\in\mathbb{R}^{r\times r}$ can be any orthonormal matrix.
\end{proposition}

The proof is given in Appendix \ref{prop: augment ap}.

{Proposition \ref{prop: augment} is a general result for augmented pairs generation with fixed and deterministic augmentation functions. The result itself only depends on the augmented data matrices, thus it is straightforward to generalize to the case where different augmentation functions are applied to different samples, we omit it here for the simplicity of notations. Moreover, when the augmentation is sampled from a stochastic distribution, we can also characterize the optimal solution of the expected loss in the same way.} Specifically, if we apply the random masking augmentation \eqref{aug: random masking}, we can further obtain a result to characterize the optimal solution. 
For any square matrix $A\in\mathbb{R}^{d\times d}$, we denote $D(A)$ to be $A$ with all off-diagonal entries set to be zero and $\Delta(A)=A-D(A)$ to be $A$ with all diagonal entries set to be zero. Then we have the following corollary for random masking augmentation. 
\begin{corollary}
	\label{prop: diagonal contrast}
	Under the same conditions as in Proposition \ref{prop: augment}, if we use random masking (Definition \ref{aug: random masking}) as our augmentation function, then the minimizer of the expected loss function of contrastive learning problem \eqref{opt: CL ap} over the distribution of random augmentations (i.e., $\mathbb{E}_{g_1,g_2}\mathcal{L}_{SelfCon}(W)$) is given by:
	\begin{equation*}
		W_{\CL} = C\left(\sum_{i=1}^{r}u_i\sigma_i v_i^\top\right)^\top,
	\end{equation*}
	where $C>0$ is a positive constant, $\sigma_i$ is the $i$-th largest eigenvalue of the following matrix:
	\begin{equation}
		\Delta(XX^\top)-\frac{1}{n-1}X(1_n 1_n^\top-I_n)X^\top,
	\end{equation}
	 $u_i$ is the corresponding eigenvector and $V=[v_1,\cdots,v_r]\in\mathbb{R}^{r\times r}$ can be any orthonormal matrix.
\end{corollary}

The proof is given in Appendix \ref{prop: diagonal contrast ap}.

With Proposition \ref{prop: augment} and Corollary \ref{prop: diagonal contrast} established, we can find that the self-supervised contrastive learning equipped with augmented pairs generation and random masking augmentation can eliminate the effect of random noise on the diagonal entries of the observed covariance matrix. Since $\Cov(\xi)=\Sigma$ is a diagonal matrix, when 
the diagonal entries $\Cov(U^\star z)=\nu^2 U^\star U^{\star \top}$ only take a small proportion of the total Frobenius norm, the contrasting augmented pairs will preserve the core features while eliminating most of the random noise and give a more accurate estimation of core features. 

\subsection{Feature Recovery from Noisy Data}
\label{sec: recover}
After bridging both autoencoder and contrastive learning with PCA, now we can perform the analysis of feature recovery ability to understand the benefit of contrastive learning over autoencoders. 
As mentioned above, our target is to recover the subspace spanned by the columns of $U^\star$, which can further help us obtain information on the unobserved $z$ that is important for in-domain downstream tasks. However, the observed data has a covariance matrix of $\nu^2U^\star U^{\star \top}+\Sigma$ rather than the desired $\nu^2U^\star U^{\star \top}$, 
which brings difficulty to representation learning. We demonstrate that contrastive learning can better exploit the structure of core features and obtain better estimation than autoencoders in this setting.

We start with autoencoders. 
In the noiseless case, the covariance matrix is $\nu^2U^\star U^{\star \top}$ and autoencoders can \textit{perfectly} recover the core features. However, in noisy cases, the random noises sometimes perturb the core features, which makes autoencoders fail to learn the core features. Such noisy cases are widespread in real applications such as measurement errors and backgrounds in images such as grasses and sky. Interestingly, we will later show that contrastive learning can better recover $U^\star$ despite the presence of large noise.

To provide rigorous analysis, we first introduce the incoherent constant \citep{candes2009exact}.
\begin{definition}[Incoherent Constant]
	We define the incoherence constant of $U \in \mathbb{O}_{d,r}$ as
	\begin{equation}
		\label{eq: incoherent}
		I(U)= \max _{i \in[d]}\left\|e_{i}^\top U\right\|^{2}.
	\end{equation}
\end{definition}

Intuitively, the incoherent constant measures the degree of the incoherence of the distribution of entries among different coordinates, or loosely speaking, the similarity between $U$ and canonical basis $\{e_i\}_{i=1}^{d}$. For uncorrelated random noise, the covariance matrix is diagonal and its eigenspace is exactly spanned by the canonical basis $\{e_i\}_{i=1}^d$ (if the diagonal entries in $\Sigma$ are all different), which attains the maximum value of the incoherent constant. On the contrary, the core features usually exhibit certain correlation structures and the corresponding eigenspace of the covariance matrix is expected to have a lower incoherent constant.

We then introduce a few assumptions which our theoretical results are built on. Recall that in the spiked covariance model \eqref{model: spiked covariance}, $x = U^\star z+\xi$, $\Cov(z)=\nu^2I_r$ and $\Cov(\xi)=\diag(\sigma_1^2,\cdots,\sigma_d^2)$.

\begin{assumption}[Regular Covariance Condition]
	\label{asm: regular}
	The condition number of covariance matrix $\Sigma = \diag(\sigma_{1}^2,\cdots,\sigma_{d}^2)$ satisfies $
		\kappa := \sigma_{(1)}^{2}/\sigma_{(d)}^{2}<C,$
	where $\sigma_{(j)}^2$ represents the $j$-th largest number among $\sigma_{1}^2,\cdots,\sigma_{d}^2$ and $C>0$ is a universal constant.
\end{assumption}

\begin{assumption}[Signal to noise ratio condition]
	\label{asm: SNR}
	Define the signal-to-noise ratio $\rho := \nu / \sigma_{(1)}$, we assume $\rho=\Theta(1)$, implying that the covariance of noise is of the same order as that of the core features.
\end{assumption}

\begin{assumption}[Incoherent Condition]
	\label{asm: incoherent}
	The incoherent constant of the core feature matrix $U^\star\in\mathbb{O}_{d, r}$ satisfies $I(U^\star)=O\qty(r\log d/d).$ 
\end{assumption}
The incoherent constant often appears in the literature of matrix completion \citep{candes2009exact} and PCA \citep{zhang2018heteroskedastic}.
The order of $I(U^\star)$ can be arbitrary as long as it decreases to $0$ as $d\rightarrow\infty$. One can directly adapt the later results to this setting.
If $U$ is distributed uniformly on $\mathbb{O}_{d,r}$, then the expectation of incoherent constant is of order $r\log d / d$.
\begin{lemma}[Expectation of incoherent constant over a uniform distribution]
	\label{lem: incoherent}
	\begin{equation}
	\label{eq: expectation incoherent }
		\mathbb{E}_{U\sim \unif(\mathbb{O}_{d, r})}I(U) = O\qty(\frac{r}{d}\log d).
	\end{equation}
\end{lemma}
Thus, we set $I(U^\star)$ to the order $r\log d/d$ for simplicity.
The proof is given in Appendix \ref{lem: incoherent ap}.
Here we provide a remark on the implication of our assumptions, and we defer a further discussion on how to generalize our main results under weaker assumptions in Remark \ref{rem: general assumption}.
\begin{remark}
\label{remark: assumption}
The three assumptions above can be explained as follows: Assumption \ref{asm: regular} implies that the variances of all dimensions are of the same order. For Assumption \ref{asm: SNR}, we focus on a large noise regime where the noise may hurt the estimation significantly. Here we assume the ratio lies in a constant range, but our theory can easily adapt to the case where $\rho$ has a decreasing order. Specifically, for Theorems \ref{thm: recover PCA}, \ref{thm: recover CL}, \ref{thm: best linear predictor risk CL} and \ref{thm: best linear predictor risk PCA lower bound} presented below, we derive an explicit dependence on $\rho$ of each result in  the appendix. One can check Equations \eqref{dependence of rho AE}, \eqref{dependence of rho CL}, \eqref{dependence on rho: classification}, \eqref{dependence on rho: regression}, \eqref{dependence on rho: classification lb} and \eqref{dependence on rho: regression lb} for details. Assumption \ref{asm: incoherent} implies a stronger correlation among the coordinates of core features, which is the essential property to distinguish them from random noise. 
\end{remark}
Now we are ready to present our first result, showing that the autoencoders are unable to recover the core features in the large-noise regime. Due to the equivalence among PCA, autoencoders, and GANs we presented in Section~\ref{sec: connection to autoencoders},  for brevity, we only focus on autoencoders hereafter.
\begin{theorem}[Recovery Ability of Autoencoders, Lower Bound]
	\label{thm: recover PCA}
	Consider the spiked covariance model \eqref{model: spiked covariance}, under Assumptions \ref{asm: regular}-\ref{asm: incoherent} and $n> d\gg r$, let $W_{\AEN}$ be the learned representation of autoencoders with singular value decomposition $W_{\AEN}=(U_{\AEN}\Sigma_{\AEN}V_{\AEN}^\top)^\top$ (as in Equation \eqref{opt: AE}). If we further assume $\{\sigma_{i}^2\}_{i=1}^d$ are different from each other and $\sigma_{(1)}^2/(\sigma_{(r)}^2-\sigma_{(r+1)}^2)<C_{\sigma}$ for some universal constant $C_{\sigma}$. Then there exist two universal constants $C_\rho>0$, $c\in (0,1)$, such that when $\rho<C_\rho$, we have
	\begin{equation}
	\label{PCA lower bound}
		\mathbb{E}\left\|\sin \Theta\left(U^\star, U_{\AEN}\right)\right\|_F \geq c\sqrt{r}.
	\end{equation}
\end{theorem}
The proof is given in Appendix \ref{thm: recover PCA ap}. The condition $d \gg r$ means that there exists a sufficiently small constant $c>0$ independent of $d$ and $r$ such that $r / d < c$ holds. The additional assumptions $\{\sigma_{i}^2\}_{i=1}^d$ are different from each other and $\sigma_{(1)}^2/(\sigma_{(r)}^2-\sigma_{(r+1)}^2)<C_{\sigma}$ for some universal constant $C_{\sigma}$ are made to ensure the identifiability of top-$r$ eigenspace. We need these conditions to guarantee the uniqueness of $U_{\AEN}$. As an extreme example, the top-$r$ eigenspace of the identity matrix can be any $r$-dimensional subspace and thus not unique. To avoid discussing such arbitrariness of the output, we make these assumptions to guarantee the separability of the eigenspace.

Then we investigate the feature recovery ability of the self-supervised contrastive learning approach. 

\begin{theorem}[Recovery Ability of Contrastive Learning, Upper Bound]
    \label{thm: recover CL}
	Under the spiked covariance model \eqref{model: spiked covariance}, random masking augmentation in Definition \ref{aug: random masking}, Assumptions \ref{asm: regular}-\ref{asm: incoherent} and $n> d\gg r$, let $W_{\CL}$ be any solution that minimizes Equation \eqref{loss: self contrastive}, and denote its singular value decomposition as $W_{\CL}=(U_{\CL}\Sigma_{\CL}V_{\CL}^\top)^\top$, then we have
	\begin{equation}
	\label{CL upper 1}
		\mathbb{E}\left\|\sin \Theta\left(U^\star, U_{\CL}\right)\right\|_F \lesssim\frac{r^{3/2}}{d}\log d+\sqrt{\frac{dr}{n}}. 
	\end{equation}
\end{theorem}
The proof is given in Appendix \ref{thm: recover CL ap}. The two terms in equation \eqref{CL upper 1} can be explained as follows: the first term is due to the shift between the distributions of the augmented data and the original data. Specifically, the random masking augmentation generates two views with disjoint nonzero coordinates and thus can mitigate the influence of random noise on the diagonal entries in the covariance matrix. However, such augmentation slightly hurts the estimation of core features. This bias, appearing as the first term in Equation \eqref{CL upper 1}, is measured by the incoherent constant defined in Equation \eqref{eq: incoherent}. The second term corresponds to the estimation error of the population covariance matrix. 

Theorems \ref{thm: recover PCA} and \ref{thm: recover CL} characterize the difference in feature recovery ability between autoencoders and contrastive learning. The autoencoders fail to recover most of the core features in the large-noise regime since $\|\sin\Theta(U,U^\star)\|_F$ has a trivial upper bound $\sqrt{r}$. In contrast, with the help of data augmentation, the contrastive learning approach mitigates the corruption of random noise while preserving core features. As $n$ and $d$ increase, it yields a consistent estimator of core features and 
further leads to better performance in the in-domain downstream tasks, as shown in the next section.
\begin{remark}
\label{rem: general assumption} 
    Here we discuss the potential generalization of our results to the setting with weaker assumptions. Intuitively speaking, the random masking augmentation exploits the prior knowledge that the core features in the original signal are more structural across different coordinates compared with the random noise. Thus the essential requirements are 
    \begin{enumerate}
        \item Noise is less correlated between different coordinates compared with core features.
        \item Core features and noise are very different, i.e., $\|\sin\Theta(U^\star,U_\Sigma)\|_F$ is large.
    \end{enumerate}
    Those two requirements correspond to the diagonal assumption on $\Sigma$ and incoherent assumption on $U^\star$ (Assumption \ref{asm: incoherent}). In particular, the latter is to give a lower bound for $\|\sin\Theta(U^\star,U_\Sigma)\|_F$ when $\Sigma$ is diagonal and heteroskedasticity. For more general $\Sigma$ and $U^\star$, it suffices to assume $\frac{\|\Delta(\Sigma)\|_2}{\nu^2}=o(1)$, and $\|\sin\Theta(U^\star,U_\Sigma)\|_F=\Omega(\sqrt{r})$, and we can still draw a similar comparison under these assumptions. Notice that by Lemma \ref{lem: incoherent}, when $U^\star$ is randomly chosen we will immediately have $\E\|\sin\Theta(U^\star,U_\Sigma)\|_F=(1-o(1))\sqrt{r}$. 
Similar arguments also apply to all of the later results in this paper and we omit them for simplicity.
\end{remark}

\begin{remark}\label{rem: masked ae} Similar random masking augmentation (as in Definition \ref{aug: random masking}) can also apply to autoencoders. Although directly applying this augmentation would not work as well since it will not affect the optimal solution (see discussion in Appendix \ref{sec: MAE}), an alternative strategy is to reconstruct the whole data from the masked one. This method was originally proposed as denoising autoencoders (DAEs) for general augmentation \cite{vincent2008extracting}, and was proven to be powerful with masking augmentation in a recently proposed representation learning method, masked autoencoders (MAEs) \citep{he2021masked}. 
    DAEs are a variant of autoencoders that are trained to reconstruct the original image from randomly masked patches. It has been found that DAEs (especially MAEs) outperform other self-supervised methods like MoCo v3, DINO, and BEiT after fine-tuning \citep{he2021masked}.
    
    More specifically, under the same setup described in Section~\ref{sec: prelim}, let $A$ be the random masking augmentation defined in Definition \ref{aug: random masking}. We adopt the symmetric linear encoders and decoders.
    Given samples $X = [x_1, \dots, x_n] \in \R^{d \times n}$,
    we formally define the loss minimization problem\footnote{In \citep{he2021masked}, the loss function is computed on masked coordinates only, but as the authors noted, ``This choice is purely result-driven: computing the loss on all pixels leads to
a slight decrease in accuracy (e.g., 0.5\%)." Hence we will analyze the loss with respect to all coordinates for simplicity. } of DAEs as
    \begin{align}
        \min_{W \in \R^{r \times d}: W W^\top = 2 I_r} \frac{1}{n} \mathbb{E}_A \qty[ \| W^\top W A X - X\|_F^2 ].\label{loss: MAE}
    \end{align}
    Notice that the DAEs may not preserve the norm of the input since $\mathbb{E}_A[\|A x_i\|^2] = (1/2) \|x_i\|^2$.
    As a result, we optimize the loss under the scaled constraint $W W^\top = 2 I_r$. 
    
    Then, we claim that under the same conditions as in Theorem \ref{thm: recover CL}, DAEs behave similarly to contrastive learning: 
    Let $W_{\DAE}$ be any solution that minimizes equation \eqref{loss: MAE}, and denote its singular value decomposition as $W_{\DAE}=(U_{\DAE}\Sigma_{\DAE}V_{\DAE}^\top)^\top$, then
    we have (the proof is given in Appendix \ref{proof: masked 1}.)
	\begin{equation}
		\mathbb{E}\left\|\sin \Theta\left(U^\star, U_{\DAE}\right)\right\|_F \lesssim\frac{r^{3/2}}{d}\log d+\sqrt{\frac{dr}{n}}. \label{eq: masked ae}
	\end{equation}
    From Theorem \ref{thm: recover PCA}, we know that under high dimensional settings with large sample sizes, DAEs (or masked autoencoders) significantly outperform classic autoencoders. {Moreover, compared to Theorem \ref{thm: recover CL}, the upper bounds of DAEs are the same as contrastive learning with random masking augmentations. We also provide experimental results on synthetic datasets to verify this result in Appendix \ref{proof: masked 1}.  Although they have similar performance in our linear representation framework because both of them exploit the masking views to eliminate noise, the difference could arise from other aspects such as network architecture and training algorithms, for example, \cite{he2021masked} used a vision Transformer \citep{dosovitskiy2020image} while \cite{chen2020simple} used a ResNet \citep{He2016DeepRL}.} 

\end{remark}

\subsection{Performance on In-Domain Downstream Tasks}
\label{sec: downstream}
In the previous section, we have seen that contrastive learning can recover the core feature effectively. In practice, we are interested in using the learned features on in-domain downstream tasks. \citet{he2020momentum} experimentally showed the overwhelming performance of linear classifiers trained on representations learned with contrastive learning against several supervised learning methods in those in-domain downstream tasks. 



Following the recent success, 
here we evaluate the in-domain downstream performance of simple predictors, which take a linear transformation of the representation as an input. Let $W_{\CL}$ and $W_{\AEN}$ be the learned representations based on train data $X \in \mathbb{R}^{n\times d}$.
We observe a new signal $\check x = U^\star \check z + \check \xi$ independent of $X$ following the spiked covariance model \eqref{model: spiked covariance}.
For simplicity, assume $\check z$ follows $N(0, \nu^2 I_r)$ and is independent of $\check \xi$. 
We consider two major types of in-domain downstream tasks: classification and regression.
For the binary classification task, we observe a new supervised sample $\check y$ following the binary response model:
\begin{align}\label{model: binary classification}
    \check y | \check z &\sim \text{Ber}(F(\langle \check z, w^\star\rangle/\nu)),
\end{align}
where $F: \mathbb{R} \to [0, 1]$ is a known monotone increasing function satisfying $1 - F(u) = F(-u)$ for any $u \in \mathbb{R}$ , and $w^\star \in \mathbb{R}^r$ is a unit vector of coefficients. 
Notice that our model \eqref{model: binary classification} includes a logistic model (when $F(u) = 1/(1 + e^{-u})$) and probit models (when $F(u) = \Phi(u)$, where $\Phi$ is the cumulative distribution function of the standard normal distribution.) We can also interpret model \eqref{model: binary classification} as a shallow neural network model with width $r$ for binary classification.
For the regression task, we observe a new supervised sample $\check y$ following the linear regression model:
\begin{align}\label{model: spiked covariance regression}
    \check y &= \langle \check z, w^\star\rangle/\nu + \check \epsilon,
\end{align}
where $\check \epsilon \sim (0, \sigma_\epsilon^2)$ is independent of $\check z$, , and $w^\star \in \mathbb{R}^r$ is a unit vector of coefficients as before.
We can interpret this model as a principal component regression model (PCR) \citep{jolliffe1982note} under standard error-in-variables settings\footnote{
In error-in-variables settings, the bias term of the measurement error appears in prediction and estimation risk. Since our focus lies in proving a better performance of contrastive learning against autoencoders, we ignore the unavoidable bias term here by considering the excess risk. 
}, where we assume that the coefficients lie in a low-dimensional subspace spanned by column vectors of $U^\star$. We either estimate or predict the signal based on the observed samples contaminated by the measurement error $\check\xi$.
For details of PCR in error-in-variables settings, see, for example, \citet{cevid2020spectral,agarwal2020principal,bing2021prediction}.

In classification setting, we specify $0$-$1$ loss, that is, $\ell_c(\delta) \triangleq \1\{\check y \neq \delta(\check x)\}$ for some predictor $\delta$ taking values in $\{0, 1\}$.
For regression task, we employ the squared error loss $\ell_r(\delta) \triangleq (\check y - \delta(\check x))^2$.
Based on some learned representation $W$, we consider a class of linear predictors. Namely, $\delta_{W, w}(\check x) \triangleq \1\{F(w^\top W \check x) \geq 1/2\}$ for classification task and $\delta_{W, w}(\check x) \triangleq w^\top W\check x$ for regression task, where $w \in \mathbb{R}^r$ is a weight vector $w \in \mathbb{R}^r$. Note that the learned representation depends only on unsupervised samples $X$.
Let $\mathbb{E}_{\mathcal{D}}[\cdot]$ and $\mathbb{E}_{\mathcal{E}}[\cdot]$ the expectations with respect to $(X, Z)$ and $(\check y, \check x, \check z)$, respectively.

Our goal as stated above is to bound the prediction risk of predictors $\{\delta_{W, w} : w \in \mathbb{R}^r\}$ constructed upon the learned representations $W_{\CL}$ and $W_{\AEN}$, that is, the quantity $\inf_{w \in \mathbb{R}^r} \mathbb{E}_{\mathcal{E}}[\ell(\delta_{W_{\CL}, w})]$ and $\inf_{w \in \mathbb{R}^r} \mathbb{E}_{\mathcal{E}}[\ell(\delta_{W_{\AEN}, w})]$.

Now we state our results on the performance of the in-domain downstream prediction task.
\begin{theorem}[Excess Risk for In-Domain Downstream Task: Upper Bound]\label{thm: best linear predictor risk CL}
    Suppose the conditions in Theorem \ref{thm: recover CL} hold. Then, for the regression task, we have 
    \begin{align*}
		\mathbb{E}_{\mathcal{D}}[\inf_{w\in\mathbb{R}^r} \mathbb{E}_{\mathcal{E}}[\ell_r(\delta_{W_{\CL}, w})]-\inf_{w\in\mathbb{R}^r}\mathbb{E}_{\mathcal{E}}[\ell_r(\delta_{U^{\star \top}, w})] \lesssim \frac{r^{3/2}}{d}\log d + \sqrt{\frac{dr}{n}},
	\end{align*}
    and for the classification task,
    \begin{align*}
        \mathbb{E}_{\mathcal{D}}[\inf_{w \in \mathbb{R}^r} \mathbb{E}_{\mathcal{E}}[\ell_c(\delta_{W_{\CL}, w})] - \inf_{w \in \mathbb{R}^r} \mathbb{E}_{\mathcal{E}}[\ell_c(\delta_{U^{\star \top}, w})] = O\qty(\frac{r^{3/2}}{d} \log d + \sqrt{\frac{dr}{n}}) \wedge 1.
    \end{align*}
\end{theorem}
The proofs are given in Appendix \ref{thm: best linear predictor risk PCA lower bound restatement}.

This result shows that the price of estimating $U^\star$ by contrastive learning on an in-domain downstream prediction task can be made small in the case where the core feature lies in a relatively low-dimensional subspace, and the number of samples is relatively large compared to the ostensible dimension of data.

However, the in-domain downstream performance of autoencoders is not as good as contrastive learning.
We obtain the following lower bound for the in-domain downstream prediction risk with the autoencoders.
\begin{theorem}[Excess Risk for In-Domain Downstream Task: Lower Bound)]\label{thm: best linear predictor risk PCA lower bound}
	Suppose the conditions in Theorem \ref{thm: recover PCA} hold.
	Assume $r \leq r_c$ holds for some constant $r_c > 0$.
	Additionally assume that $\rho = \Theta(1)$ is sufficiently small and $n \gg d \gg r$. Then,
    For the regression task,
    \begin{align*}
        \mathbb{E}_{\mathcal{D}}[\inf_{w\in\mathbb{R}^r} \mathbb{E}_{\mathcal{E}}[\ell_r(\delta_{U_{\AEN}, w})]-\inf_{w\in\mathbb{R}^r}\mathbb{E}_{\mathcal{E}}[\ell_r(\delta_{U^\star, w})] \geq c_c',
    \end{align*}
    and for classification task, if $F$ is differentiable at $0$ and $F'(0) > 0$, then
    \begin{align*}
        \mathbb{E}_{\mathcal{D}}[\inf_{w \in \mathbb{R}^r} \mathbb{E}_{\mathcal{E}}[\ell_c(\delta_{U_{\AEN}, w})] - \inf_{w \in \mathbb{R}^r} \mathbb{E}_{\mathcal{E}}[\ell_c(\delta_{U^\star, w})] \geq c_r',
    \end{align*}
    where $c_r' > 0$ and $c_c' > 0$ are constants independent of $n$ and $d$.
\end{theorem}
The proof is given in Appendix \ref{thm: best linear predictor risk PCA lower bound restatement}. The condition $n \gg d \gg r$ means that there exists a sufficiently small constant $c>0$ independent of $n$, $d$ and $r$ such that $d / n \vee r / d < c$ holds.

The constant $c'$ appearing in Theorem \ref{thm: best linear predictor risk PCA lower bound} is a constant term independent of $d$ and $n$. 
Thus, when $d$ is sufficiently large compared to $r$ and $d/n$ is small, the upper bound of in-domain downstream task performance via contrastive learning in Theorem \ref{thm: best linear predictor risk CL} is smaller than the lower bound of in-domain downstream task performance via autoencoders.
The assumption of $r \leq r_c$ in Theorem \ref{thm: best linear predictor risk PCA lower bound} is assumed for clarity of presentation. Using the same techniques in the proof of Theorem \ref{thm: best linear predictor risk PCA lower bound}, one can obtain a constant lower bound for autoencoders with slightly stronger assumptions, for example, $\rho^2 = O(1/\log d)$ with $n \gg dr$, without assuming $r \leq r_c$. Our theory can be adapted to both of these assumptions.
Our results support the empirical success of contrastive learning.

\section{The Impact of Labeled Data in Supervised Contrastive Learning}\label{sec: labeled data}
 Recent works have explored adding label information to improve contrastive learning \citep{khosla2020supervised}.
Empirical results show that label information can significantly improve the accuracy of the in-domain downstream tasks.  However, when domain shift is considered, the label information hardly improves and even hurts transferability \citep{Islam2021ABS}. For example, in Table 2 of \citet{khosla2020supervised} and the first column in Table 4 of \citet{Islam2021ABS}, supervised contrastive learning shows significant improvement with 7\%-8\% accuracy increase on in-domain downstream classification on ImageNet and Mini-ImageNet. On the contrary, in Table 4 of \citet{khosla2020supervised} and Table 4 of \citet{Islam2021ABS},  supervised contrastive learning hardly increases the predictive accuracy compared to the self-supervised contrastive learning (the difference of mean accuracy is less than 1\%) and can harm significantly on some datasets (e.g. 5.5\% lower for SUN 397 in Table 4 of \citet{khosla2020supervised}).  These results indicate that some mechanisms in supervised contrastive learning hurt model transferability while the improvement in source tasks is significant. Moreover, in Table 4 of \citet{Islam2021ABS}, it is observed that combining supervised learning and self-supervised contrastive learning together achieves the best transfer learning performance compared to each of them individually. Motivated by those empirical observations, in this section, we aim to investigate the impact of labeled data in contrastive learning and provide a theoretical foundation for these phenomena.

\subsection{Feature Mining in Multi-Class Classification}
\label{sec: classification}
We first demonstrate the impact of labels in contrastive learning under the standard single-sourced (i.e. no transfer learning) setting. Suppose our samples are drawn from $r+1$ different classes with probability $p_k$ for class $k\in[r+1]$, and $\sum_{k=1}^{r+1}p_k=1$. For each class, samples are generated from a class-specific Gaussian distribution:
\begin{equation}
	\label{model: Gaussian}
	x^k= \mu^k+\xi^k,\quad \xi^k\sim\mathcal{N}(0,\Sigma^k), \quad \forall k=1,2,\cdots,r+1.
\end{equation} 
{
We assume the norms of $\mu^k,\forall k\in [r+1]$ are in the same order, that is, denote $\nu=\|\mu^1\|/\sqrt{r}$, we have $\|\mu^k\|=O(\sqrt{r}\nu), \forall k\in [r+1]$. We further assume $\Sigma^k=\diag(\sigma_{1,k}^2,\cdots,\sigma_{d,k}^2)$, denote $\sigma_{(1)}^2=\max_{1\leq i\leq d, 1\leq j\leq r+1}\sigma_{i,j}^2$ and assume $\sum_{k=1}^{r+1}p_k\mu^k=0$, where the last assumption is added to ensure identifiability since the classification problem \eqref{model: Gaussian} is invariant under translation. Denote $\Lambda=\sum_{k=1}^{r+1}p_k\mu^k\mu^{k\top}$, we assume $\operatorname{rank}(\Lambda)=r$ and $C_1\nu^2<\lambda_{(r)}(\Lambda)<\lambda_{(1)}(\Lambda)<C_2\nu^2$ for two universal constants $C_1$ and $C_2$.
We remark that this model is a labeled version of the spiked covariance model \eqref{model: spiked covariance} since the core features and random noise are both sub-Gaussian. We use $r+1$ classes to ensure that $\mu^k$'s span an $r$-dimensional space, and denote its orthonormal basis as $U^\star$. Recall that our target is to recover $U^\star$.}
\begin{remark}
    Here we focus on explaining the impact of label information in the SupCon algorithm\citep{khosla2020supervised}. SupCon is designed for multi-class classification tasks and it requires using the class label to find positive samples. In that case, labels from a linear function of the latent features can not be used as class labels in a multi-class classification setting. Hence we proposed to use the Gaussian Mixture Model \eqref{model: Gaussian} to generate class labels while keeping the most consistency with models used earlier(i.e., the spiked covariance model \eqref{model: spiked covariance}).
\end{remark}
As introduced in Definition \ref{pair: supervised}, the supervised contrastive learning introduced by \citet{khosla2020supervised} allows us to generate contrastive pairs using labeled information and discriminate instances across classes. When we have both labeled data and unlabeled data, we can perform contrastive learning based on pairs that are generated separately for the two types of data.

\paragraph{Data Generating Process}
Formally, let us consider the case where we draw $n$ samples as unlabeled data $X=[x_1,\cdots,x_n]\in\mathbb{R}^{d\times n}$ from the Gaussian mixture model \eqref{model: Gaussian} with $p_1=p_2=\cdots=p_{r+1}$. For the labeled data, we draw $(r+1)m$ samples; $m$ samples for each of the $r+1$ classes in the Gaussian mixture model, and denote them as $\hat{X}=[\hat{x}_1,\cdots,\hat{x}_{(r+1)m}]\in\mathbb{R}^{d\times (r+1)m}$. We discuss the above case for simplicity. More general versions that allow different sample sizes for each class are considered in Theorem \ref{thm: general supcon} (in the appendix). We study the following hybrid loss to illustrate how the label information helps promote performance over self-supervised contrastive learning:
\begin{equation}
	\label{sup contrastive task}
	\min_{W\in\mathbb{R}^{r\times d}}\mathcal{L}(W):=\min_{W\in\mathbb{R}^{r\times d}}\mathcal{L}_{\text{SelfCon}}(W)+\alpha\mathcal{L}_{\text{SupCon}}(W),
\end{equation}
where $\alpha>0$ is the ratio between supervised loss and self-supervised contrastive loss. Here we consider this generalized hybrid loss to show the benefit of exploiting additional unlabeled data. If we choose $\alpha\rightarrow\infty$ it will correspond to the original SupCon loss.

We first provide a high-level explanation of why label information can help learn core features. When the label information is unavailable, no matter how much (unlabeled) data we have, we can only take them (and their augmented views) as positive samples. 
In such a scenario, performing augmentation leads to an unavoidable trade-off between estimation bias and accuracy. However, if we have additional class information, we can contrast between data in the same class to extract more beneficial features that help distinguish a particular class from others and therefore reduce the bias. 
\begin{theorem}
\label{thm: SupCon}
	Suppose the labeled and unlabeled samples are generated as the process mentioned above. If Assumptions \ref{asm: regular}-\ref{asm: incoherent} hold, $n> d\gg r$ and let $W_{\CL}$ be any solution that minimizes the supervised contrastive learning problem in Equation \eqref{sup contrastive task}, and denote its singular value decomposition as $W_{\CL}=(U_{\CL}\Sigma_{\CL}V_{\CL}^\top)^\top$, then we have
	\begin{equation*}
		\begin{aligned}
			\mathbb{E}\|\sin\Theta(U_{\CL},U^\star)\|_F\lesssim&\frac{1}{1+\alpha }\left(\frac{r^{3/2}}{d}\log d+\sqrt{\frac{dr}{n}}\right)+\frac{\alpha}{1+\alpha} \sqrt{\frac{dr}{m}}.
		\end{aligned}
	\end{equation*}
\end{theorem}
The proof is given in Appendix \ref{thm: general supcon}
\begin{corollary}
    From Theorem \ref{thm: SupCon}, it directly follows that when we have $m$ labeled data for each class and no unlabeled data ($\alpha\to\infty$),
    \begin{equation*}
        \mathbb{E}\|\sin\Theta(U_{\CL},U)\|_F\lesssim\sqrt{\frac{dr}{m}}.
    \end{equation*}
\end{corollary}
The first bound in Theorem \ref{thm: SupCon} demonstrates how the effect of labeled data changes with the ratio $\alpha$ in the hybrid loss in Equation \eqref{sup contrastive task}. In addition, compared with Theorem \ref{thm: recover CL}, when we only have labeled data ($\alpha\to\infty$), the second bound in Theorem \ref{thm: SupCon} indicates that with labeled data being available, the supervised contrastive learning can yield consistent estimation as $m\rightarrow\infty$ while the self-supervised contrastive learning consists of an irreducible bias term $O(r^{3/2}\log d/d)$. At a high level, label information can help gain accuracy by creating more positive samples for a single anchor and therefore extract more decisive features. One should notice a caveat that when labeled data is extremely rare compared to unlabeled data, the estimation of supervised contrastive learning suffers from high variance. In comparison, self-supervised contrastive learning, which can exploit a much larger number of samples, may outperform it.


\subsection{Information Filtering in Multi-Task Transfer Learning}
\label{sec: transfer}

In  this section, we show that the theoretical tools developed in this paper can be used to illustrate the role of label information when using contrastive learning in the transfer learning setting. Label information can tell us the beneficial information for the in-domain downstream task, and learning with labeled data will filter out useless information and preserve the decisive parts of core features. However, in transfer learning, 
the label information is sometimes rather found to hurt the performance of contrastive learning. For example, in Table 4 of \citet{Islam2021ABS}, while supervised contrastive learning gains 8\% improvement in source tasks by incorporating label information, it improves only 1\% on generalizing to new datasets on average and can even hurt on some datasets. Such observation implies that label information in contrastive learning has very different roles for generalization on source tasks and new tasks. In this section, we consider two regimes of transfer learning -- tasks are insufficient/abundant. In both regimes, we provide theories to support the empirical observations and further demonstrate how to wisely combine supervised and self-supervised contrastive learning to avoid those harms and achieve better performance. Specifically, we consider a transfer learning problem with \textbf{regression setting and binary classification setting}. Suppose we have $T$ source tasks which share a common data generative model \eqref{model: spiked covariance}. For the $t$-th task, the labels are generated by
$
	y^t = \langle w_t, z \rangle/\nu
$
in a regression setting while $y^t=\sign(\langle w_t, z \rangle)$ in a binary classification setting,
where $w_t\in\mathbb{R}^r$ is a unit vector varying across tasks. These two settings share the same distribution for $x$ only differ in the way to generate the labels $y^t$.

To incorporate label information, we maximize the Hilbert-Schmidt Independence Criteria (HSIC)  \citep{gretton2005measuring,barshan2011supervised}, which has been widely used in literature \citep{song2007colored,song2007dependence,song2007supervised,barshan2011supervised}.


\subsubsection{Hilbert-Schmidt Independent Criteria}
\label{sec: HSIC}
\citet{gretton2005measuring} proposed the Hilbert Schmidt Independent Criteria (HSIC) to measure the dependence between two random variables. It computes the Hilbert-Schmidt norm of the cross-covariance operator associated with their Reproducing Kernel Hilbert Space (RKHS). Such measurement has been widely used as a supervised loss function in feature selection \citep{song2007supervised}, feature extraction \citep{song2007colored}, clustering \citep{song2007dependence} and supervised PCA \citep{barshan2011supervised}. 

The basic idea behind HSIC is that two random variables are independent if and only if any bounded continuous functions of the two random variables are uncorrelated. Let $\mathcal{F}$ be a separable RKHS containing
all continuous bounded real-valued functions mapping from $\mathcal{X}$ to $\mathbb{R}$ and $\mathcal{G}$ be that for maps from $\mathcal{Y}$ to $\R$. For each point $x \in$ $\mathcal{X}$, there exists a corresponding element $\phi \in \mathcal{F}$ such that $\left\langle\phi(x), \phi\left(x^{\prime}\right)\right\rangle_{\mathcal{F}}=k\left(x, x^{\prime}\right)$, where $k: \mathcal{X} \times \mathcal{X} \rightarrow \mathbb{R}$ is a unique positive definite kernel. Likewise, define the kernel $l(\cdot,\cdot)$ and feature map $\psi$ for $\mathcal{G}$. 
The empirical HSIC is defined as follows.
\begin{definition}[Empirical HSIC \citep{gretton2005measuring}]
Let $Z:=\{(x_{1}, y_{1}),\ldots$, $(x_{m}, y_{m})\}$ $\subseteq \mathcal{X} \times \mathcal{Y}$ be a series of $m$ independent and identically distributed observations. An estimator of HSIC, written as $\operatorname{HSIC}(Z, \mathcal{F}, \mathcal{G})$, is given by
$$
\operatorname{HSIC}(Z, \mathcal{F}, \mathcal{G}):=(m-1)^{-2} \operatorname{tr} (K H L H),
$$
where $H, K, L \in \mathbb{R}^{m \times m}, K_{i j}:=k\left(x_{i}, x_{j}\right), L_{i j}:=l\left(y_{i}, y_{j}\right)$ and $H:=I_m-(1/m)1_m1_m^\top$.
\end{definition}
In our setting, we aim to maximize the dependency between learned features $WX\in\mathbb{R}^{r\times n}$ and label $y\in\mathbb{R}^{n}$ via HSIC. Substituting $K \leftarrow X^\top W^\top WX$ and $L \leftarrow yy^\top$, we obtain our supervised loss for the representation matrix $W$:
\begin{equation}
\label{HSIC AP}
    \operatorname{HSIC}(X,y;W)=\frac{1}{(n-1)^2}\tr(X^\top W^\top WXHyy^\top H).
\end{equation}

A more commonly used supervised loss in the regression task is the mean squared error.
Here we explain the equivalence of maximizing HSIC with penalty $\|W W^\top\|_F^2$ and minimizing the mean squared error in the regression task.

Recall that in the contrastive learning framework, we first learn the representation via a linear transformation and then perform linear regression to learn a predictor with the learned representation. 
Consider the mean squared error $\mathcal{L}(\delta)=(1/n)\sum_{i=1}^n(\delta(x_i)-y_i)^2$, where $\delta(x_i)$ is a predictor of $y_i$. Also consider a linear class of predictors $\delta_{W, w}(x_i) = w^\top Wx_i$ with parameter $w \in \R^r$. 
Assume that both $X$ and $y$ are centered.
For any fixed representation $W$, the minimum mean squared error is given by
\begin{equation*}
\begin{aligned}
    \min_{w \in \R^r} \mathcal{L}(\delta_{W, w})=&\frac{1}{n}\|(WX)^\top w^\star-y\|^2\\
    =&\frac{1}{n}\qty(y^\top y-\tr(X^\top W^\top(WXX^\top W^\top)^{-1}WXy y^\top)),
\end{aligned}
\end{equation*}
where $w^\star=(WXX^\top W^\top)^{-1}WXy$.
Ignoring the constant term $y^\top y$, it can be seen that the only essential difference between the minimization problem $\min_{w \in \R^r} \mathcal{L}(\delta_{W, w})$ and maximizing HSIC in Equation \eqref{HSIC AP} is the normalization term $(WXX^\top W^\top)^{-1/2}$ for $W$. 
Thus, minimizing the mean squared error is equivalent to maximizing HSIC with regularization term $\|WW^\top\|_F^2$.
Since $L_{\text{SelfCon}}$ contains the regularization term $\|WW^\top\|_F^2$, we can jointly use $L_{\text{SelfCon}}$ and HSIC as a surrogate for the standard regression error to avoid the singularity of learned representation $W$.
\subsubsection{Main results}
First, we consider the regression setting. Before stating our results, we prepare some notations. Suppose we have $n$ unlabeled data $X=[x_1,\cdots,x_n]\in\mathbb{R}^{d\times n}$ and $m$ labeled data for each source task $\hat{X}^t=[\hat{x}_1^t,\cdots,\hat{x}_m^t],y^t=[y_1^t,\cdots,y_m^t],\forall t=1,\dots, T$ where $x_i$ and $\hat{x}_j^t$ are independently drawn from the spiked covariance model \eqref{model: spiked covariance}, 
we learn the linear representation via the joint optimization:
\begin{equation}
	\label{trans contrastive task}
	\min_{W\in\mathbb{R}^{r\times d}}\mathcal{L}(W):=\min_{W\in\mathbb{R}^{r\times d}}\mathcal{L}_{\text{SelfCon}}(W)-\alpha\sum_{t=1}^{T}\operatorname{HSIC}(\hat{X}^t,y^t;W),
\end{equation}
where $\alpha>0$ is a pre-specified ratio between the self-supervised contrastive loss and HSIC. A more general setting, where the ratio $\alpha$ and the number of labeled data for each source task are allowed to depend on $t$, is considered in the appendix, see Section \ref{sec: Information filtering proof} for details. We now present a theorem showing the recoverability of $W$ by minimizing the hybrid loss function~\eqref{trans contrastive task}. 

\begin{theorem}
	\label{thm: transfer t<r}
	In the regression setting where $y^t = \langle w_t, z \rangle/\nu$ , suppose Assumptions \ref{asm: regular}-\ref{asm: incoherent} hold for the spiked covariance model \eqref{model: spiked covariance} and $n> d\gg r$, if we further assume that $\alpha>C$ for some constant $C$, $T<r$ and $w_t$'s are orthogonal to each other, and let $W_{\CL}$ be any solution that optimizes the problem in Equation \eqref{trans contrastive task}, and denote its singular value decomposition as $W_{\CL}=(U_{\CL}\Sigma_{\CL}V_{\CL}^\top)^\top$, then we have:
    \begin{small}
	\begin{equation}
	\label{CL transfer upper bound}
    \begin{aligned}
      \mathbb{E}\|\sin\Theta(U_{\CL},U^\star)\|_F\lesssim&\sqrt{r-T} \left(\frac{r\log d}{d}+\sqrt{\frac{d}{n}}+\alpha T\sqrt{\frac{ d}{m}} \wedge 1 \right) +\sqrt{T}\left(\frac{r\log d}{\alpha d}+\frac{1}{\alpha}\sqrt{\frac{d}{n}}+T\sqrt{\frac{d}{m}}\right).
    \end{aligned}
    \end{equation}
    \end{small}
\end{theorem}
The proof is given in Appendix \ref{thm: transfer t<r ap}.

Similar to Section \ref{sec: downstream}, we can obtain an in-domain downstream task risk in a supervised contrastive learning setting. Consider a new test task where a label is generated by $\check y = \langle w^\star, \check z \rangle / \nu$ with $\check x = U^\star \check z + \check\xi$.
Recall that the loss in the in-domain downstream task is measured by the squared error: $\ell_r(\delta) := (\check y - \delta(\check x))^2$. We obtain the following result.
\begin{theorem}\label{thm: downstream risk transfer t<r}
    Suppose the conditions in Theorem \ref{thm: transfer t<r} hold. Then, 
    \begin{equation}
        \begin{aligned}\label{CL transfer upper bound 4}
		&\mathbb{E}_{\mathcal{D}}[\inf_{w\in\mathbb{R}^r} \mathbb{E}_{\mathcal{E}}[\ell_r(\delta_{W_{\CL}, w})]-\inf_{w\in\mathbb{R}^r}\mathbb{E}_{\mathcal{E}}[\ell_r(\delta_{U^{\star \top}, w})]\\
		&\quad\lesssim \sqrt{r-T} \left(\frac{r\log d}{d}+\sqrt{\frac{d}{n}}+\alpha T\sqrt{\frac{ d}{m}} \wedge 1 \right) +\sqrt{T}\left(\frac{r\log d}{\alpha d}+\frac{1}{\alpha}\sqrt{\frac{d}{n}}+T\sqrt{\frac{d}{m}}\right).
	\end{aligned}
    \end{equation}
    
\end{theorem}
The proof is given in Appendix \ref{thm: downstream risk transfer t<r ap}.

In Theorem \ref{thm: transfer t<r} and Theorem \ref{thm: downstream risk transfer t<r}, as $\alpha$ goes to infinity (corresponding to the case where we only use the supervised loss), the upper bounds in Equations \eqref{CL transfer upper bound} and \eqref{CL transfer upper bound 4} are reduced to $\sqrt{r-T}+T^{3/2}\sqrt{d/m}$, which is worse than the  $r^{3/2}\log d/d$ rate obtained by self-supervised contrastive learning (Theorem \ref{thm: recover CL}). This implies that when the model focuses mainly on the supervised loss, the algorithm will extract the information only beneficial for the source tasks and fail to estimate other parts of core features. As a result, when the target task has a very different distribution, labeled data will bring extra bias and therefore hurt the transferability. Additionally, one can minimize the right-hand side of Equation \eqref{CL transfer upper bound} to obtain a sharper rate. Specifically, 
we can choose an appropriate $\alpha$ such that the upper bound becomes $\sqrt{r^2(r-T)}\log d/d$ (when $n,m\rightarrow\infty$), obtaining a smaller rate than that of the self-supervised contrastive learning. These facts provide theoretical foundations for the recent empirical observations that smartly combining supervised and self-supervised contrastive learning achieves significant improvement in transferability compared with performing each of them individually \citep{Islam2021ABS}. 
\begin{remark}
    A heuristic intuition of this surprising fact is that when tasks are not diverse enough, supervised training will only focus on the features that are helpful to predict the labels of source tasks and ignore other features. For example, we have unlabeled images which contain cats or dogs and the background can be sandland or forest. If the source task focuses on classifying the background, supervised learning will not learn features associated with cats and dogs, while self-supervised learning can learn these features since they are helpful to discriminate different images. As a result, although supervised learning can help to classify sandland and forest, it can hurt performance on the classification of dogs and cats and we should incorporate self-supervised contrastive learning to learn these features.
\end{remark}
When the tasks are abundant enough 
then estimation via labeled data can recover core features completely. Similar to Theorem \ref{thm: transfer t<r} and Theorem \ref{thm: downstream risk transfer t<r}, we have the following results. 
\begin{theorem}
	\label{thm: transfer t>r}
	In the regression setting where $y^t = \langle w_t, z \rangle/\nu$, suppose Assumptions \ref{asm: regular}-\ref{asm: incoherent} hold for the spiked covariance model \eqref{model: spiked covariance} and $n> d\gg r$, if we further assume that $T>r$ and $\lambda_{(r)}(\sum_{i=1}^{T} w_iw_i^\top)>c$ for some constant $c>0$, suppose $W_{\CL}$ is the optimal solution of optimization problem Equation \eqref{trans contrastive task}, and denote its singular value decomposition as $W_{\CL}=(U_{\CL}\Sigma_{\CL}V_{\CL}^\top)^\top$, then we have:
	\begin{equation}\label{CL transfer upper bound2}
		\begin{aligned}
	    	\mathbb{E}\|\sin\Theta(U_{\CL},U^\star)\|_F\lesssim&\frac{\sqrt{r}}{\alpha+1}\qty(\frac{r}{d}\log d+\sqrt{\frac{d}{n}})+T\sqrt{\frac{dr}{m}}.
		\end{aligned}
	\end{equation}
\end{theorem}
The proof is given in Appendix \ref{thm: transfer t>r ap}.

\begin{theorem}\label{thm: downstream risk transfer t>r}
    Suppose the conditions in Theorem \ref{thm: transfer t>r} hold. Then, 
    \begin{align}\label{CL transfer upper bound3}
		\mathbb{E}_{\mathcal{D}}[\inf_{w\in\mathbb{R}^r} \mathbb{E}_{\mathcal{E}}[\ell_r(\delta_{W_{\CL}, w})]-\inf_{w\in\mathbb{R}^r}\mathbb{E}_{\mathcal{E}}[\ell_r(\delta_{U^{\star \top}, w})] \lesssim \frac{\sqrt{r}}{\alpha+1}\qty(\frac{r}{d}\log d+\sqrt{\frac{d}{n}})+T\sqrt{\frac{dr}{m}}.
	\end{align}
\end{theorem}
The proof is given in Appendix \ref{thm: downstream risk transfer t>r ap}.

Theorem~\ref{thm: transfer t>r} and Theorem~\ref{thm: downstream risk transfer t>r} show that in the case where tasks are abundant, as $\alpha$ goes to infinity (corresponding to the case where we use the supervised loss only), the upper bounds in Equations \eqref{CL transfer upper bound2} and \eqref{CL transfer upper bound3} are reduced to $T\sqrt{rd/m}$. This rate can be worse than the $\sqrt{r^3}\log d/d+\sqrt{rd/n}$ rate obtained by self-supervised contrastive learning when $m$ is small. Recall that when the number of tasks is small, labeled data introduce extra bias term $\sqrt{r-T}$ (Theorem~\ref{thm: transfer t<r} and Theorem~\ref{thm: downstream risk transfer t<r}). We note that when the tasks are abundant enough, the harm of labeled data is mainly due to the variance brought by the labeled data. When $m$ is sufficiently large, supervised learning on source tasks can yield a consistent estimation of core features, whereas self-supervised contrastive learning can not. 


In extending our results from regression to the binary classification setting, the only difference is in the label generation process, and we can obtain similar results with some modification of the proofs. 
The corresponding feature recovery bounds of Theorem \ref{thm: transfer t<r} (where the tasks are insufficient) and Theorem \ref{thm: transfer t>r} (where the tasks are abundant) are stated as follows:

\begin{theorem}
	\label{thm: transfer t<r classification}
	In the classification setting where $y^t = \sign(\langle w_t, z \rangle/\nu)$, suppose the conditions in Theorem \ref{thm: transfer t<r} hold and $z$ in the spiked covariance model \eqref{model: spiked covariance} follows a Gaussian distribution, then we have:
    \begin{small}
	\begin{equation}
	\label{CL transfer upper bound classification}
    \begin{aligned}
      \mathbb{E}\|\sin\Theta(U_{\CL},U^\star)\|_F\lesssim&\sqrt{r-T} \left(\frac{r\log d}{d}+\sqrt{\frac{d}{n}}+\alpha T\sqrt{\frac{ d}{m}} \wedge 1 \right) +\sqrt{T}\left(\frac{r\log d}{\alpha d}+\frac{1}{\alpha}\sqrt{\frac{d}{n}}+T\sqrt{\frac{d}{m}}\right).
    \end{aligned}
    \end{equation}
    \end{small}
\end{theorem}

\begin{theorem}
	\label{thm: transfer t>r classification}
	In the classification setting where $y^t = \sign(\langle w_t, z \rangle/\nu)$ , suppose the conditions in Theorem \ref{thm: transfer t>r} hold and $z$ in the spiked covariance model \eqref{model: spiked covariance} follows a Gaussian distribution, then we have:
	\begin{equation}\label{CL transfer upper bound2 classification}
		\begin{aligned}
	    	\mathbb{E}\|\sin\Theta(U_{\CL},U^\star)\|_F\lesssim&\frac{\sqrt{r}}{\alpha+1}\qty(\frac{r}{d}\log d+\sqrt{\frac{d}{n}})+T\sqrt{\frac{dr}{m}}.
		\end{aligned}
	\end{equation}
\end{theorem}

As the feature recovery bounds remain to be the same, the counterpart of the in-domain downstream tasks results, Theorem 
\ref{thm: downstream risk transfer t<r} and Theorem 
\ref{thm: downstream risk transfer t>r}, in this classification setting follows immediately. For the sake of space, we defer the generalized version and proofs in Theorem \ref{thm: transfer t<r ap classification} and Theorem \ref{thm: transfer t>r ap classification} in the appendix.

\section{Numerical Experiments}\label{sec: simulation}
\subsection{Linear Model with Synthetic Data}
\label{sec: synthetic}


To verify our theory, we conducted numerical experiments on the spiked covariance model \eqref{model: spiked covariance} under a linear representation setting and contrastive loss functions defined in \eqref{loss: self contrastive} and \eqref{trans contrastive task}. As we have explicitly formulated the loss function and derived its equivalent form in the main body and appendix, we simply minimize the corresponding loss by gradient descent to find the optimal linear representation $W$. For self-supervised contrastive learning with random masking augmentation, we independently draw the augmentation function by Definition \ref{aug: random masking} and apply them to all samples in each iteration. To ensure convergence, we set the maximum number of iterations for it (typically 10000 or 50000 depending on dimension $d$). 

We report two criteria to evaluate the quality of the representation, in-domain downstream error, and sine distance. To obtain the sine distance for a learned representation $W$, we perform singular value decomposition to get $W=(U\Sigma V^\top)^\top$ and then compute $\|\sin\Theta(U,U^\star)\|_F$. To obtain the in-domain downstream task performance, in the comparison between autoencoders and contrastive learning, we first draw $n$ labeled data from spiked covariance model \eqref{model: spiked covariance} with labels generated as in Section \ref{sec: downstream}, then we train the model by using the data without labels to obtain the linear representation $W$, and learn a linear predictor $w$ using the data with labels and compute the regression error. In the transfer learning setting, we draw some labeled data from the source tasks and additional unlabeled data. The number of labeled data is set to be $m=1000$ and the number of unlabeled data is set to be $n=1000$. Then train with them to obtain the linear representation $W$, and draw labeled data from a new source task to learn a linear predictor $w$ to compute the regression error. In particular, we subtract the optimal regression error obtained by the best representation $U^{\star\top}$ for each regression error and report the difference, or more precisely, the excess risk as in-domain downstream performance.

The results are reported in Fig. \ref{fig: CL and AE} and \ref{fig: theta distance} and Table \ref{tab: downstream trans} and \ref{tab: recover trans}. As predicted by Theorems \ref{thm: recover PCA} and \ref{thm: recover CL}, the feature recovery error and in-domain downstream task risk of contrastive learning decrease as $d$ increases (Fig. \ref{fig: CL and AE}: \textbf{Left}) and as $n$ increases (Fig. \ref{fig: CL and AE}: \textbf{Center}) while that of autoencoders is insensible to the changes in $d$ and $n$. Consistent with our theory, in Fig. \ref{fig: CL and AE}: \textbf{Right}, it is observed that when tasks are not abundant, the transfer performance exhibit a $U$-shaped curve, and the best result is achieved by choosing an appropriate $\alpha$. When tasks are abundant and labeled data are sufficient, the error remains small when we take large $\alpha$. 

\begin{figure}
	\centering
	\includegraphics[width=0.3\linewidth]{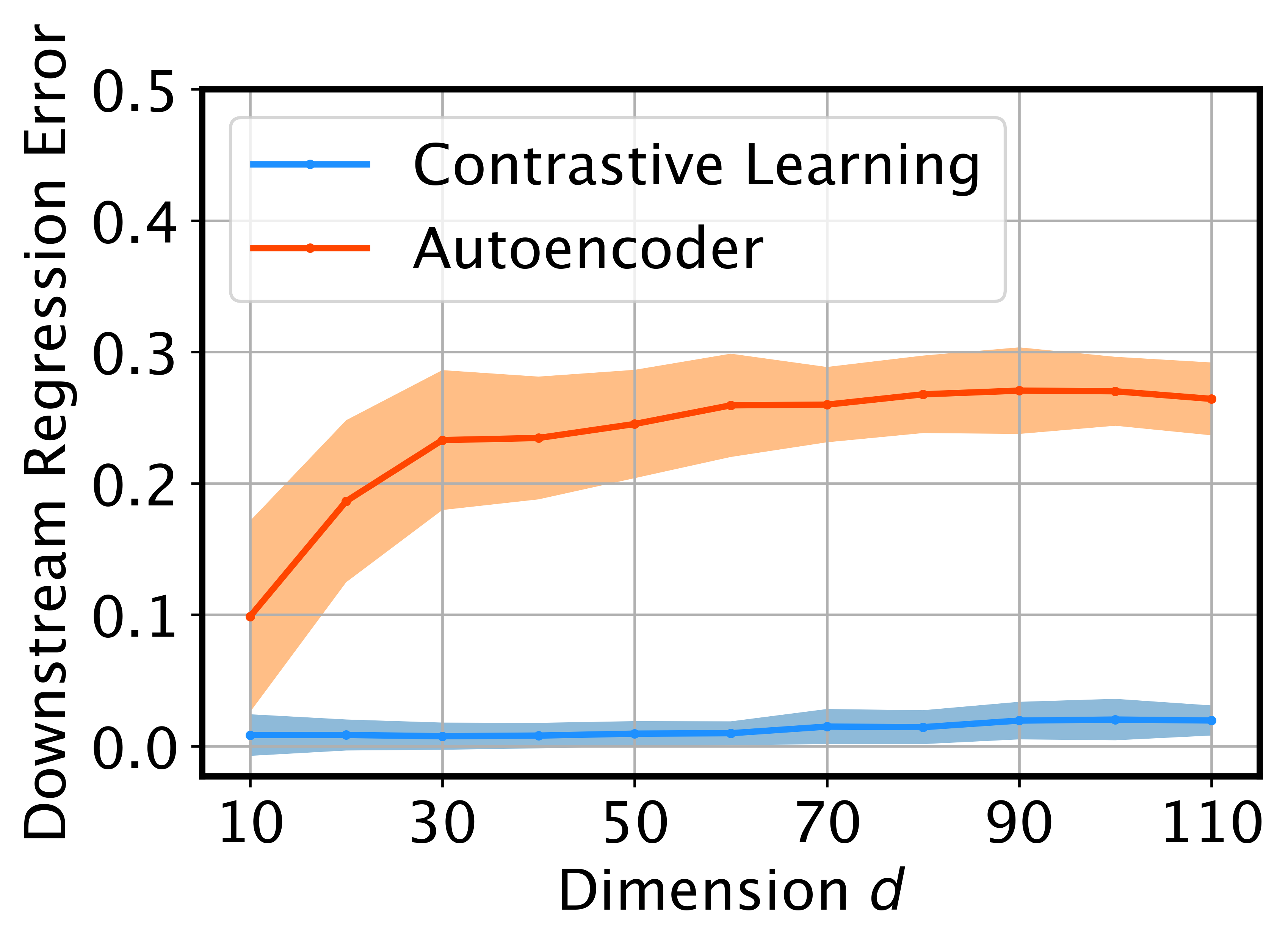}
    \includegraphics[width=0.3\linewidth]{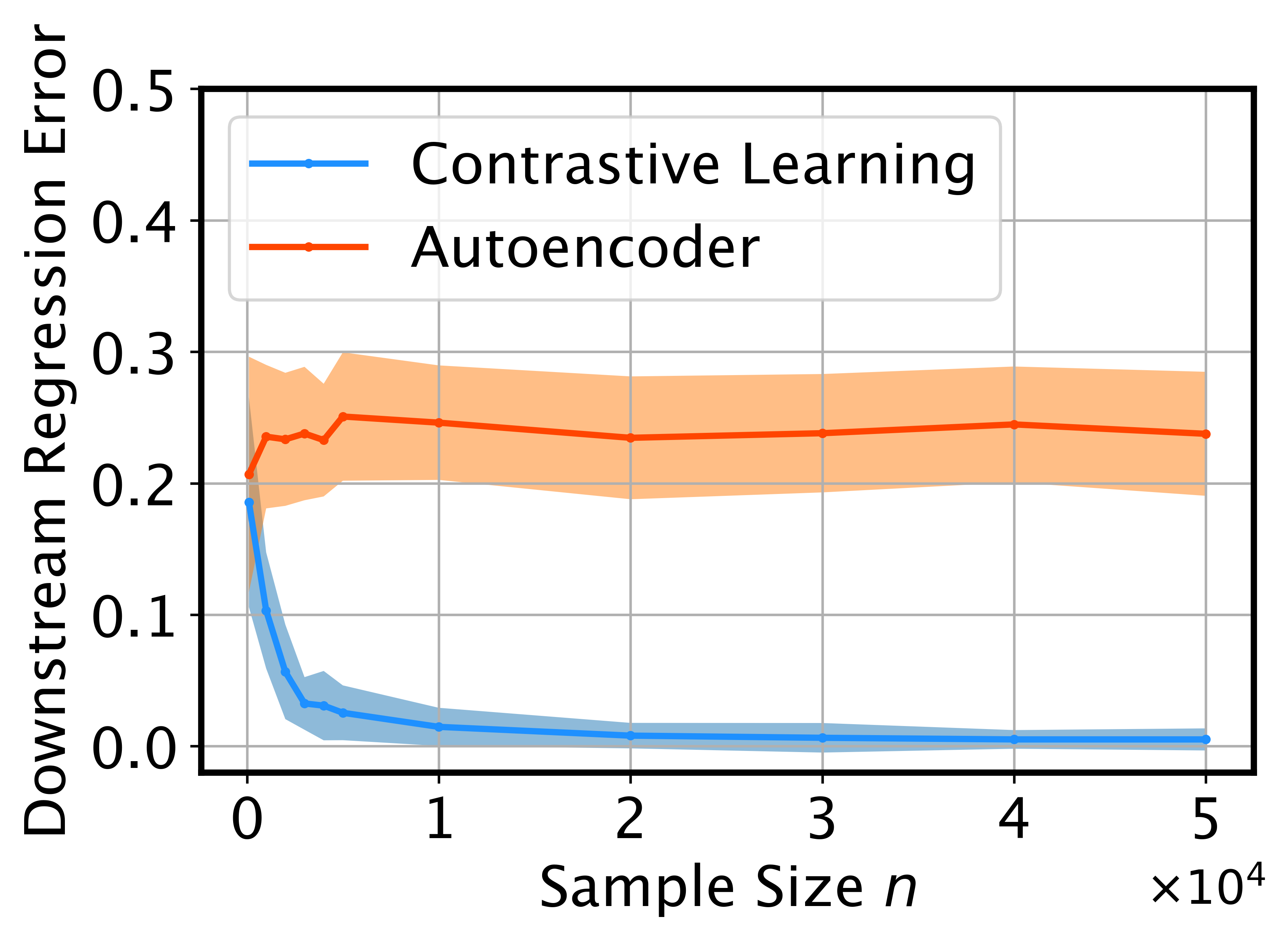}
	\includegraphics[width=0.3\linewidth]{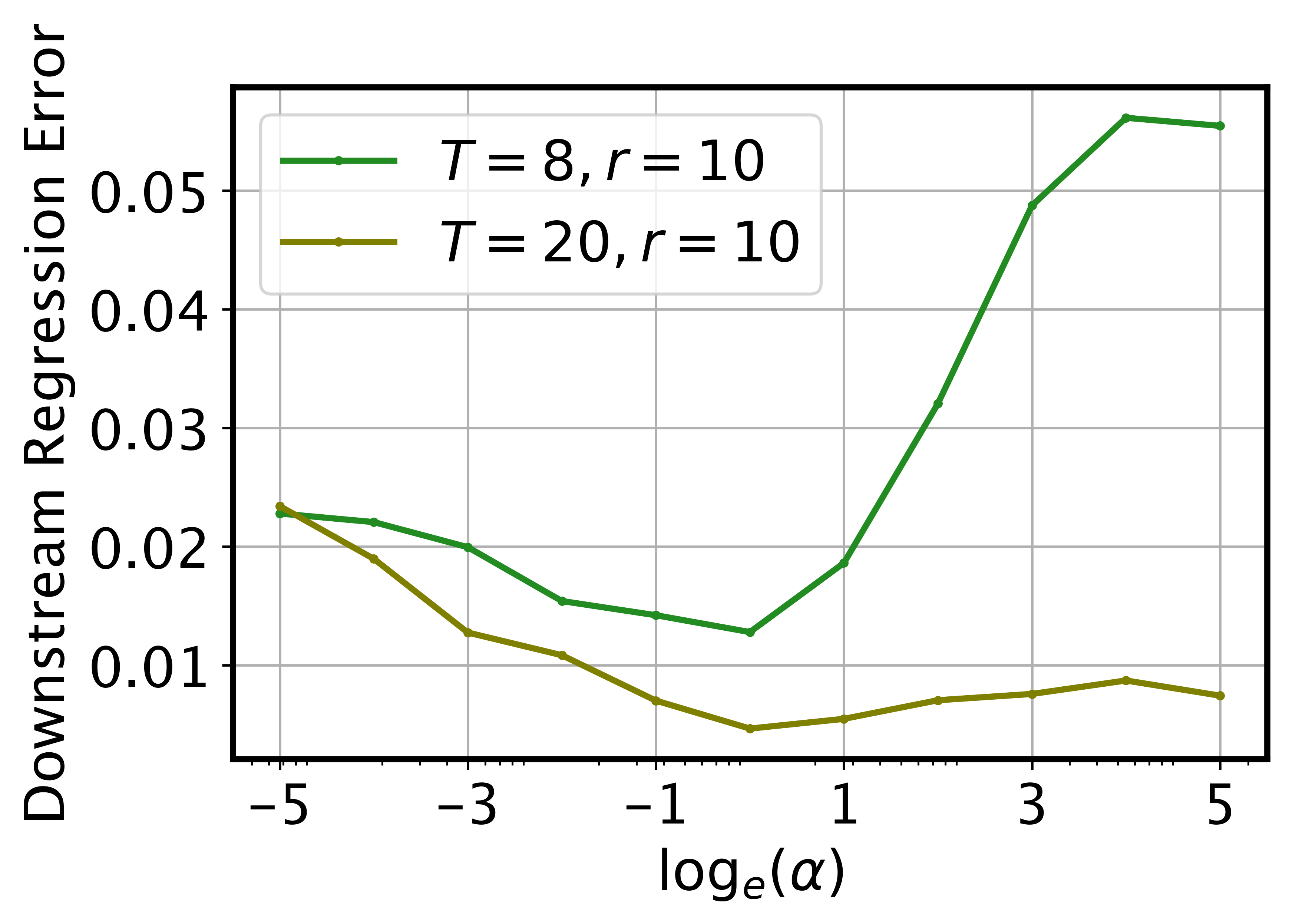}
    \caption{The vertical axes indicate the downstream regression error. We subtract the regression error of the ground truth features to measure the excess error.
    \textbf{Left:} 
    Comparison of in-domain downstream task performance between contrastive learning and autoencoders the dimension $d$. The sample size $n$ is set as $n=20000$. \textbf{Center:} Comparison of in-domain downstream task performance between contrastive learning and autoencoders the dimension $n$. The dimension $d$ is set as $d=40$. \textbf{Right:} In-domain downstream task performance in transfer learning against penalty parameter $\alpha$ in log scale. $T$ is the number of source tasks and $r$ is the dimension of the representation function. We set the number of labeled data and unlabeled data as $m=1000$ and $n=1000$ respectively.\vspace{-1.5em}}
	\label{fig: CL and AE}
\end{figure}

\begin{figure}

	\centering
	\includegraphics[width=0.29\linewidth]{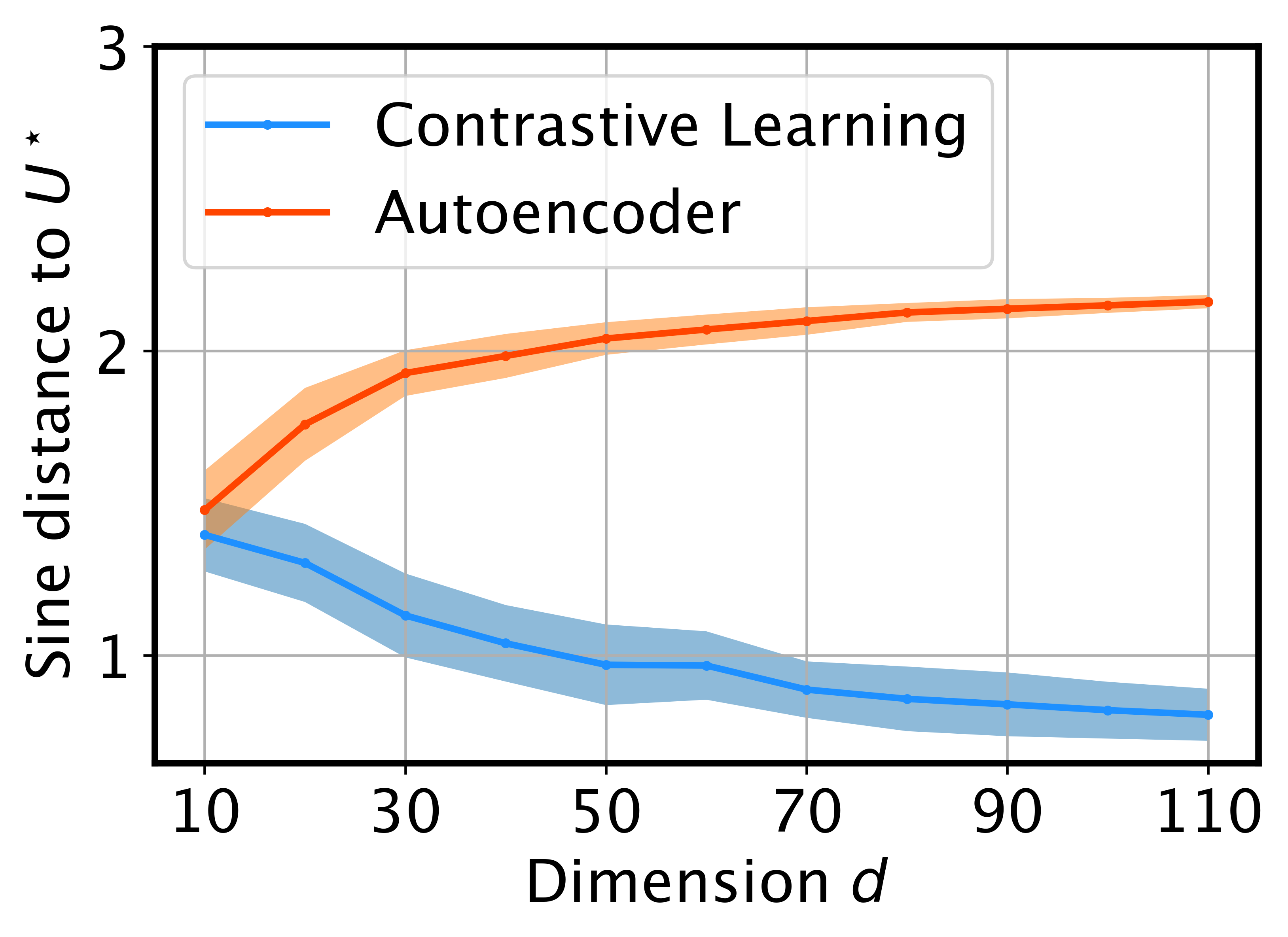}
    \includegraphics[width=0.29\linewidth]{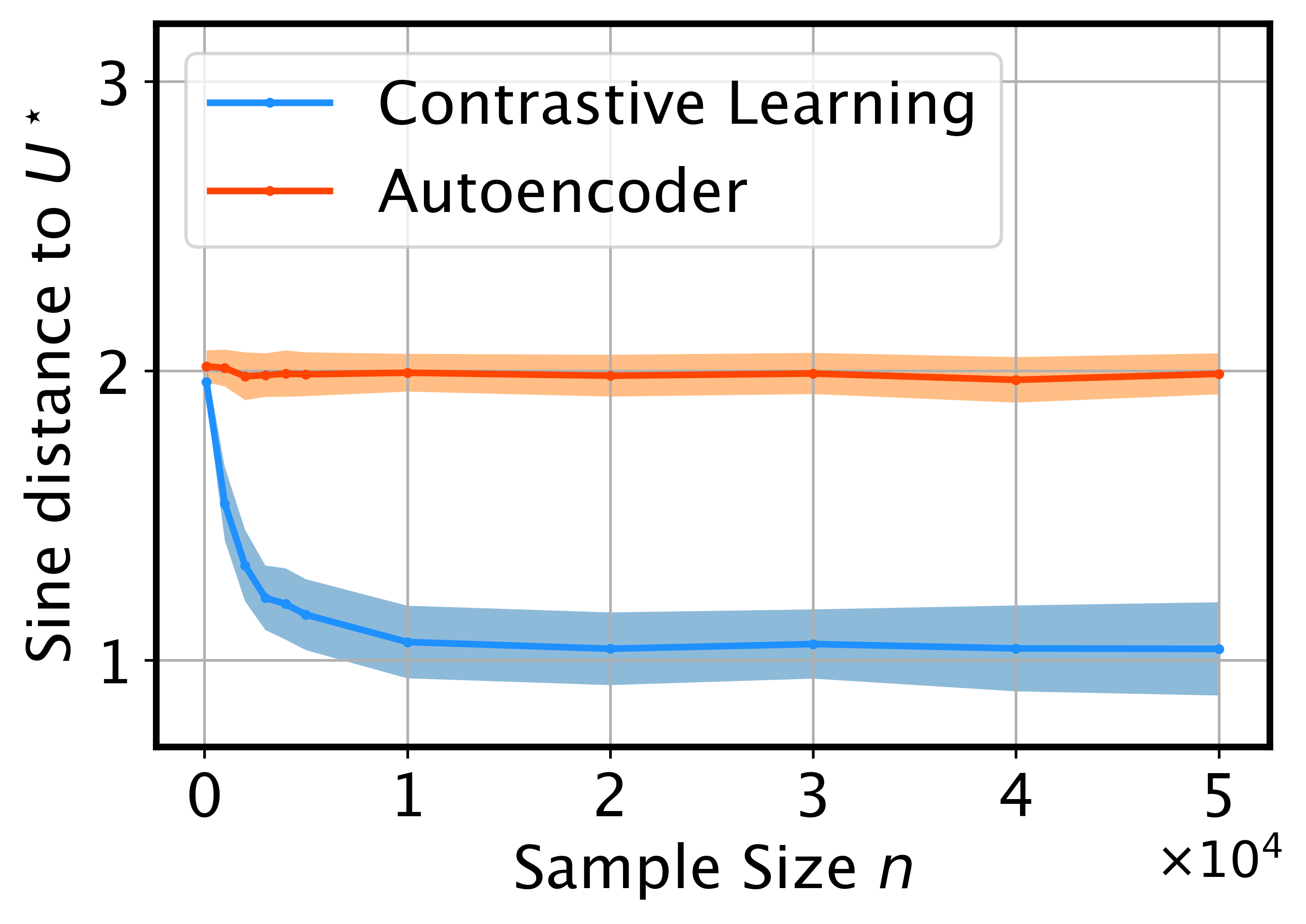}
	\includegraphics[width=0.32\linewidth]{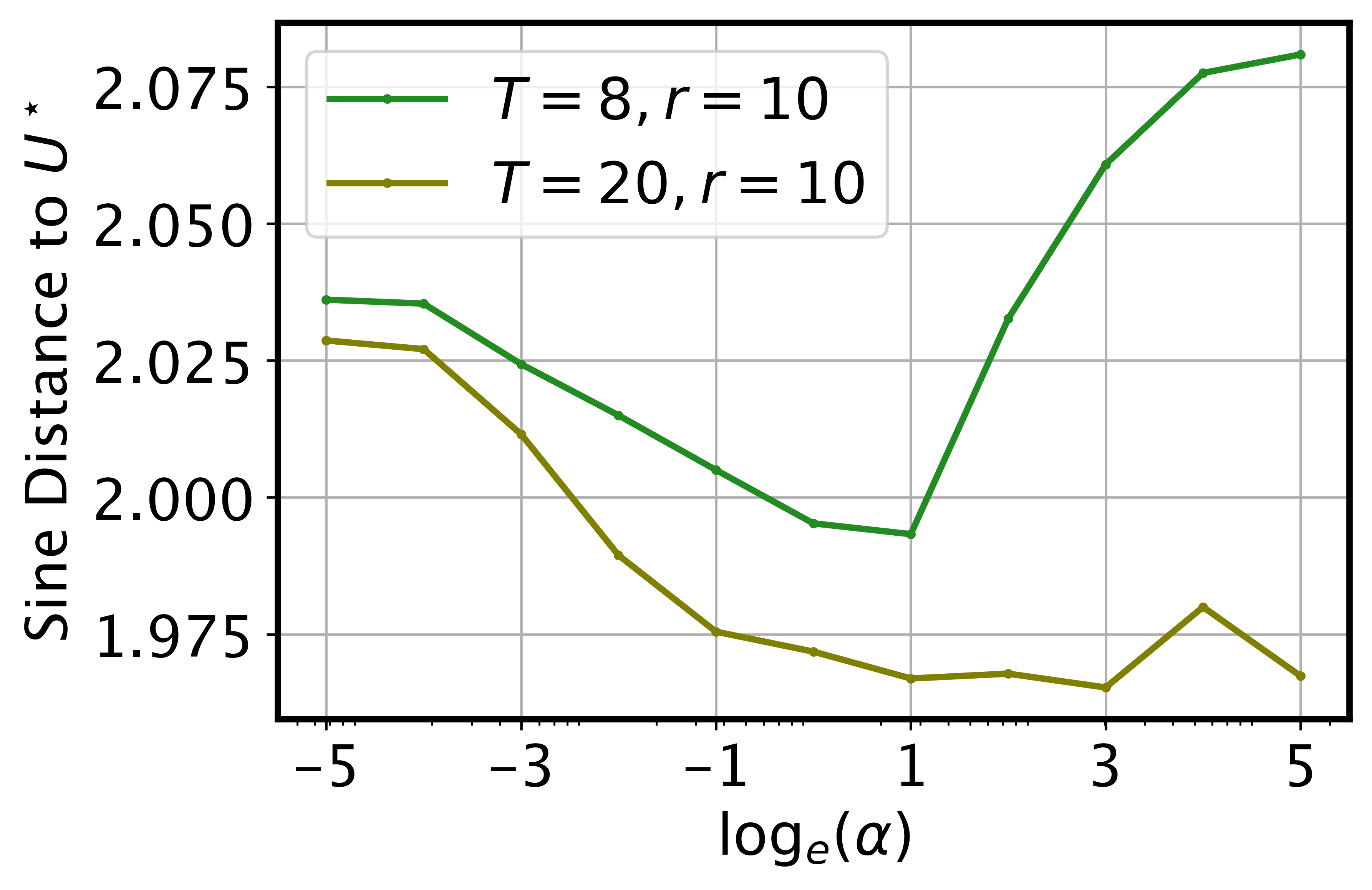}
    \caption{\textbf{Left:} Comparison of learned feature between contrastive learning and autoencoders against the dimension $d$. The sample size $n$ is set as $n=20000$. \textbf{Center:} Comparison of feature recovery performance between contrastive learning and autoencoders against the dimension $n$. The dimension $d$ is set as $d=40$. \textbf{Right:} Feature recovery performance in transfer learning against penalty parameter $\alpha$ in log scale. $T$ is the number of source tasks. We set the number of labeled data and unlabeled data as $m=1000$ and $n=1000$ respectively.}
    \label{fig: theta distance}
\end{figure}
\begin{table}
    \centering
    \resizebox{\columnwidth}{!}{
    \begin{tabular}{ l||cccccccccccc|  }
        \hline
        $\log_e(\alpha)$ & -5 & -4 & -3 & -2 & -1 & 0 & 1 & 2 & 3 & 4 & 5\\
        \hline
        \hline
        $T=8, r=10$    & 0.0242 & 0.0231 & 0.0199 & 0.0141 & 0.0122 & 0.0125 & 0.0184 & 0.0345 & 0.0499 & 0.0535 & 0.0587 \\
        \hline
        $T=20, r=10$   & 0.0223 & 0.0163 & 0.0156 & 0.0096 & 0.0079 & 0.0055 & 0.0064 & 0.0064 & 0.0067 & 0.0070 & 0.0079 \\
        \hline
    \end{tabular}
    }
    \caption{In-domain downstream performance in transfer learning against the penalty parameter $\alpha$. $T$ is the number of source tasks.}
    \label{tab: downstream trans}
\end{table}

\begin{table}
    \centering
    \resizebox{\columnwidth}{!}{
    \begin{tabular}{ l||cccccccccccc|  }
        \hline
        $\log_e(\alpha)$ & -5 & -4 & -3 & -2 & -1 & 0 & 1 & 2 & 3 & 4 & 5\\
        \hline
        \hline
        $T=8, r=10$    & 2.0373 & 2.0371 & 2.0228 & 1.9908 & 2.0021 & 2.0055 & 2.0010 & 2.0362 & 2.0699 & 2.0705 & 2.0813 \\
        \hline
        $T=20, r=10$   & 2.0352 & 2.0292 & 2.0030 & 1.9871 & 1.9740 & 1.9690 & 1.9766 & 1.9702 & 1.9790 & 1.9714 & 1.9672 \\
        \hline
    \end{tabular}
    }
    \caption{Feature recovery performance in transfer learning against the penalty parameter $\alpha$. $T$ is the number of source tasks.}
    \label{tab: recover trans}
\end{table}
\subsection{Neural Nets with Real-World Dataset}
\label{sec: real data}
{In this section, we provide experimental results in real-world datasets to support our theoretical results. Although our model settings and assumptions might be violated under this scenario, as we shall see, our findings still remain valid in practice. }

We conduct the experiments using the datasets STL-10 \citep{Coates2011AnAO} and CIFAR-10 \citep{Krizhevsky2009LearningML} with the neural nets architecture ResNet-18 \citep{He2016DeepRL}. Our experiments are carried out based on linear evaluation following SimCLR \citep{chen2020simple}, where we first train a ResNet-18 encoder and a two-layer MLP projector with unlabeled augmented data. We then freeze the encoder, train a logistic regression on top of it with labeled data, and lastly evaluate the performance on the test data. Following \citet{chen2020simple}, we apply augmentations including resized cropping, horizontal flipping, color distortion, and Gaussian blurring to generate the augmented data, and use the InfoNCE loss function to train the network. All training is carried out with the Adam optimizer \citep{Kingma2015AdamAM}, batch size 256, learning rate $3\times 10^{-4}$, weight decay $10^{-4}$, and a cosine annealing learning rate scheduler for 100 epochs. Our codes are implemented in Pytorch and run on an NVIDIA V100 GPU.

\paragraph{Contrastive learning v.s. standard autoencoders:}
Here we provide real-world evidence for our theoretical findings in Section \ref{sec: recover} by comparing the performance of contrastive learning versus a standard autoencoder. The architecture of encoders is the same for these two methods, and we use an inversed ResNet-18 as the decoder. During the training time, we use the encoder-decoder architecture and mean squared error loss to reconstruct the input, and then we train a linear classifier on the features learned by the encoder. The results are listed in Table \ref{tab: simclr autoencoder}, we can find that contrastive learning demonstrates superior performance over the standard autoencoder.
\begin{table}[H]
    \centering
    \begin{tabular}{ l||cc  }
        \hline
        Testing Accuracy & Contrastive Learning & Standard Autoencoder\\
        \hline
        \hline
        CIFAR10& $65.11\pm0.51$&$44.76\pm0.16$ \\
        STL10& $71.02\pm0.47$&$39.00\pm0.58$\\
        \hline
    \end{tabular}
    \caption{Comparison of linear evaluation performance of contrastive learning and autoencoders.}
    \label{tab: simclr autoencoder}
\end{table}
\paragraph{The impact of labeled data in transfer learning}
Now we we provide real-world evidence for our theoretical findings in Section \ref{sec: transfer}. Following the joint optimization formulation in equation \ref{trans contrastive task}, we combine the InfoNCE loss function for unlabeled data and cross-entropy loss for labeled data with a ratio $\alpha$. For both STL-10 and CIFAR-10 datasets, we divide the test data into two sets, one consists of the first five classes and the other one consists of the remaining five classes. During training, we use the training data as unlabeled data and the first set of test data as the labeled data to train the model jointly, and then train a linear classifier with the second set of test data on features learned by the encoder. As predicted by Theorem \ref{thm: transfer t<r}, when $\alpha$ is small, introducing labeled data from the first five classes would be beneficial to learn better representations and improve the performance on the last five classes; when $\alpha$ is large, labeled data from the first five classes will make the model only focus on features that are useful to discriminate the first five classes and ignore other features, thus introducing labeled data could be harmful to the performance on the last five classes. Testing accuracy on the last five classes with different $\alpha$ are listed in table \ref{tab: simclr transfer}, it is observed that the accuracy first increases and then decreases as $\alpha$ grows, which is consistent with our theoretical results.
\begin{table}[H]
    \centering
    \resizebox{\columnwidth}{!}{
    \begin{tabular}{ l||ccccccccccc  }
        \hline
        Testing Accuracy & $\alpha=0.0$& $\alpha=0.1$& $\alpha=0.2$& $\alpha=0.3$& $\alpha=0.4$& $\alpha=0.5$& $\alpha=0.6$&$\alpha=0.7$& $\alpha=0.8$& $\alpha=0.9$& $\alpha=1.0$\\
        \hline
        \hline
        CIFAR10 & 74.27&75.52&74.86&75.31&75.21&74.86&74.46&73.85&72.20&69.17&51.31\\
        \hline
        STL10 &
        82.56&83.54&83.27&83.24&83.21&83.03&82.34&82.11&80.94&76.88&52.37  \\
        \hline
    \end{tabular}
    }
    \caption{\textbf{Transfer learning performance with different $\alpha$}. For each experiment, we report the average accuracy for three independent runs.}
    \label{tab: simclr transfer}
\end{table}
\section{Conclusion}\label{sec: conclusion}
In this work, we  establish a theoretical framework to study contrastive learning under the linear representation setting. 
We theoretically prove that contrastive learning, compared with autoencoders and GANs, can obtain a better low-rank representation under the spiked covariance model, which further leads to better performance in in-domain downstream tasks. We also highlight the impact of labeled data in supervised contrastive learning and multi-task transfer learning: labeled data can reduce the domain shift bias in contrastive learning, but it harms the learned representation in transfer learning.
To our knowledge, our result is the first theoretical result to guarantee the success of contrastive learning by comparing it with existing representation learning methods. However, to get a tractable analysis, like many other theoretical works in representation learning \citep{du2020few,lee2021predicting,tripuraneni2021provable}, our work starts with linear representations, which still provides important insights. Recently, \citet{wen2021toward} and \citet{refinetti2022dynamics} studied the training dynamics of autoencoders and contrastive learning with nonlinear shallow neural networks. Extending our results to these more complex models is an interesting direction for future work.
\section*{Acknowledgement}
L.Z. is supported by National Science Foundation DMS 2015378. J.Z. is supported
by the National Science Foundation (CCF 1763191 and CAREER 1942926), the US National Institutes of Health (P30AG059307 and U01MH098953) and grants from the Silicon Valley Foundation and the Chan-Zuckerberg Initiative.


\newpage

\appendix

\section{Background and omitted discussion}
\subsection{Comparison with other works}\label{section: comparison}
Here we compare the results in this paper with some closely related works. 

To rigorously analyze contrastive learning, we consider the random masking augmentation strategy which is also analyzed in \citet{wen2021toward}. In \citet{wen2021toward}, the authors aim to understand the training dynamics of contrastive learning in a shallow nonlinear neural network and focus more on dealing with nonlinearity. In comparison, our work focus on the comparison between contrastive learning and autoencoders and the role of label information in contrastive learning. To make the problem mathematically tractable, we adopt a linear model, which is simple but enough to shed light on many mysterious phenomena in practice.  Moreover, while they assume a sparse coding model, where the features are extremely sparse, and Gaussianity of signals and noise, our analysis only requires that the features are sub-Gaussian \eqref{model: spiked covariance}. Furthermore, our technique allows the signal-to-noise ratio to have different orders, as long as it decreases slowly, while their analysis is restricted to a particular signal-to-noise ratio. 

\citet{tian2022deep} studied the relationship between contrastive learning and PCA from a game-theoretic point of view. 
Specifically, the authors decompose the gradient descent on the contrastive loss into two dynamics, namely the max-player and min-player. It is proven that in deep linear networks, the max-player is equivalent to PCA and the landscape has no spurious minimum. While the results on max-player can be applied to a family of contrastive loss, it is still difficult to analyze the min-player in a general setting. In our paper, we use a linear contrastive loss \eqref{triplet loss} to explicitly obtain the features learned by contrastive learning. 
Moreover, our results can be directly extended to a deep linear network setting by the equivalence of a single linear transformation and a deep linear network.
The major difference is the non-convexity of the loss landscape.

\citet{garg2020functional} studied the combination of supervised learning and self-supervised learning. They viewed training with unlabeled data as functional regularization on learning the representation function, and obtained sample complexity bounds in a PAC-learning style for various settings. In particular, they found that such functional regularization can help to reduce the amount of labeled data needed, and showed autoencoders and masked self-supervision as two concrete examples. Apart from \citet{garg2020functional}, this paper focuses on a regime in combining self-supervised learning and supervised learning, where a trade-off between labeled data and unlabeled data exists. 
Specifically, Theorem 3 in \citet{garg2020functional} assumes that a ground truth representation exists such that it can keep both self-supervised loss and supervised loss at a very low threshold. 
However, as the authors admit, it is hard to determine such a threshold in practical settings. For example, since the unlabeled data and labeled data come from different domains, such as Image-Net and CIFAR-10, domain-specific features may have a much lower loss compared with domain-transferable features. 
In our paper, we first study the regime where tasks are not diverse enough in Theorem \ref{thm: transfer t<r} (which corresponds to the case where ground truth does not exist) and show the trade-off between supervised loss and self-supervised loss.
Then in Theorem \ref{thm: transfer t>r} we show that when tasks are abundant (which corresponds to the case where ground truth exists), labeled data helps to achieve better error bounds, which is similar to the result of \citet{garg2020functional}. 
Our result of Theorem \ref{thm: transfer t<r} provides novel insight into the regime where tasks are not diverse, which has been left untouched in the literature.

In \citet{li2021self}, the authors proposed a novel self-supervised loss function based on HSIC and discussed the relationship between InfoNCE and the proposed SSL-HSIC loss. The SSL-HSIC loss measures the dependence between the output features and one-hot encoded labels(which serve as the indicators of positive samples) and minimizing the SSL-HSIC loss encourages the network to discriminate augmented views from different samples. In comparison to this self-supervised loss, we use HSIC in Section \ref{sec: transfer} as a supervised loss to measure the dependence between output features and the true labels, which is a common usage of HSIC in previous works \citep{barshan2011supervised,song2007supervised}. Moreover, we want to point out that the proposed estimator of SSL-HSIC (see Equation (11) in \citet{li2021self}) can be reduced to the linear loss we use in this paper when the kernel $k(\cdot,\cdot)$ is chosen to be a simple inner product. The authors argued that the standard InfoNCE loss may yield meaningless features in some cases thus the proposed HSIC-based loss could be a better alternative, and provided empirical results illustrating the comparable performance of SSL-HSIC. It remains to be further explored the benefits of such an HSIC-based method. 

\cite{lee2021predicting} studied self-supervised learning under a conditional independence assumption, and showed that with the optimal representation learned in pretext tasks, the in-domain downstream risk is guaranteed to be small. In the contrastive learning context, for example, an image classification downstream task, such an assumption implies that the augmented views generated from the same picture are roughly independent conditional on its ground-truth class, which could be too strong since two views are usually strongly correlated. Such an assumption would be closer to supervised contrastive learning as we have discussed in Section \ref{sec: classification} where we contrast two independent samples from the same class, such sample pairs are independent conditional on the true label, but it requires label information and thus not applied to self-supervised contrastive learning. Compared with this work, we studied a specific data-generating model where the two views are obtained by practical augmentation and thus could be more close to a real-world setting. It is also remarkable that in\cite{lee2021predicting}, their analysis can be adapted to the nonlinear representation setting in the sense that they directly assumed that the optimal representation is obtained. However, the representation learning in the contrastive learning context, even under a linear representation setting, could be non-convex. Thus our analysis starts from a simple setting to obtain a deep understanding of what could be learned in the pretext tasks.

\cite{saunshi2022understanding} proposed a novel perspective that theoretical analysis of contrastive learning must take the inductive bias into account. It is shown that without considering the function class, it is possible that the learned features totally fail in downstream classification tasks. Furthermore, it is shown that within the linear representation class, contrastive self-supervised learning is guaranteed to learn meaningful features under certain conditions. In our paper, we have restricted ourselves to a similar linear representation setting and avoided the collapse raised by complicated models. Compared with their analysis in a linear representation setting, our bounds provide an exact order while their results still need to quantify the expressivity and inconsistency measure, which is very difficult without very strong assumptions. Moreover, their analysis requires a finite input space and augmentation set, which is much stronger than our data-generating model. 

Later than our paper, \cite{fu2022details} considered a similar issue of transfer learning performance of the supervised contrastive method. The authors argue that directly minimizing the supervised contrastive method will lead to feature collapse, which implies that within each class, all data points will have the same embedding and thus the supervised contrastive method loses information that could be useful for transfer learning. This intuition is similar to our setting in Section \ref{sec: transfer}, and based on it the authors proposed a new loss function that combines the supervised contrastive loss and within-class self-supervised loss together. At a high level, this new loss function is quite similar to equation \ref{trans contrastive task} since they are both linear interpolations of self-supervised contrastive loss and supervised loss, and the common motivation is to encourage the model to learn more background features. And our analysis in Section \ref{sec: transfer} can provide a theoretical foundation for when could such interpolation work and how to choose the ratio $\alpha$.

\subsection{Disucssion about the regularization term}
\label{section: regularization}
In this paper we use a quadratic regularization term $R_1(W)=\|WW^\top\|_F^2$ instead of a standard $\ell_2$ regularization term $R_2(W)=\|W\|_F^2$. Denote $W^T=[w_1,\cdots,w_r]$, then we can write these two terms as:
    $$
    R_1(W)=\sum_i^r\|w_i\|^4+\sum_{i,j}\langle w_i,w_j\rangle^2, \quad R_2(W)=\sum_i^r\|w_i\|^2
    $$
    The main difference between these two terms is that except for penalizing the norm of representation $w_i$, it also penalizes the similarity between different representations. In particular, in the linear representation setting, where we will deal with optimization problems like $$
    \min_{W\in\mathbb{R}^{r\times d}} \tr(W AW^\top)+\frac{\lambda}{2}R(W)
    $$
    where $A$ is a symmetric matrix determined by data and augmentation, we can easily find that the $\ell_2$ regularization would fail. To see this, we can rewrite the loss function as
    $$
    \tr(W AW^\top) +\frac{\lambda}{2}\|W\|_F^2=\sum_{i=1}^r (w_i^\top A w_i +\frac{\lambda}{2}\|w_i\|^2)=\sum_{i=1}^r w_i^\top(A +\frac{\lambda}{2}I)w_i, 
    $$
    it is easy to find that the optimal solution of each $w_i$ would be at infinity. Moreover, even if we add constraints like $\|w_i\|<C,\forall i\in [r]$, the optimal solution of each $w_i$ would all be the eigenvector corresponding to the smallest eigenvalue of $A$ thus the model would only learn a single representation from the data. In contrast, the quadratic regularization term encourages the diversity of representation by penalizing the similarity between representations, i.e., $\langle w_i,w_j \rangle^2$. In this situation, we have:
    $$
    \tr(W AW^\top) +\frac{\lambda}{2}\|WW^\top\|_F^2= \frac{\lambda}{2}\|W^\top W+\frac{1}{\lambda}A\|_F^2-\frac{1}{2\lambda}\|A\|_F^2
    $$
    and it is easy to find that the optimal solution of $w_i$ would be finite and each $w_i$ corresponds to different eigenvectors of $A$ and is orthogonal to each other, which implies that they would learn totally different representations. As a result, the quadratic regularization term would be more helpful to self-supervised learning with linear representation. In real-world practice, $\ell_2$ regularization can still work well with the help of non-linearity and normalization techniques, but we also provide an empirical observation that applying quadratic regularization would be helpful to improve the performance compared with using standard weight decay. 
    We would also like to point out that in the linear regime, the choice of $\lambda\in\mathbb{R}_+$ would only affect the norm of $w_i$ and makes no difference to the direction of $w_i$ and the quality of representation, thus we do not specify this value in our analysis.
    Similar regularization techniques are also used in \citet{liu2021self} for theoretical analysis in the linear representation setting.
    
    We verify this conjecture in neural networks, and the results are provided in Table \ref{tab: simclr reg}. The first column corresponds to training with standard weight decay, and in the setting of the second column, we do not apply weight decay for each weight matrix of fully connected layers in the encoder, and add an additional regularization term $\lambda\|W^TW\|_F^2$ on the loss function instead. Note that the quadratic regularization does not apply to convolution layers and bias terms, thus we keep the weight decay on these parameters. We search the regularization parameter in each settings from $\lambda=0.1,0.01,0.001,0.0001,0.00001,0.000001$ and find that $\lambda=0.0001$ yields the best performance in each settings and datasets.
    
    It is observed that quadratic regularization slightly improves the performance of contrastive learning, which is consistent with our intuition.
    \begin{table}
        \centering
        \begin{tabular}{ l||ccc|  }
            \hline
            Testing Accuracy & weight decay=0.0001 &quadratic regularization=0.0001\\
            \hline
            \hline
            CIFAR10    & $65.11\pm0.51$ & $\mathbf{65.54\pm0.26}$ \\
            \hline
            STL10   &  $71.02\pm 0.47$ & $\mathbf{71.39\pm0.39}$ \\
            \hline
        \end{tabular}
        \caption{\textbf{Quadratic regularization v.s. weight decay}. We compare the top-1 accuracy of linear classifiers trained on features learned by SimCLR with the ResNet-18 encoder and different regularization methods. We repeat each experiment for 5 runs and report the mean and standard deviation. More details are provided in Section \ref{sec: real data}}
        \label{tab: simclr reg}
    \end{table}

\subsection{Background on distance between subspaces}
\label{sec: distance}
In this section, we will provide some basic properties of sine distance between subspaces. Recall the definition:
\begin{equation}
\label{sin theta distance ap}
    \left\|\sin \Theta\left(U_{1}, U_{2}\right)\right\|_F \triangleq\left\|U_{1 \perp}^\top U_{2}\right\|_F=\left\|U_{2 \perp}^\top U_{1}\right\|_F.
\end{equation}
where $U_1,U_2\in\mathbb{O}_{d, r}$ are two orthogonal matrices. Similarly, we can also define:
\begin{equation*}
    \left\|\sin \Theta\left(U_{1}, U_{2}\right)\right\|_2 \triangleq\left\|U_{1 \perp}^\top U_{2}\right\|_2=\left\|U_{2 \perp}^\top U_{1}\right\|_2.
\end{equation*}
We first give two equivalent definitions of this distance: 
\begin{proposition}
\label{prop: distance 1}
    $$\left\|\sin \Theta\left(U_{1}, U_{2}\right)\right\|_F^2 = r-\left\|U_{1}^\top U_{2}\right\|_F^2$$
\end{proposition}
\begin{proof}
    Write $U=[U_1,U_{1 \perp}]\in\mathbb{O}_{d, d}$. We have
    \begin{equation*}
        r=\|U_2\|_F^2=\|U^\top U_2\|_F^2 = \left\|U_{1 \perp}^\top U_{2}\right\|_F^2+\left\|U_{1}^\top U_{2}\right\|_F^2,
    \end{equation*}
    then by definition of sine distance, we can obtain the desired equation.
\end{proof}
\begin{proposition}
\label{prop: distance 2}
$$
\left\|\sin \Theta\left(U_{1}, U_{2}\right)\right\|_F^2=\frac{1}{2}\|U_1U_1^\top-U_2U_2^\top\|_F^2
$$
\end{proposition}
\begin{proof}
    Expand the right hand and use Proposition \ref{prop: distance 1} we have:
    \begin{equation*}
    \begin{aligned}
        \frac{1}{2}\|U_1U_1^\top-U_2U_2^\top\|_F^2=&\frac{1}{2}(\|U_1U_1^\top\|_F^2+\|U_2U_2^\top\|_F^2-2\tr(U_1U_1^\top U_2U_2^\top))\\
        =&\frac{1}{2}(r+r-2\tr(U_1^\top U_2U_2^\top U_1))\\
        =&r-\|U_1^\top U_2\|_F^2=\left\|\sin \Theta\left(U_{1}, U_{2}\right)\right\|_F^2.
    \end{aligned}
    \end{equation*}
\end{proof}
With Propositions \ref{prop: distance 1} and \ref{prop: distance 2}, it is easy to verify its properties to be a distance function. Obviously, we have $0\leq\left\|\sin \Theta\left(U_{1}, U_{2}\right)\right\|_F\leq \sqrt{r}$ and $\left\|\sin \Theta\left(U_{1}, U_{2}\right)\right\|_F=\left\|\sin \Theta\left(U_{2}, U_{1}\right)\right\|_F$ by definition. Moreover, we have the following results:

\begin{lemma}[Lemma 1 in \citet{cai2018rate}]\label{lem: bound 2 norm by sin distance}
    For any $U, V \in \mathbb{O}_{d,r}$, 
    \begin{align}
        \|\sin\Theta(U, V)\|_2 \leq \inf_{O \in \mathbb{O}_{r,r}}\|U O - V\|_2 \leq \sqrt{2} \|\sin\Theta(U, V)\|_2,
    \end{align}
    and
    \begin{align}
        \|\sin\Theta(U, V)\|_F \leq \inf_{O \in \mathbb{O}_{r,r}}\|U O - V\|_F \leq \sqrt{2} \|\sin\Theta(U, V)\|_F.
    \end{align}
\end{lemma}

\begin{proposition}[Identity of indiscernibles]
    $$
    \left\|\sin \Theta\left(U_{1}, U_{2}\right)\right\|_F=0\Leftrightarrow \exists O\in\mathbb{O}^{r\times r}, \text{ s.t. } U_1O=U_2
    $$
\end{proposition}
\begin{proof}
    It is a straightforward corollary by definition:
    \begin{equation*}
    \begin{aligned}
        \left\|\sin \Theta\left(U_{1}, U_{2}\right)\right\|_F=0&\Leftrightarrow\left\|U_{1 \perp}^\top U_{2}\right\|_F=0\Leftrightarrow U_{2\perp}\perp U_1\\
        &\Leftrightarrow \exists O\in\mathbb{O}^{r\times r}, \text{ s.t. } U_1O=U_2.
    \end{aligned}
    \end{equation*}
\end{proof}
\begin{proposition}[Triangular inequality]
\label{triangular inequality of sin theta}
    $$\left\|\sin \Theta\left(U_{1}, U_{2}\right)\right\|_F\leq\left\|\sin \Theta\left(U_{1}, U_{3}\right)\right\|_F+\left\|\sin \Theta\left(U_{2}, U_{3}\right)\right\|_F$$
\end{proposition}
\begin{proof}
    By the triangular inequality for Frobenius norm we have:
    \begin{equation*}
        \|U_1U_1^\top-U_2U_2^\top\|_F\leq\|U_1U_1^\top-U_3U_3^\top\|_F+\|U_2U_2^\top-U_3U_3^\top\|_F,
    \end{equation*}
    then apply Proposition \ref{prop: distance 2} to replace the Frobenius norm with sine distance we can finish the proof.
\end{proof}

\section{Omitted proofs for Section \ref{sec: unsupervised}}

\subsection{Proofs for Section \ref{sec: connection to autoencoders} and Section \ref{sec: connection to cl}}

In this section, we will provide the proof of Proposition \ref{prop: augment} and Corollary \ref{prop: diagonal contrast}, the restatement of them and the detailed proof can be found in Proposition \ref{prop: augment ap} and Corollary \ref{prop: diagonal contrast ap}.
\begin{proposition}[Restatement of Proposition \ref{prop: augment}]
	\label{prop: augment ap}
	For two fixed augmentation functions $g_1,g_2:\mathbb{R}^d\rightarrow\mathbb{R}^d$, denote the augmented data matrices as $X_1=[g_1(x_1),\cdots,g_1(x_n)]\in\mathbb{R}^{d\times n}$ and $X_2=[g_2(x_1),\cdots,g_2(x_n)]\in\mathbb{R}^{d\times n}$, when the augmented pairs are generated as in Definition \ref{pair: augmented}, the optimal solution of contrastive learning problem \eqref{opt: CL ap} is given by:
	\begin{equation*}
		W_{\CL} = C\left(\sum_{i=1}^{r}u_i\sigma_i v_i^\top\right)^\top,
	\end{equation*}
	where $C>0$ is a positive constant, $\sigma_i$ is the $i$-th largest eigenvalue of the following matrix:
	\begin{equation}
	\label{CL mat}
		X_1X_2^\top+X_2X_1^\top-\frac{1}{2(n-1)}(X_1+X_2)(1_r 1_r^\top-I_r)(X_1+X_2)^\top,
	\end{equation}
	$u_i$ is the corresponding eigenvector and $V=[v_1,\cdots,v_r]\in\mathbb{R}^{r\times r}$ can be any orthonormal matrix.
\end{proposition}
\begin{proof}[Proof of Proposition \ref{prop: augment}]
\label{proof augment}
	When augmented pairs generation in Definition \ref{pair: augmented} is applied, the contrastive loss can be written as:
	\begin{small}
	\begin{equation*}
		\begin{aligned}
			\mathcal{L}_{\text{SelfCon}}(W) =& \frac{\lambda}{2}\|WW^\top\|_F^2-\frac{1}{n}\sum_{i=1}^n[\langle Wg_1(x_i),Wg_2(x_i)\rangle\\&-\frac{1}{4(n-1)}\sum_{j\neq i} \langle Wg_1(x_i)+Wg_2(x_i),Wg_1(x_j)+Wg_2(x_i)\rangle]\\
			 =&\frac{\lambda}{2}\|WW^\top\|_F^2-\frac{1}{n}\sum_{i=1}^n\langle Wg_1(x_i),Wg_2(x_i)\rangle\\&+\frac{1}{4n(n-1)}\sum_{i=1}^n\sum_{j\neq i} \langle Wg_1(x_i)+Wg_2(x_i),Wg_1(x_j)+Wg_2(x_i)\rangle\\
			 =&\frac{\lambda}{2}\|WW^\top\|_F^2-\frac{1}{2n}\tr(X_1^\top W^\top WX_2+X_2^\top W^\top WX_1)\\
			&+\frac{1}{4n(n-1)}\tr((1_n 1_n^\top-I_n)(X_1+X_2)^\top W^\top W(X_1+X_2))\\
			 =& \frac{\lambda}{2}\|WW^\top\|_F^2\\
			 &-\frac{1}{2n}\tr((X_2X_1^\top+X_1X_2^\top-\frac{1}{2(n-1)}(X_1+X_2)(1_n 1_n^\top-I_n)(X_1+X_2)^\top)W^\top W)\\
			 =& \frac{1}{2}\biggl\|\lambda W^\top W-\frac{1}{2n\lambda}\qty(X_2X_1^\top+X_1X_2^\top-\frac{1}{2(n-1)}(X_1+X_2)(1_n 1_n^\top-I_n)(X_1+X_2)^\top)\biggr\|_F^2\\
			 &-\biggl\|\frac{1}{2n\lambda}\qty(X_2X_1^\top+X_1X_2^\top-\frac{1}{2(n-1)}(X_1+X_2)(1_n 1_n^\top-I_n)(X_1+X_2)^\top)\biggr\|_F^2.
		\end{aligned}
	\end{equation*}
	\end{small}
	Note that the last term only depends on $X$, and the first term implies that when $W_{\CL}$ is the optimal solution, $\lambda W_{\CL}W_{\CL}^\top$ is the best rank-$r$ approximation of $\frac{1}{(n-1)\lambda}XHX^\top$, where $H := 1_n 1_n^\top - I_n$. Applying Lemma \ref{Lemma: best rank r approximation} to the first term, we can conclude that $W_{\CL}$ satisfies the desired conditions.
\end{proof}
\begin{corollary}[Restatement of Corollary \ref{prop: diagonal contrast}]
	\label{prop: diagonal contrast ap}
	Under the same conditions as in Proposition \ref{prop: augment}, if we use random masking (Definition \ref{aug: random masking}) as our augmentation function, then in expectation over the data augmentation, the optimal solution of contrastive learning problem \eqref{opt: CL ap} is given by:
	\begin{equation*}
		W_{\CL} = C\left(\sum_{i=1}^{r}u_i\sigma_i v_i^\top\right)^\top,
	\end{equation*}
	where $C>0$ is a positive constant, $\sigma_i$ is the $i$-th largest eigenvalue of the following matrix:
	\begin{equation}
	\label{dd mat}
		\Delta(XX^\top)-\frac{1}{n-1}X(1_n 1_n^\top-I_n)X^\top,
	\end{equation}
	 $u_i$ is the corresponding eigenvector and $V=[v_1,\cdots,v_r]\in\mathbb{R}^{r\times r}$ can be any orthonormal matrix.

\end{corollary}
\begin{proof}[Proof of Corollary \ref{prop: diagonal contrast}]\label{proof: diagonal contrast}
	Following the proof of Proposition \ref{prop: augment}, now we only need to compute the expectation over the augmentation distribution defined in Definition \ref{aug: random masking}:
	\begin{align}
		\mathcal{L}_{\text{SelfCon}}(W) =& \frac{\lambda}{2}\|WW^\top\|_F^2-\mathbb{E}_{(g_1,g_2)}[\frac{1}{n}\sum_{i=1}^n[\langle Wg_1(x_i),Wg_2(x_i)\rangle\nonumber\\
		&-\frac{1}{4(n-1)}\sum_{j\neq i} \langle Wg_1(x_i)+Wg_2(x_i),Wg_1(x_j)+Wg_2(x_i)\rangle]]\nonumber\\
		=&\frac{\lambda}{2}\|WW^\top\|_F^2-\mathbb{E}_{(g_1,g_2)}[\frac{1}{2n}{\tr}((X_2X_1^\top+X_1X_2^\top\nonumber\\
		&-\frac{1}{2(n-1)}(X_1+X_2)(1_n 1_n^\top-I_n)(X_1+X_2)^\top)W^\top W)].
		\label{loss: random masking}
	\end{align}
	Note that by the definition of random masking augmentation, we have $X_1=AX,X_2=(I-A)X$, which implies $X_1+X_2=X$. On the other hand, $X_1$ and $X_2$ have no common nonzero entries, hence the matrix $X_1X_2^\top+X_2X_1^\top$ only consists of off-diagonal entries and each of the off-diagonal entry denoted as $x_{ij}$ appears if and only if $a_i+a_j=1$. Moreover, if it appears, we must have $x_{ij}$ equals to the $(i,j)$-th element of $XX^\top$. With this result, we can then compute the expectation in Equation \eqref{loss: random masking}:
	\begin{equation*}
		\begin{aligned}
			\mathcal{L}_{\text{SelfCon}}(W) =&\frac{\lambda}{2}\|WW^\top\|_F^2-\mathbb{E}_{(g_1,g_2)}\biggl[\frac{1}{2n}{\tr}((X_2X_1^\top+X_1X_2^\top\\
			&-\frac{1}{2(n-1)}(X_1+X_2)(1_n 1_n^\top-I_n)(X_1+X_2)^\top)W^\top W)\bigg]\\
			=&\frac{\lambda}{2}\|WW^\top\|_F^2-\frac{1}{2n}\tr(\qty(\frac{1}{2}\Delta(XX^\top)-\frac{1}{2(n-1)}X(1_n 1_n^\top-I_n)X^\top)W^\top W)\\
			=&\frac{1}{2}\norm{\lambda W^\top W-\frac{1}{4n\lambda}\qty(\Delta(XX^\top)-\frac{1}{n-1}X(1_n 1_n^\top-I_n)X^\top)}_F^2\\&-\norm{\frac{1}{4n\lambda}\qty(\Delta(XX^\top)-\frac{1}{n-1}X(1_n 1_n^\top-I_n)X^\top)}_F^2.
		\end{aligned}
	\end{equation*}
	By a similar argument as in the proof of Proposition \ref{prop: augment}, we can conclude that $W_{\CL}$ satisfies the desired conditions.
\end{proof}
\begin{remark}
    Note that the two views generated by random masking augmentation have disjoint non-zero dimensions, hence contrasting such positive pairs yields a correlation between different dimensions only. That is why the first term in equation (\ref{dd mat}) appears to be $\Delta(XX^\top)$ where the diagonal entries are eliminated.
\end{remark}

\subsection{Proofs for Section \ref{sec: recover}}
In this section, we will prove Lemma \ref{lem: incoherent}, Theorems \ref{thm: recover PCA} and \ref{thm: recover CL} in Section \ref{sec: recover}. The restatement and proof of them can be found in Lemma \ref{lem: incoherent ap}, Theorem \ref{thm: recover PCA ap}, and Theorem \ref{thm: recover CL ap}. 

Before starting the proof, we give two technical lemmas to help the proof.
\begin{lemma}[Uniform distribution on the unit sphere \citep{Marsaglia1972ChoosingAP}]
	\label{Lemma: sphere distribution}
	If $x_1,x_2,\cdots,x_n$ i.i.d. $\sim \mathcal{N}(0,1)$, then $(x_1/\sqrt{{\sum_{i=1}^{n}x_i^2}},\cdots,x_n/\sqrt{\sum_{i=1}^{n}x_i^2})$ is uniformly distributed on the unit sphere $\mathbb{S}^{d}=\{(x_1,\cdots,x_n)\in\mathbb{R}^n:\sum_{i=1}^n x_i^2=1\}$.
\end{lemma}

\begin{lemma}
    \label{Lemma: max chi}
	If $x_1,x_2,\cdots,x_n$ i.i.d. $\sim \mathcal{N}(0,1)$, then:
	\begin{equation*}
		\mathbb{E}\max_{1\leq i \leq n}x_i^2\leq 2\log(n).
	\end{equation*}
\end{lemma}

\begin{proof}
	Denote $Y=\max_{1\leq i \leq n}x_i^2$, then we have:
	\begin{equation*}
		\begin{aligned}
			\label{eq: exptey}
			\exp(t\mathbb{E}Y)\leq\mathbb{E}\exp(tY)\leq\mathbb{E}\sum_{i=1}^{n}\exp(tx_i^2)= n\mathbb{E}\exp(tx_i^2).
		\end{aligned}
	\end{equation*}
	Note that the moment-generating function of chi-square distribution with $v$ degrees of freedom is:
	\begin{equation*}
		M_{X}(t)=(1-2t)^{-v/2}.
	\end{equation*}
	Then combine this fact with Equation \eqref{eq: exptey} we have:
	\begin{equation*}
		\exp(t\mathbb{E}Y)\leq n(1-2t)^{-\frac{1}{2}},
	\end{equation*} 
	which implies:
	\begin{equation*}
		\mathbb{E}Y\leq\frac{\log(n)}{t}-\frac{1-2t}{2t},\quad\forall t<\frac{1}{2}.
	\end{equation*}
	In particular, take $t\rightarrow\frac{1}{2}$ yields:
	$$\mathbb{E}Y\leq2\log (n)$$ as desired.
\end{proof}
\begin{lemma}[Restatement of Lemma \ref{lem: incoherent}]
	\label{lem: incoherent ap}
	\begin{equation}
	\label{eq: expectation incoherent ap}
		\mathbb{E}_{U\sim \unif(\mathbb{O}_{d, r})}I(U^\star) = O\qty(\frac{r}{d}\log d).
	\end{equation}
\end{lemma}
\begin{proof}[Proof of Lemma \ref{lem: incoherent}]\label{proof: incoherent}
	Denote the columns of $U$ as $U=[u_1,\cdots,u_r]\in \mathbb{O}_{d, r}$, we have:
	\begin{equation*}
		\begin{aligned}
			\mathbb{E}_{U\sim \unif(\mathbb{O}_{d, r})}I(U)=&\mathbb{E}_{U\sim \unif(\mathbb{O}_{d, r})}\max _{i \in[d]}\sum_{j=1}^{r}|e_{i}^\top u_j|^{2}\\
			\leq& \mathbb{E}_{U\sim \unif(\mathbb{O}_{d, r})}\sum_{j=1}^{r}\max _{i \in[d]}|e_{i}^\top u_j|^{2}\\
			=& r\mathbb{E}_{u\sim \unif(\mathbb{S}^d)}\max _{i \in[d]}|e_{i}^\top u|^{2}.
		\end{aligned}
	\end{equation*}
	By Lemma \ref{Lemma: sphere distribution} we can transform this expectation on the uniform sphere distribution into normalized multivariate Gaussian variables:
	\begin{equation}
		\label{eq: IU gaussian}
		\mathbb{E}_{U\sim \unif(\mathbb{O}_{d, r})}I(U)=r\mathbb{E}_{x_1,\cdots,x_d}\frac{\max_{i \in[d]}x_i^2}{\sum_{j=1}^d x_j^2}.
	\end{equation} 
	where $x_1,x_2,\cdots,x_d$ are i.i.d. standard normal random variables. Apply Chebyshev’s inequality we know that:
	\begin{equation*}
		\mathbb{P}\qty(|\frac{1}{d}\sum_{i=1}^{d}x_j^2-1|>\epsilon)\leq\frac{2}{d\epsilon^2}.
	\end{equation*}
	In particular, take $\epsilon=1$ we have:
	\begin{equation*}
		\mathbb{P}\qty(\sum_{i=1}^{d}x_j^2<\frac{d}{2})\leq\frac{8}{d}.
	\end{equation*}
	Then take it back into Equation \eqref{eq: IU gaussian} and apply Lemma \ref{Lemma: max chi} we obtain:
	\begin{equation*}
		\begin{aligned}
			\mathbb{E}_{U\sim \unif(\mathbb{O}_{d, r})}I(U)=&r\mathbb{E}_{x_1,\cdots,x_d}\frac{\max_{i \in[d]}x_i^2}{\sum_{j=1}^d x_j^2}\1\{\sum_{i=1}^{d}x_j^2<\frac{d}{2}\}\\
			&+r\mathbb{E}_{x_1,\cdots,x_d}\frac{\max_{i \in[d]}x_i^2}{\sum_{j=1}^d x_j^2}\1\{\sum_{i=1}^{d}x_j^2\geq\frac{d}{2}\}\\
			\leq& r\mathbb{P}\qty(\sum_{i=1}^{d}x_j^2<\frac{d}{2})+\frac{2r}{d}\mathbb{E}_{x_1,\cdots,x_d}\max_{i \in[d]}x_i^2\\
			\leq& \frac{8r}{d}+\frac{4r\log d}{d}
		\end{aligned}
	\end{equation*}
	as desired.
\end{proof}

Now we start proving our main results. 
Note that $U_{\AEN}$ is the top-$r$ left eigenspace of the observed covariance matrix and $U^\star$ is that of the core feature covariance matrix, and by Assumption \ref{asm: SNR} the observed covariance matrix is dominated by the covariance of random noise. The Davis-Kahan theorem provides a technique to estimate the eigenspace distance via estimating the difference between target matrices. We will adopt this technique to prove the lower bound of the feature recovery ability of autoencoders in Theorem \ref{thm: recover PCA}.
\begin{theorem}[Restatement of Theorem \ref{thm: recover PCA}]
\label{thm: recover PCA ap}
	Consider the spiked covariance model Eq.(\ref{model: spiked covariance}), under Assumptions \ref{asm: regular}-\ref{asm: incoherent} and $n> d\gg r$, let $W_{AE}$ be the learned representation of autoencoder with singular value decomposition $W_{AE}=(U_{AE}\Sigma_{AE}V_{AE}^\top)^\top$ (as in Eq.(\ref{opt: AE})). If we further assume $\{\sigma_{i}^2\}_{i=1}^d$ are different from each other and $\sigma_{(1)}^2/(\sigma_{(r)}^2-\sigma_{(r+1)}^2)<C_{\sigma}$ for some universal constant $C_{\sigma}$. Then there exist two universal constants $C_\rho>0,c\in (0,1)$, such that when $\rho<C_\rho$, we have
	\begin{equation}
		\mathbb{E}\left\|\sin \Theta\left(U^\star, U_{AE}\right)\right\|_F \geq c\sqrt{r}.
	\end{equation}
\end{theorem}
\begin{proof}[Proof of Theorem \ref{thm: recover PCA}]\label{proof: recover PCA}
	Denote $M=\nu^2U^\star U^{\star \top}$ to be the target matrix, $x_i=U^\star z_i+\xi_i,\quad i=1,2,\cdots n$ to be the samples generated from model \ref{model: spiked covariance} and let $X=[x_1,\cdots,x_n]\in\mathbb{R}^{d\times n},Z=[z_1,\cdots,z_n]\in\mathbb{R}^{r\times n},E=[\xi_1,\cdots,\xi_n]\in\mathbb{R}^{d\times n}$ to be the corresponding matrices. In addition, we write the column mean matrix $\bar{X}\in\mathbb{R}^{n\times d}$ of a matrix $X\in\mathbb{R}^{n\times d}$ to be $\bar{X}=\frac{1}{n}X1_n 1_n^\top$, that is, each column of $\bar{X}$ is the column mean of $X$. We denote the sum of variance $\sigma_{i}^2$ as $\sigma_{\text{sum}}^2=\sum_{i=1}^d\sigma_{i}^2$. As shown in Equation \eqref{opt: AE}, autoencoders find the top-$r$ eigenspace of the following matrix:
	\begin{equation*}
		\begin{aligned}
			\hat{M}_1 = \frac{1}{n}X(I_n-\frac{1}{n}1_n 1_n^\top)X^\top&=\frac{1}{n}(U^\star Z+E)(U^\star Z+E)^\top-\frac{1}{n}(U^\star\bar{Z}+\bar{E})(U^\star\bar{Z}+\bar{E})^\top.
		\end{aligned}
	\end{equation*}
	The rest of the proof is divided into three steps for the sake of presentation.\\
	
	\paragraph{Step 1. Bound the difference between $\hat{M}_1$ and $\Sigma$}
	In this step, we aim to show that the data recovery of autoencoders is dominated by the random noise term. Note that $\Sigma=\Cov(\xi)=\mathbb{E}\xi\xi^\top$, we just need to bound the norm of the following matrix:
	\begin{equation}
		\label{PCA eq}
		\hat{M}_1-\Sigma=\frac{1}{n}U^\star ZZ^\top U^{\star \top}+\frac{1}{n}(U^\star ZE^\top+EZ^\top U^{\star \top})+(\frac{1}{n}EE^\top-\Sigma)-\frac{1}{n}(U^\star\bar{Z}+\bar{E})(U^\star\bar{Z}+\bar{E})^\top,
	\end{equation}
	and we will deal with these four terms separately.
	\begin{enumerate}
		\item For the first term, note that $\mathbb{E}zz^\top=\nu^2I_r$, the first term can then be divided into two terms
		\begin{equation}
			\label{PCA step1 term1 }
			\frac{1}{n}U^\star ZZ^\top U^{\star \top} = M+U^\star(\frac{1}{n}ZZ^\top-\mathbb{E}zz^\top)U^{\star \top}.
		\end{equation}
		Then apply the concentration inequality of Wishart-type matrices (Lemma \ref{Lemma: bound ZZT}) we have:
		\begin{equation*}
			\mathbb{E}\|\frac{1}{n}ZZ^\top-\mathbb{E}zz^\top\|_2\leq(\sqrt{\frac{r}{n}}+\frac{r}{n})\nu^2.
		\end{equation*}
		Plug it back into (\ref{PCA step1 term1 }) we obtain the bound for the first term:
		\begin{equation}
			\label{PCA step1 term1 b}
			\|\frac{1}{n}UZZ^\top U^\top\|_2\leq\|M\|_2+\|U\|_2\|\frac{1}{n}ZZ^\top-\mathbb{E}zz^\top\|_2\|U\|_2\leq\qty(1+\sqrt{\frac{r}{n}}+\frac{r}{n})\nu^2.
		\end{equation}
		\item For the second term, since $Z$ and $E$ are independent, we must have $\mathbb{E}U^\star ZE^\top=0$, so apply Lemma \ref{Lemma: bound EV} twice we have:
		\begin{equation}
			\begin{aligned}
				\label{PCA step1 term2}
				\frac{1}{n}\mathbb{E}\|EZ^\top U^\star\|_2=&\frac{1}{n}\mathbb{E}_{Z}[\mathbb{E}_{E}[\|EZ^\top U^\star\|_2|Z]]\\
				\lesssim& \frac{1}{n}\mathbb{E}_Z[\|Z\|_2(\sigma_{\text{sum}}+r^{1/4}\sqrt{\sigma_{\text{sum}}\sigma_{(1)}}+\sqrt{r}\sigma_{(1)})]\\
				\lesssim& \frac{1}{n}\mathbb{E}_Z[\|Z\|_2]\sqrt{d}\sigma_{(1)}\\
				\lesssim&\frac{1}{n}\sqrt{d}\sigma_{(1)}(r^{1/2}\nu+(nr)^{1/4}\nu+n^{1/2}\nu)\\
				\lesssim&\frac{\sqrt{d}}{\sqrt{n}}\sigma_{(1)}\nu.
			\end{aligned}
		\end{equation} 
		
		\item For the third term, apply Lemma \ref{Lemma: bound ZZT} again yields:
		\begin{equation}
			\label{PCA step1 term3}
			\mathbb{E}\|\frac{1}{n}EE^\top-\Sigma\|_2\leq\qty(\sqrt{\frac{d}{n}}+\frac{d}{n})\sigma_{(1)}^2.
		\end{equation}
		\item For the last term, note that each column of $\bar{Z}$ and $\bar{E}$ are the same, so we can rewrite it as:
		\begin{equation*}
			\frac{1}{n}(U^\star\bar{Z}+\bar{E})(U^\star\bar{Z}+\bar{E})^\top=(U^\star\bar{z}+\bar{\xi})(U^\star\bar{z}+\bar{\xi})^\top,
		\end{equation*}
		where $\bar{z}=\frac{1}{n}\sum_{i=1}^{n}z_i$ and $\bar{\xi}=\frac{1}{n}\sum_{i=1}^{n}\xi_i$. Since $z$ and $\xi$ are independent zero mean sub-Gaussian random variables and $\Cov(z)=\nu^2I_r,\Cov(\xi) = \Sigma$, we can conclude that:
		\begin{equation}
			\begin{aligned}
				\label{PCA step1 term4}
				&\mathbb{E}\|\frac{1}{n}(U^\star \bar{Z}+\bar{E})(U^\star \bar{Z}+\bar{E})^\top\|_2\leq\mathbb{E}\|\bar{z}\bar{z}^\top\|_2+2\mathbb{E}\|\bar{z}\bar{\xi}^\top\|_2+\mathbb{E}\|\bar{\xi}\bar{\xi}^\top\|_2\\
				\lesssim&\frac{r\nu^2}{n}+\frac{\sqrt{d}}{\sqrt{n}}\sigma_{(1)}\nu+\frac{d\sigma_{(1)}^2}{n}.
			\end{aligned}
		\end{equation}
	\end{enumerate}
	To sum up, combine equations (\ref{PCA step1 term1 b})(\ref{PCA step1 term2})(\ref{PCA step1 term3})(\ref{PCA step1 term4}) together we obtain the upper bound for the 2 norm expectation of matrix $\hat{M}-\Sigma$:
	\begin{equation}
		\label{PCA difference bound}
		\mathbb{E}\|\hat{M}_1-\Sigma\|_2\lesssim \nu^2\qty(1+\sqrt{\frac{r}{n}}+\frac{r}{n})+\sigma_{(1)}^2\qty(\sqrt{\frac{d}{n}}+\frac{d}{n})+\sqrt{\frac{d}{n}}\sigma_{(1)}\nu.
	\end{equation}
	
	\paragraph{Step 2. Bound the sine distance between eigenspaces}
	As we have shown in step 1, the target matrix of autoencoders is close to the covariance matrix of random noise, that is, $\Sigma$. Note that $\Sigma$ is assumed to be a diagonal matrix with different elements, hence its eigenspace only consists of canonical basis $e_i$. Denote $U_{\Sigma}$ to be the top-$r$ eigenspace of $\Sigma$ and $\{e_i\}_{i\in C}$ to be its corresponding basis vectors, apply the Davis-Kahan Theorem \ref{Lemma: DK} we can conclude that:
	\begin{equation*}
		\begin{aligned}
			&\mathbb{E}\|\sin\Theta(U_{\AEN},U_{\Sigma})\|_F\leq\frac{2\sqrt{r}\mathbb{E}\|\hat{M}_1-\Sigma\|_2}{\sigma_{(r)}^2-\sigma_{(r+1)}^2}\\
			\lesssim& \sqrt{r}\frac{1}{\sigma_{(1)}^2}\qty(\nu^2\qty(1+\sqrt{\frac{r}{n}}+\frac{r}{n})+\sigma_{(1)}^2\qty(\sqrt{\frac{d}{n}}+\frac{d}{n})+\sqrt{\frac{d}{n}}\sigma_{(1)}\nu)\\
			\lesssim& \sqrt{r} \qty(\rho^2+\sqrt{\frac{d}{n}}+\rho\sqrt{\frac{d}{n}}).\\
		\end{aligned}
	\end{equation*}
	
	\paragraph{Step 3. Obtain the final result by triangular inequality}
	By Assumption \ref{asm: incoherent} we know that the distance between canonical basis and the eigenspace of core features can be large:
	\begin{equation*}
		\begin{aligned}
			\|\sin\Theta(U^\star ,U_{\Sigma})\|_F^2&=\|U_{\Sigma\perp}^\top U^\star \|_F^2=\sum_{i\in[d] \setminus C}\|e_i^\top U^\star \|^2=\|U^\star \|_F^2-\sum_{i\in C}\|e_i^\top U^\star \|^2\\
			&\geq r-rI(U^\star )= r-O\qty(\frac{r^2}{d}\log d).
		\end{aligned}
	\end{equation*}
	Then apply the triangular inequality of sine distance (Proposition \ref{triangular inequality of sin theta}) we can obtain the lower bound of autoencoders.
	\begin{equation}
	\label{dependence of rho AE}
		\begin{aligned}
			\mathbb{E}\|\sin\Theta(U_{\AEN},U^\star )\|_F&\geq\mathbb{E}\|\sin\Theta(U^\star ,U_\Sigma)\|_F-\mathbb{E}\|\sin\Theta(U_{\AEN},U_\Sigma)\|_F\\
			&\geq\sqrt{r}-O\qty(\frac{r}{\sqrt{d}}\sqrt{\log d})-O\qty(\sqrt{r} \qty(\rho^2+\sqrt{\frac{d}{n}}+\rho\sqrt{\frac{d}{n}})).
		\end{aligned}
	\end{equation}
	By Assumption \ref{asm: SNR}, it implies that when n and d are sufficiently large and $\rho$ is sufficiently small (smaller than a given constant $C_\rho>0$), there exists a universal constant $c\in(0,1)$ such that:
	\begin{equation*}
	    \mathbb{E}\|\sin\Theta(U_{\AEN},U^\star )\|_F\geq c\sqrt{r}.
	\end{equation*}
\end{proof}

To start the proof, we introduce a technical lemma first.
\begin{lemma}[Lemma 4 in \citet{zhang2018heteroskedastic}]
	\label{Lemma: Delta}
	If $M \in \mathbb{R}^{p \times p}$ is any square matrix and $\Delta(M)$ is the matrix $M$ with diagonal entries set to 0 , then
	$$
	\|\Delta(M)\|_2 \leq 2\|M\|_2.
	$$
	Here, factor " 2 " in the statement above cannot be improved.
\end{lemma}
\begin{theorem}[Restatement of Theorem \ref{thm: recover CL}]
\label{thm: recover CL ap}
	Under the spiked covariance model Eq.(\ref{model: spiked covariance}), random masking augmentation in Definition \ref{aug: random masking}, Assumptions \ref{asm: regular}-\ref{asm: incoherent} and $n> d\gg r$, let $W_{CL}$ be any solution that minimizes Eq.(\ref{loss: self contrastive}), and denote its singular value decomposition as $W_{CL}=(U_{CL}\Sigma_{CL}V_{CL}^\top)^\top$, then we have
	\begin{equation}
		\mathbb{E}\left\|\sin \Theta\left(U^\star, U_{CL}\right)\right\|_F \lesssim\frac{r^{3/2}}{d}\log d+\sqrt{\frac{dr}{n}}. 
	\end{equation}
\end{theorem}
\begin{proof}[Proof of Theorem \ref{thm: recover CL}]\label{proof: recover CL}
	The proof strategy is quite similar to that of Theorem \ref{thm: recover PCA} and we follow the notation defined in the first paragraph of that proof. As we have shown in Corollary \ref{prop: diagonal contrast}, under our linear representation setting, the contrastive learning algorithm finds the top-$r$ eigenspace of the following matrix:
	\begin{equation*}
		\begin{aligned}
			\hat{M}_2 =& \frac{1}{n}\qty(\Delta(XX^\top)-\frac{1}{n-1}X(1_n 1_n^\top-I_n)X^\top)\\ =&\frac{1}{n}\Delta((U^\star Z+E)(U^\star Z+E)^\top)-\frac{1}{n-1}(U^\star  \bar{Z}+\bar{E})(U^\star \bar{Z}+\bar{E})^\top\\
			&+\frac{1}{n(n-1)}(U^\star  Z+E)(U^\star  Z+E)^\top.
		\end{aligned}
	\end{equation*}
	
	To prove the theorem, first we need to bound the difference between $\hat{M}_2$ and $M$. We aim to show that the contrastive learning algorithm is dominated by the core feature term. Note that $\Sigma=\mathbb{E}Uzz^\top U^\top$, we just need to bound the norm of the following matrix:
	\begin{equation}
		\label{CL eq}
		\begin{aligned}
			\hat{M}_2-M=&(\frac{1}{n}\Delta(U^\star ZZ^\top U^{\star \top})-M)+\frac{1}{n}\Delta(U^\star ZE^\top+EZ^\top U^{\star \top})+\frac{1}{n}\Delta (EE^\top)\\ 
			&-\frac{1}{n-1}(U^\star \bar{Z}+\bar{E})(U^\star \bar{Z}+\bar{E})^\top+\frac{1}{n(n-1)}(U^\star Z+E)(U^\star Z+E)^\top.
		\end{aligned}
	\end{equation}
	and we will also deal with these five terms separately.
	\begin{enumerate}
		\item For the first term, we can divide it into two parts:
		\begin{equation}
			\label{CL step1 term1 }
			\frac{1}{n}\Delta(U^\star ZZ^\top U^{\star \top})-M = \Delta(\frac{1}{n}U^\star ZZ^\top U^{\star  T}-M)+\Delta(M)-M.
		\end{equation}
		Then apply Lemma \ref{Lemma: Delta} and Lemma \ref{Lemma: bound ZZT} we have:
		\begin{equation*}
			\mathbb{E}\|\Delta(\frac{1}{n}U^\star ZZ^\top U^{\star \top}-M)\|_2\leq 2\mathbb{E}\|\frac{1}{n}U^\star ZZ^\top U^{\star \top}-M\|_2\leq 2(\sqrt{\frac{r}{n}}+\frac{r}{n})\nu^2.
		\end{equation*}
		Using the incoherent condition $I(U)=O(\frac{r}{d}\log d)$, we know that:
		\begin{equation*}
			\|M-\Delta(M)\|_2\leq\nu^2\max_{i \in[d]}\|e_i^\top U^\star \|_2^2=\nu^2I(U^\star )\lesssim\frac{r}{d}\log d\nu^2.
		\end{equation*}
		\\
		Combine the two equations above together we obtain the bound for the first term:
		\begin{align}
			\label{CL step1 term1 b}
			\mathbb{E}\|\frac{1}{n}\Delta(U^\star ZZ^\top U^{\star \top})-M\|_2&\leq\mathbb{E}\|\Delta(\frac{1}{n}U^\star ZZ^\top U^{\star \top}-M)\|_2+\|M-\Delta(M)\|_2\\
			&\lesssim \nu^2(\frac{r}{d}\log d+\frac{r}{n}+\sqrt{\frac{r}{n}}).
		\end{align}
		\item For the second term, apply equation (\ref{PCA step1 term2}) yields:
		\begin{equation}
			\begin{aligned}
				\label{CL step1 term2}
				\frac{1}{n}\mathbb{E}\|\Delta(U^\star ZE^\top+EZ^\top U^{\star \top})\|_2\leq\frac{4}{n}\mathbb{E}\|EZ^\top U^{\star \top}\|_2\lesssim\frac{\sqrt{d}}{\sqrt{n}}\sigma_{(1)}\nu.
			\end{aligned}
		\end{equation} 
		
		\item For the third term, apply equation (\ref{PCA step1 term3}) yields:
		\begin{equation}
			\label{CL step1 term3}
			\mathbb{E}\|\frac{1}{n}\Delta(EE^\top)\|_2=\mathbb{E}\|\Delta(\frac{1}{n}EE^\top-\Sigma)\|_2\leq 2 \|\frac{1}{n}EE^\top-\Sigma\|_2\lesssim(\sqrt{\frac{d}{n}}+\frac{d}{n})\sigma_{(1)}^2.
		\end{equation}
		\item For the fourth term, apply equation (\ref{PCA step1 term4}) yields:
		\begin{equation}
			\begin{aligned}
				\label{CL step1 term4}
				\mathbb{E}\|\frac{1}{n-1}(U^\star \bar{Z}+\bar{E})(U\bar{Z}+\bar{E})^\top\|_2\lesssim&\mathbb{E}\|\frac{1}{n}(U\bar{Z}+\bar{E})(U\bar{Z}+\bar{E})^\top\|_2\\
				\lesssim&\frac{r\nu^2}{n}+\frac{\sqrt{d}}{\sqrt{n}}\sigma_{(1)}\nu+\frac{d\sigma_{(1)}^2}{n}.
			\end{aligned}
		\end{equation}
		\item For the last term, by equations (\ref{PCA step1 term1 b})(\ref{PCA step1 term2})(\ref{PCA step1 term3}) we know:
		\begin{align*}
			&\mathbb{E}\|\frac{1}{n}(U^\star Z+E)(U^\star Z+E)^\top\|_2\\
			&\lesssim\|\Sigma\|_2+\qty(1+\sqrt{\frac{r}{n}}+\frac{r}{n})\nu^2+\sqrt{\frac{d}{n}}\sigma_{(1)}\nu+\qty(\sqrt{\frac{d}{n}}+\frac{d}{n})\sigma_{(1)}^2.
		\end{align*}
		Thus we can conclude that:
		\begin{equation}
			\label{CL step1 term5}
			\mathbb{E}\|\frac{1}{n(n-1)}(U^\star Z+E)(U^\star Z+E)^\top\|_2\lesssim\frac{d}{n}\sigma_{(1)}^2+\frac{r}{n}\nu^2.
		\end{equation}
	\end{enumerate}
	To sum up, combine equations (\ref{CL step1 term1 b})(\ref{CL step1 term2})(\ref{CL step1 term3})(\ref{CL step1 term4})(\ref{CL step1 term5}) together we obtain the upper bound for the 2 norm expectation of matrix $\hat{M}_2-M$:
	\begin{equation}
		\label{CL difference bound}
		\mathbb{E}\|\hat{M}_2-M\|_2 \lesssim \nu^2\qty(\frac{r}{d}\log d+\sqrt{\frac{r}{n}}+\frac{r}{n})+\sigma_{(1)}^2\qty(\sqrt{\frac{d}{n}}+\frac{d}{n})+\sigma_{(1)}\nu\sqrt{\frac{d}{n}}.
	\end{equation}
	With the upper bound for $\|\hat{M}_2-M\|_2$, simply apply Lemma \ref{Lemma: DK} we can obtain the desired bound for sine distance:
	\begin{equation}
		\begin{aligned}
		\label{dependence of rho CL}
			&\mathbb{E}\|\sin\Theta(U_{\CL},U^\star )\|_F\leq\frac{2\sqrt{r}\mathbb{E}\|\hat{M}_2-M\|_2}{\nu^2}\\
			\lesssim& \sqrt{r}\frac{1}{\nu^2}\qty(\nu^2\qty(\frac{r}{d}\log d+\sqrt{\frac{r}{n}}+\frac{r}{n})+\sigma_{(1)}^2\qty(\sqrt{\frac{d}{n}}+\frac{d}{n})+\sigma_{(1)}\nu\sqrt{\frac{d}{n}})\\
			=& \sqrt{r}\qty(\qty(\frac{r}{d}\log d+\sqrt{\frac{r}{n}}+\frac{r}{n})+\rho^{-2}\qty(\sqrt{\frac{d}{n}}+\frac{d}{n})+\rho^{-1}\sqrt{\frac{d}{n}})\\
			\lesssim&\frac{r^{3/2}}{d}\log d+\sqrt{\frac{dr}{n}}.
		\end{aligned}
	\end{equation}
	Moreover, there exists an orthogonal matrix $\hat{O}\in\mathbb{O}^{r\times r}$ depending on $U_{\CL}$ such that:
	\begin{equation*}
		\mathbb{E}\|U^\top U_{\CL}\hat{O}-I_r\|_F=\mathbb{E}\|U_{\CL}\hat{O}-U\|_F\leq\frac{2\sqrt{r}\mathbb{E}\|\hat{M}_2-M\|_2}{\nu^2}\lesssim\frac{r^{3/2}}{d}\log d+\sqrt{\frac{dr}{n}}.
	\end{equation*}
	which finishes the proof.
\end{proof}
\subsection{Proofs for Section \ref{sec: downstream}}
\label{appx: downstream}

In this section, we will provide the proof of Theorems \ref{thm: best linear predictor risk CL} and \ref{thm: best linear predictor risk PCA lower bound} with both regression and classification settings. The corresponding statement and proof can be found in Theorems \ref{thm: best linear predictor risk CL restatement} and \ref{thm: best linear predictor risk PCA lower bound restatement}.

For notation simplicity, define the prediction risk of predictor $\delta$ for classification and regression tasks as $\mathcal{R}_c(\delta) := \mathbb{E}_{\mathcal{D}}[\ell_c(\delta)]$ and $\mathcal{R}_r(\delta) := \mathbb{E}_{\mathcal{D}}[\ell_r(\delta)]$, respectively.
Define $\Sigma_x := \nu^2 U^\star U^{\star \top} + \Sigma$.
We write $\delta_{U, w}$ for $\delta_{U^\top, w}$ with a slight abuse of notation.
For two matrices $A$ and $B$ of the same order, we define $A \succeq B$ when $A - B$ is positive semi-definite.

\begin{theorem}[Restatement of Theorem \ref{thm: best linear predictor risk CL}]\label{thm: best linear predictor risk CL restatement}
    Suppose the conditions in Theorem \ref{thm: recover CL} hold. Then, for the classification task, we have 
    \begin{align*}
        \mathbb{E}_{\mathcal{D}}[\inf_{w \in \mathbb{R}^r} \mathbb{E}_{\mathcal{E}}[\ell_c(\delta_{W_{\CL}, w})] - \inf_{w \in \mathbb{R}^r} \mathbb{E}_{\mathcal{E}}[\ell_c(\delta_{U^{\star \top}, w})] = O\qty(\frac{r^{3/2}}{d} \log d + \sqrt{\frac{dr}{n}}) \wedge 1,
    \end{align*}
    and for regression tasks,
    \begin{align*}
		\mathbb{E}_{\mathcal{D}}[\inf_{w\in\mathbb{R}^r} \mathbb{E}_{\mathcal{E}}[\ell_r(\delta_{W_{\CL}, w})]-\inf_{w\in\mathbb{R}^r}\mathbb{E}_{\mathcal{E}}[\ell_r(\delta_{U^{\star \top}, w})] \lesssim \frac{r^{3/2}}{d}\log d + \sqrt{\frac{dr}{n}}.
	\end{align*}
\end{theorem}

\begin{theorem}[Restatement of Theorem \ref{thm: best linear predictor risk PCA lower bound}]\label{thm: best linear predictor risk PCA lower bound restatement}
	Suppose the conditions in Theorem \ref{thm: recover PCA} hold.
	Assume $r \leq r_c$ holds for some constant $r_c > 0$.
	Additionally assume that $\rho = \Theta(1)$ is sufficiently small and $n \gg d \gg r$. Then,
    For the regression task,
    \begin{align*}
        \mathbb{E}_{\mathcal{D}}[\inf_{w\in\mathbb{R}^r} \mathbb{E}_{\mathcal{E}}[\ell_r(\delta_{U_{\AEN}, w})]-\inf_{w\in\mathbb{R}^r}\mathbb{E}_{\mathcal{E}}[\ell_r(\delta_{U^\star, w})] \geq c_c',
    \end{align*}
    and for classification task, if $F$ is differentiable at $0$ and $F'(0) > 0$, then
    \begin{align*}
        \mathbb{E}_{\mathcal{D}}[\inf_{w \in \mathbb{R}^r} \mathbb{E}_{\mathcal{E}}[\ell_c(\delta_{U_{\AEN}, w})] - \inf_{w \in \mathbb{R}^r} \mathbb{E}_{\mathcal{E}}[\ell_c(\delta_{U^\star, w})] \geq c_r',
    \end{align*}
    where $c_r' > 0$ and $c_c' > 0$ are constants independent of $n$ and $d$.
\end{theorem}

The proofs of Theorem \ref{thm: best linear predictor risk CL restatement} and \ref{thm: best linear predictor risk PCA lower bound restatement} relies on Lemma \ref{lem: auxiliary lemma 4 for risk matrix}, \ref{lem: squared risk lb naive}, \ref{lem: classification risk upper bound}, \ref{lem: classification risk lower bound} and \ref{lem: excess risk} which are proved later in this section.

\begin{proof}[Proof of Theorem \ref{thm: best linear predictor risk CL restatement}: Classification Task Part]\label{proof: best linear predictor risk CL 1}
    Lemma \ref{lem: classification risk upper bound} gives
    for any $U \in \mathbb{O}_{d,r}$,
    \begin{align}
        &\mathbb{E}_{\mathcal{D}}[\inf_{w \in \mathbb{R}^r} \mathcal{R}_c(\delta_{U, w}) - \inf_{w \in \mathbb{R}^r} \mathcal{R}_c(\delta_{U^\star, w})]\\
        &\quad\leq ((\kappa(1 + \rho^2))^3 + \kappa \rho^2 (1 +\rho^{-2})^2 + (\kappa\rho^2 \vee 1)^{-1}) \mathbb{E}_{\mathcal{D}}[\|\sin\Theta(U, U^\star)\|_2].\label{dependence on rho: classification}
    \end{align}
    Substituting $U \leftarrow U_{AE}$ combined with Assumption \ref{asm: SNR} and $\kappa = O(1)$ concludes the proof.
\end{proof}

\begin{proof}[Proof of Theorem \ref{thm: best linear predictor risk CL restatement}: Regression Part]\label{proof: best linear predictor risk CL 2}
    Note that under Assumption \ref{asm: SNR} and $\kappa = O(1)$,
    $(1 + \rho^{-2})/(1 + \kappa^{-1}\rho^{-2})^2 = O(1)$.
    Lemma \ref{lem: excess risk} gives
    for any $U \in \mathbb{O}_{d,r}$,
    \begin{align}
        \mathbb{E}_{\mathcal{D}}[\inf_{w \in \mathbb{R}^r} \mathcal{R}_r(\delta_{U, w}) - \inf_{w \in \mathbb{R}^r} \mathcal{R}_r(\delta_{U^\star, w})] = O\qty( (1 + \rho^{-2}) \mathbb{E}_{\mathcal{D}} [\|\sin\Theta(U, U^\star)\|_2] \|w^\star\|^2).\label{dependence on rho: regression}
    \end{align}
    Theorem \ref{thm: recover CL} with substitution $U \leftarrow U_{AE}$ gives the desired result.
\end{proof}

\begin{proof}[Proof of Theorem \ref{thm: best linear predictor risk PCA lower bound restatement}: Classification Part]\label{proof: best linear predictor risk PCA lower bound 1}
    Lemma \ref{lem: squared risk lb naive} gives that for $c_1 := 1 - 1/(2\kappa r_c) \in (0, 1)$, we can take $n \gg d \gg r$ and sufficiently small $\rho > 0$ so that $\mathbb{E}_{\mathcal{D}}[ \|\sin\Theta(U_{AE}, U^\star)\|_F^2 ] \geq c_1 r$ holds.
    By Lemma \ref{lem: classification risk lower bound},
    \begin{align}
        &\mathbb{E}_{\mathcal{D}} [\inf_{w \in \mathbb{R}^r} \mathcal{R}_c(\delta_{U_{AE}, w}) - \inf_{w \in \mathbb{R}^r} \mathcal{R}_c(\delta_{U^\star, w})]\nonumber\nonumber\\
        &\quad\gtrsim \frac{(1 + \rho^2)^{3/2}}{(1 + \kappa \rho^2)^{3/2}} \rho^2
        \qty( \frac{1}{1 + \rho^2} - \kappa (r - \|\sin\Theta(U_{AE}, U^\star)\|_F^2) )\nonumber\\
        &\quad\geq \frac{(1 + \rho^2)^{3/2}}{(1 + \kappa \rho^2)^{3/2}} \rho^2
        \qty( \frac{1}{1 + \rho^2} - \kappa (1 - c_1) r )\nonumber\\
        &\quad\geq \frac{(1 + \rho^2)^{3/2}}{(1 + \kappa \rho^2)^{3/2}} \rho^2
        \qty( \frac{1}{1 + \rho^2} - \frac{1}{2} ),\label{dependence on rho: classification lb}
    \end{align}
    where the last inequality follows since $r \leq r_c$. If we further take $\rho = \Theta(1) < 1/2$, the right hand becomes a positive constant.
    This concludes the proof.
\end{proof}

\begin{proof}[Proof of Theorem \ref{thm: best linear predictor risk PCA lower bound restatement}: Regression Part]
    \label{proof: best linear predictor risk PCA lower bound 2}
    From proposition \ref{prop: excess_risk}, we have
    \begin{align*}
        &\inf_{w \in \mathbb{R}^r} \mathcal{R}_r(\delta_{U_{AE}, w}) - \inf_{w \in \mathbb{R}^r} \mathcal{R}_r(\delta_{U^\star, w})\\
        &\quad= {w^\star}^\top ( (I + (1/\nu^2) U^{\star \top} \Sigma U^\star)^{-1}\\
        &\quad\quad- U^{\star \top} U_{AE} (U_{AE}^\top U^\star U^{\star \top} U_{AE} + (1/\nu^2) U_{AE}^\top \Sigma U_{AE})^{-1} U_{AE}^\top U^\star) w^\star.
    \end{align*}
    Thus from Lemma \ref{lem: auxiliary lemma 4 for risk matrix},
    \begin{align}
        &\inf_{w \in \mathbb{R}^r} \mathcal{R}_r(\delta_{U_{AE}, w}) - \inf_{w \in \mathbb{R}^r} \mathcal{R}_r(\delta_{U^\star, w})\nonumber\\
        &\quad\geq \qty(\frac{1}{1 + \rho^{-2}} + \rho^2\kappa \qty( \|\sin\Theta(U_{AE}, U^\star)\|_F^2 - r))\|w^\star\|^2.\label{dependence on rho: regression lb}
    \end{align}
    Using Lemma \ref{lem: squared risk lb naive} and by the same argument in the proof of Theorem \ref{thm: best linear predictor risk PCA lower bound}: Classification Part, we conclude the proof.
\end{proof}

\begin{lemma}\label{lem: auxiliary lemma 1 for risk matrix}
    For any $U \in \mathbb{O}_{d,r}$,
    \begin{align*}
        \lambda_{\min}(\nu^2 U^{\star \top} U (U^\top \Sigma_x U)^{-1} U^\top U^\star) \geq \frac{\nu^2}{\nu^2 + \sigma_{(1)}^2} (1 - \|\sin\Theta(U, U^\star)\|_2^2).
    \end{align*}
\end{lemma}

\begin{proof}
    Since $\lambda_{\min}(A C) \geq \lambda_{\min}(A)\lambda_{\min}(C)$ for symmetric positive semi-definite matrices $A$ and $C$,
    \begin{align*}
        &\lambda_{\min}(\nu^2 U^{\star \top} U (U^\top \Sigma_x U)^{-1} U^\top U^\star)\\
        &\quad\geq \lambda_{\min}(U^\top U^\star U^{\star \top} U) \lambda_{\min}(\nu^2 (U^\top \Sigma_x U)^{-1})\\
        &\quad\geq \lambda_{\min}(I - (I - U^\top U^\star U^{\star \top} U)) \frac{\nu^2}{\lambda_{\max}(\nu^2 U^\top U^\star U^{\star \top} U + U^\top \Sigma U)}\\
        &\quad\geq \frac{\nu^2}{\nu^2 + \sigma_{(1)}^2} (1 - \|\sin\Theta(U, U^\star)\|_2^2),
    \end{align*}
    where we used Weyl's inequality $\lambda_{\min}(A + C) \geq \lambda_{\min}(A) - \|C\|_2$ in the second inequality.
\end{proof}

\begin{lemma}\label{lem: auxiliary lemma 2 for risk matrix}
    For any $U \in \mathbb{O}_{d,r}$,
    \begin{align*}
        \lambda_{\max}(\nu^2 U^{\star \top} U (U^\top \Sigma_x U)^{-1} U^\top U^\star) \leq \frac{\nu^2}{\nu^2 (1 - \|\sin\Theta(U, U^\star)\|_2) + \sigma_{(d)}^2}.
    \end{align*}
\end{lemma}

\begin{proof}
    Since $\|A C\|_2 \leq \|A\|_2 \|C\|_2$,
    \begin{align*}
        \lambda_{\max}(\nu^2 U^{\star \top} U (U^\top \Sigma_x U)^{-1} U^\top U^\star)
        &\leq \lambda_{\max}(\nu^2 (U^\top \Sigma_x U)^{-1})\\
        &\leq \frac{\nu^2}{\lambda_{\min}(\nu^2 U^\top U^\star U^{\star \top} U + U^\top \Sigma U)}\\
        &\leq \frac{\nu^2}{\lambda_{\min}(\nu^2 I - \nu^2 (I - U^\top U^\star U^{\star \top} U) + U^\top \Sigma U)}\\
        &\leq \frac{\nu^2}{\nu^2 (1 - \|\sin\Theta(U, U^\star)\|_2) + \sigma_{(d)}^2},
    \end{align*}
    where we used Weyl's inequality $\lambda_{\min}(A + C) \geq \lambda_{\min}(A) - \|C\|_2$ and $\lambda_{\min}(\nu^2 I + U^\top \Sigma U) \geq \nu^2 + \sigma_{(d)}^2$.
\end{proof}

\begin{lemma}\label{lem: auxiliary lemma 3 for risk matrix}
    For any $U \in \mathbb{O}_{d,r}$,
    \begin{align*}
        &\|\nu^2 (U^{\star \top} \Sigma_x U^\star)^{-1} - \nu^2 U^{\star \top} U (U^\top \Sigma_x U)^{-1} U^\top U^\star\|_2\\
        &\quad= O\qty( \frac{1}{1 - \|\sin\Theta(U, U^\star)\|_2^2 + \kappa^{-1}\rho^{-2}} \frac{1 + \rho^{-2}}{1 + \kappa^{-1}\rho^{-2}} \|\sin\Theta(U, U^\star)\|_2 ).
    \end{align*}
\end{lemma}

\begin{proof}
    Observe that
    \begin{align*}
        &\|(U^{\star \top} \Sigma_x U^\star)^{-1} - U^{\star \top} U (U^\top \Sigma_x U)^{-1} U^\top U^\star\|_2\\
        &\quad\leq \|(U^{\star \top} \Sigma_x U^\star)^{-1} - (U^\top \Sigma_x U)^{-1}\|_2 + \|(U^\top \Sigma_x U)^{-1} - U^{\star \top} U (U^\top \Sigma_x U)^{-1} U^\top {U^\star}\|_2\\
        &\quad:= (T1) + (T2).
    \end{align*}
    For the term $(T1)$,
    \begin{align*}
        (T1) &= \|(U^\top \Sigma_x U)^{-1} (U^\top \Sigma_x U) (U^{\star \top} \Sigma_x U^\star)^{-1} - (U^\top \Sigma_x U)^{-1} (U^{\star \top} \Sigma_x U^\star) (U^{\star \top} \Sigma_x U^\star)^{-1}\|_2\\
        &\leq \|(U^\top \Sigma_x U)^{-1}\|_2 \|U^\top \Sigma_x U - U^{\star \top} \Sigma_x U^\star\|_2 \|(U^{\star \top} \Sigma_x U^\star)^{-1}\|_2.
    \end{align*}
    Note
    \begin{align*}
        \|U^\top \Sigma_x U - U^{\star \top} \Sigma_x U^\star\|_2 &= \|\nu^2 U^\top U^\star U^{\star \top} U - \nu^2 I + U^\top \Sigma U - U^{\star \top} \Sigma U^\star\|_2\\
        &\leq \nu^2 \|\sin\Theta(U, U^\star)\|_2^2 + \|U^\top \Sigma (U - U^\star) + (U - U^\star)^\top \Sigma U^\star\|_2\\
        &\leq \nu^2 \|\sin\Theta(U, U^\star)\|_2^2 + 2 \sigma_{(1)}^2\|U - U^\star\|_2.
    \end{align*}
    Also we have $\lambda_{\min}(U^\top \Sigma_x U) \geq \nu^2(1 - \|\sin\Theta(U, U^\star)\|_2^2) + \sigma_{(d)}^2$ from the proof of Lemma \ref{lem: auxiliary lemma 2 for risk matrix} and $\lambda_{\min}(U^{\star \top} \Sigma_x U^\star) \geq \nu^2 + \sigma_{(d)}^2$.
    Therefore
    \begin{align*}
        (T1) &\leq \frac{1}{(\nu^2 + \sigma_{(d)}^2) (\nu^2(1 - \|\sin\Theta(U, U^\star)\|_2^2) + \sigma_{(d)}^2)} (\nu^2 \|\sin\Theta(U, U^\star)\|_2^2 + 2 \sigma_{(1)}^2 \|U - U^\star\|_2).
    \end{align*}

    For the term $(T2)$,
    \begin{align*}
        (T2) &= \| (U^\top \Sigma_x U)^{-1} - U^{\star \top} {(U^\star + (U - U^\star))}  (U^\top \Sigma_x U)^{-1} (U^\star + (U - U^\star))^\top U^\star\|_2\\
        &= \|- U^{\star \top} (U - U^\star) (U^\top \Sigma_x U)^{-1} - (U^\top \Sigma_x U)^{-1} (U - U^\star)^\top U^\star\\
        &\quad- U^{\star \top} {(U - U^\star)} (U^\top \Sigma_x U)^{-1} (U - U^\star)^\top U^\star\|_2\\
        &\leq \frac{1}{\nu^2(1 - \|\sin\Theta(U, U^\star)\|_2^2) + \sigma_{(d)}^2} (2 \|U - U^\star\|_2 + \|U - U^\star\|_2^2).
    \end{align*}
    From Lemma \ref{lem: bound 2 norm by sin distance}, $\|\sin\Theta(U, U^\star)\|_2 \leq \|U - U^\star\|_2$.
    Finally from these results and $\|U - U^\star\|_2^2 \leq 2\|U - U^\star\|_2$,
    \begin{align*}
        &\|\nu^2 (U^{\star \top} \Sigma_x U^\star)^{-1} - \nu^2 U^{\star \top} U (U^\top \Sigma_x U)^{-1} U^\top U^\star\|_2\\
        &\quad= O\qty( \frac{\nu^2}{\nu^2(1 - \|\sin\Theta(U, U^\star)\|_2^2) + \sigma_{(d)}^2} \frac{\nu^2 + \sigma_{(1)}^2}{\nu^2 + \sigma_{(d)}^2} \|U - U^\star\|_2 ).
    \end{align*}
    Since LHS does not depend on the orthogonal transformation $U \leftarrow U O$ where $O \in \mathbb{O}_{r,r}$, we obtain
    \begin{align*}
        &\|\nu^2 (U^{\star \top} \Sigma_x U^\star)^{-1} - \nu^2 U^{\star \top} U (U^\top \Sigma_x U)^{-1} U^\top U^\star\|_2\\
        &\quad= O\qty( \frac{\nu^2}{\nu^2(1 - \|\sin\Theta(U, U^\star)\|_2^2) + \sigma_{(d)}^2} \frac{\nu^2 + \sigma_{(1)}^2}{\nu^2 + \sigma_{(d)}^2} \inf_{O \in \mathbb{O}_{r,r}}\|U O - U^\star\|_2 ).
    \end{align*}
    Combined again with Lemma \ref{lem: bound 2 norm by sin distance}, we obtain the desired result.
\end{proof}

\begin{lemma}\label{lem: auxiliary lemma 4 for risk matrix}
    For any $U \in \mathbb{O}_{d,r}$,
    \begin{align*}
        &\lambda_{\min}(\nu^2 (U^{\star \top} \Sigma_x U^\star)^{-1} - \nu^2 U^{\star \top} U (U^\top \Sigma_x U)^{-1} U^\top U^\star)\\
        &\quad\geq \frac{\nu^2}{\nu^2 + \sigma_{(1)}^2} - \frac{\nu^2}{\sigma_{(d)}^2} (r - \|\sin\Theta(U, U^\star)\|_F^2).
    \end{align*}
\end{lemma}

\begin{proof}
    Observe
    \begin{align*}
        &\lambda_{\min}(\nu^2 (U^{\star \top} \Sigma_x U^\star)^{-1} - \nu^2 U^{\star \top} U (U^\top \Sigma_x U)^{-1} U^\top U^\star)\\
        &\quad\geq \lambda_{\min}((I + (1/\nu^2) U^{\star \top} \Sigma U^\star)^{-1}) - \|U^{\star \top} U (U^\top U^\star U^{\star \top} U + (1/\nu^2) U^\top \Sigma U)^{-1} U^\top U^\star\|_2.
    \end{align*}
    Since $U^\top U^\star U^{\star \top} U \succeq 0$, it follows that $(U^\top U^\star U^{\star \top} U + (1/\nu^2) U^\top \Sigma U)^{-1} \preceq \nu^2(U^\top \Sigma U)^{-1}$. Thus
    \begin{align*}
        &\|U^{\star \top} U (U^\top U^\star U^{\star \top} U + (1/\nu^2) U^\top \Sigma U)^{-1} U^\top U^\star\|_2\\
        &\quad\leq \nu^2 \lambda_{\max}((U^\top \Sigma U)^{-1}) \|U^{\star \top} U\|^2_2\\
        &\quad\leq \frac{\nu^2}{\sigma_{(d)}^2} \|U^{\star \top} U\|^2_F\\
        &\quad= \frac{\nu^2}{\sigma_{(d)}^2} (r - \|\sin\Theta(U, U^\star)\|_F^2),
    \end{align*}
    where we used $\lambda_{\max}((U^\top \Sigma U)^{-1}) \leq 1/\lambda_{\min}(U^\top \Sigma U) \leq 1/\sigma_{(d)}^2$
    and $\left\|\sin \Theta\left(U_{1}, U_{2}\right)\right\|_F^2 = r-\left\|U_{1}^\top U_{2}\right\|_F^2$ from Proposition \ref{prop: distance 1}.
    Combined with Lemma \ref{lem: auxiliary lemma 2 for risk matrix}, we obtain
    \begin{align*}
        &\lambda_{\min}(\nu^2 (U^{\star \top} \Sigma_x U^\star)^{-1} - \nu^2 U^{\star \top} U (U^\top \Sigma_x U)^{-1} U^\top U^\star)\\
        &\quad\geq \frac{\nu^2}{\nu^2 + \sigma_{(1)}^2} - \frac{\nu^2}{\sigma_{(d)}^2} (r - \|\sin\Theta(U, U^\star)\|_F^2).
    \end{align*}
\end{proof}

\begin{lemma}\label{lem: squared risk lb naive}
	Suppose the conditions in Theorem \ref{thm: recover PCA} hold. Fix $c_1 \in (0, 1)$. There exists a constant $c_2 > 0$ such that if $\sqrt{r\log d / d} \vee \rho^2 \vee d/n < c_2$, then,
    \begin{align*}
        \mathbb{E}_{\mathcal{D}}\|\sin\Theta(U_{\AEN}, U^\star)\|_F^2  &\geq c_1 r,
    \end{align*}
    where $c_1 \in (0, 1)$ is a universal constant.
\end{lemma}

\begin{proof}
    By Cauchy-Schwartz inequality,
    \begin{align*}
        &\mathbb{E}_{\mathcal{D}}\|\sin\Theta(U_{AE}, U^\star)\|_F^2 - r\\
        &\quad\geq (\mathbb{E}_{\mathcal{D}}\|\sin\Theta(U_{AE}, U^\star)\|_F)^2 - r\\
        &\quad= (\mathbb{E}_{\mathcal{D}}\|\sin\Theta(U_{AE}, U^\star)\|_F - \sqrt{r}) \qty(\mathbb{E}_{\mathcal{D}}\|\sin\Theta(U_{AE}, U^\star)\|_F + \sqrt{r}).
    \end{align*}
    From Theorem \ref{thm: recover PCA}, there exists a constant $c_3 > 0$ such that we have
	\begin{equation*}
		\mathbb{E}_{\mathcal{D}}\left\|\sin \Theta\left(U^\star, U_{AE}\right)\right\|_F \geq\sqrt{r}-c_3 \frac{r}{\sqrt{d}}\sqrt{\log d}-c_3\sqrt{r}\qty(\rho^2 + \sqrt{\frac{d}{n}} + \rho\sqrt{\frac{d}{n}}).
	\end{equation*}
	Therefore combined with a trivial bound $\|\sin\Theta(U_{AE}, U^\star)\|_F \leq \sqrt{r}$,
    \begin{align*}
        \mathbb{E}_{\mathcal{D}}\|\sin\Theta(U_{AE}, U^\star)\|_F^2 - r &\geq -r c_3\frac{r^{1/2}}{\sqrt{d}}\sqrt{\log d} + \rho^2 + \sqrt{\frac{d}{n}} + \rho\sqrt{\frac{d}{n}}\\
        &\geq -r c_3\qty(2\frac{r^{1/2}}{\sqrt{d}}\sqrt{\log d} \vee 6\rho^2 \vee 6\sqrt{\frac{d}{n}}),.
    \end{align*}
    where we used $\rho\sqrt{d/n} \leq \rho^2 \vee d/n \leq \rho^2 \vee \sqrt{d/n}$ since $d < n$.
    Thus we can take $c_2 = 6(1 - c_1)/c_3$.
    This concludes the proof.
\end{proof}

\begin{lemma}\label{lem: classification risk upper bound}
    For any $U \in \mathbb{O}_{d,r}$,
    \begin{align*}
        &\mathbb{E}_{\mathcal{D}}[\inf_{w \in \mathbb{R}^r} \mathcal{R}_c(\delta_{U, w}) - \inf_{w \in \mathbb{R}^r} \mathcal{R}_c(\delta_{U^\star, w})]\\
        &\quad\leq ((\kappa(1 + \rho^2))^3 + \kappa \rho^2 (1 +\rho^{-2})^2 + (\kappa\rho^2 \vee 1)^{-1}) \mathbb{E}_{\mathcal{D}}[\|\sin\Theta(U, U^\star)\|_2].
    \end{align*}
\end{lemma}

\begin{proof}
    Recall that we are considering the class of linear classifiers $\{\delta_{U, w}: w\in \mathbb{R}^r\}$, where $\delta_{U, w}(\check x) = \1\{F(\check x^\top U w) > 1/2\}$.
    For notational simplicity, write $\beta := U w$ and $\beta^\star := U^\star w^\star$.
    \begin{align*}
        \mathcal{R}_c(\delta_{U, w}) = \mathbb{P}_{\mathcal{E}}(\delta_{U, w}(\check x) \neq \check y) = \mathbb{P}_{\mathcal{E}}(\check y = 0, F(\check x^\top \beta) > 1/2) + \mathbb{P}_{\mathcal{E}}(\check y = 1, F(\check x^\top \beta) \leq 1/2).
    \end{align*}
    Since $F(0) = 1/2$ and $F$ is monotone increasing, the false positive probability becomes
    \begin{align*}
        \mathbb{P}_{\mathcal{E}}(\check y = 0, F(\check x^\top \beta) > 1/2)
        &= \mathbb{P}_{\mathcal{E}}(\check y = 0, \check x^\top \beta > 0)\\
        &= \mathbb{E}_{\mathcal{E}}[\mathbb{E}_{\mathcal{E}}[\1\{\check y = 0\} | \check x, \check z ] \1\{\check x^\top \beta > 0\}]\\
        &= \mathbb{E}_{\mathcal{E}}[(1 - F(\nu^{-1}\check z^\top U^{\star \top} \beta^\star)) \1\{\check x^\top \beta > 0\}].
    \end{align*}
    Write $\omega := \check x^\top \beta$ and $\omega^\star := \nu^{-1}\check z^\top U^{\star \top} \beta^\star$. 
    From assumption, $(\omega^\star, \omega)$ jointly follows a normal distribution with mean $0$. 
    Write ${v^\star}^2 := \Var(\omega^\star) = {w^\star}^\top w^\star$, $v^2 := \Var(\omega) = \beta^\top \Sigma_x \beta$,
    where $\Sigma_x := \nu^2 U^\star U^{\star \top} + \Sigma$.
    Let $\tau := \text{Cor}(\omega^\star, \omega) = \nu {w^\star}^\top U^{\star \top} \beta / (v^\star v)$.
    By a formula for conditional normal distribution, we have
    $\omega | \omega^\star \sim N(\tau v \omega^\star / v^\star, v^2 (1 - \tau^2))$. This gives
    \begin{align*}
        &\mathbb{P}_{\mathcal{E}}(\check y = 0, F(\check x^\top \beta) > 1/2)\\
        &\quad= \mathbb{E}_{\mathcal{E}}[(1 - F(\omega^\star)) \1\{\omega > 0\}]\\
        &\quad= \mathbb{E}_{\mathcal{E}}[(1 - F(\omega^\star)) \mathbb{E}_{\mathcal{E}}[\1\{\omega > 0\} | \omega^\star]]\\
        &\quad= \mathbb{E}_{\mathcal{E}}[(1 - F(\omega^\star)) \mathbb{P}_{\mathcal{E}}(\omega > 0 | \omega^\star)]\\
        &\quad= \mathbb{E}_{\mathcal{E}}\qty[(1 - F(\omega^\star)) \mathbb{P}_{\mathcal{E}}\qty(\frac{\omega - \tau v \omega^\star / v^\star}{v (1 - \tau^2)^{1/2}} > - \frac{\tau v \omega^\star / v^\star}{v (1 - \tau^2)^{1/2}} \middle| \omega^\star)]\\
        &\quad= \mathbb{E}_{\mathcal{E}}\qty[(1 - F(\omega^\star)) \Phi(\alpha \omega^\star / v^\star)]\\
        &\quad= \mathbb{E}_{\mathcal{E}}\qty[(1 - F(\omega^\star)) \Phi(\alpha \omega^\star / v^\star) \1\{\omega^\star > 0\}] + \mathbb{E}_{\mathcal{E}}\qty[(1 - F(\omega^\star)) \Phi(\alpha \omega^\star / v^\star) \1\{\omega^\star < 0\}],
    \end{align*}
    where $\Phi$ is cumulative distribution function of $N(0, 1)$ and $\alpha := \tau / (1 - \tau^2)^{1/2}$.
    We define $\Psi_F$ as $\Psi_F(s^2) := 2E_{u \sim N(0, s^2)}[F(u) \1\{u > 0\}]$.
    When $F(u) = 1/(1+e^{-u})$, $\Psi_F(s^2)$ is called the logistic-normal integral, whose analytical form is not known \citep{pirjol2013logistic}.
    Since a random variable $\omega^\star$ is symmetric about mean $0$ and $F(u) = 1 - F(-u)$,
    \begin{align*} 
        \mathbb{E}_{\mathcal{E}}\qty[(1 - F(\omega^\star)) \Phi(\alpha \omega^\star / v^\star) \1\{\omega^\star < 0\}] &= \mathbb{E}_{\mathcal{E}}\qty[(1 - F(-\omega^\star)) \qty(1 - \Phi(\alpha \omega^\star / v^\star)) \1\{\omega^\star > 0\}]\\
        &= \mathbb{E}_{\mathcal{E}}\qty[F(\omega^\star) \qty(1 - \Phi(\alpha \omega^\star / v^\star)) \1\{\omega^\star > 0\}].
    \end{align*}
    Hence
    \begin{align*}
        &\mathbb{P}_{\mathcal{E}}(\check y = 0, F(\check x^\top \beta) > 1/2)\\
        &\quad= \mathbb{E}_{\mathcal{E}}\qty[(\Phi(\alpha \omega^\star / v^\star) + F(\omega^\star) - 2 F(\omega^\star) \Phi(\alpha \omega^\star / v^\star))\1\{\omega^\star > 0\}]\\
        &\quad= \frac{1}{2}\Psi_F({v^\star}^2) - \mathbb{E}_{\mathcal{E}}\qty[(2F(\omega^\star) - 1) \Phi(\alpha \omega^\star / v^\star)\1\{\omega^\star > 0\}].
    \end{align*}
    Note that the true negative probability is exactly the same as the false positive probability under our settings:
    \begin{align*}
        \mathbb{P}_{\mathcal{E}}(\check y = 1, F(\check x^\top \beta) \leq 1/2) &= \mathbb{E}_{\mathcal{E}}[F(\check x^\top \beta^\star) \1\{\check x^\top \beta \leq 0\}]\\
        &= \mathbb{E}_{\mathcal{E}}[F(-\check x^\top \beta^\star) \1\{\check x^\top \beta \geq 0\}]\\
        &= \mathbb{E}_{\mathcal{E}}[(1 - F(\check x^\top \beta^\star)) \1\{\check x^\top \beta \geq 0\}]\\
        &= \mathbb{P}_{\mathcal{E}}(\check y = 0, F(\check x^\top \beta) > 1/2).
    \end{align*}
    Therefore
    \begin{align*}
        \mathcal{R}_c(\delta_{U, w}) &= \Psi_F({v^\star}^2) - 2\mathbb{E}_{\mathcal{E}}\qty[(2F(\omega^\star) - 1) \Phi(\alpha \omega^\star / v^\star)\1\{\omega^\star > 0\}].
    \end{align*}
    Let
    \begin{align*}
        \tau_{\max, U} &:= \sup_{w \in \mathbb{R}^r} \nu {w^\star}^\top U^{\star \top} U w / ({w^\star}^\top w^\star w^\top U^\top \Sigma_x U w)^{1/2},\\
        \tau_{\max, U^\star} &:= \sup_{w \in \mathbb{R}^r} \nu {w^\star}^\top w / ({w^\star}^\top w^\star w^\top U^{\star \top} \Sigma_x U^\star w)^{1/2}.
    \end{align*} 
    From Cauchy-Schwartz inequality,
    \begin{align*}
        \tau_{\max, U}^2 &= \frac{\nu^2{w^\star}^\top U^{\star \top} U (U^\top \Sigma_x U)^{-1} U^\top U^\star w^\star}{{w^\star}^\top w^\star},\\
        \tau_{\max, U^\star}^2 &= \frac{\nu^2{w^\star}^\top (U^{\star \top} \Sigma_x U^\star)^{-1} w^\star}{{w^\star}^\top w^\star}.
    \end{align*}

    Define $\alpha_{\max, U} := \tau_{\max, U} / (1 - \tau_{\max, U}^2)^{1/2}$ and $\alpha_{\max, U^\star} := \tau_{\max, U^\star} / (1 - \tau_{\max, U^\star}^2)^{1/2}$.
    Then, since on the event where $\omega^\star > 0$, $\alpha \mapsto \Phi(\alpha \omega^\star / v^\star)$ is monotone increasing and $2F(w^\star) - 1$ is non-negative, we have
    \begin{align*}
        \inf_{w \in \mathbb{R}^r} \mathcal{R}_c(\delta_{U, w}) &= \Psi_F({v^\star}^2) - 2\mathbb{E}_{\mathcal{E}}\qty[(2F(\omega^\star) - 1) \Phi(\alpha_{\max, U} \omega^\star / v^\star)\1\{\omega^\star > 0\}]\\
        \inf_{w \in \mathbb{R}^r} \mathcal{R}_c(\delta_{U^\star, w}) &= \Psi_F({v^\star}^2) - 2\mathbb{E}_{\mathcal{E}}\qty[(2F(\omega^\star) - 1) \Phi(\alpha_{\max, U^\star} \omega^\star / v^\star)\1\{\omega^\star > 0\}].
    \end{align*}
    This yields
    \begin{align*}
        &\inf_{w \in \mathbb{R}^r} \mathcal{R}_c(\delta_{U, w}) - \inf_{w \in \mathbb{R}^r} \mathcal{R}_c(\delta_{U^\star, w})\\
        &\quad= 2\mathbb{E}_{\mathcal{E}}\qty[(2F(\omega^\star) - 1) (\Phi(\alpha_{\max, U^\star} \omega^\star / v^\star) - \Phi(\alpha_{\max, U} \omega^\star / v^\star))\1\{\omega^\star > 0\}].
    \end{align*}
    Note that for any $a, b \geq 0$,
    \begin{align*}
        |\Phi(b) - \Phi(a)| \leq \phi(a \wedge b) |b - a|,
    \end{align*}
    where $\phi$ is a density function of standard normal distribution.
    Observe
    \begin{align*}
        &\inf_{w \in \mathbb{R}^r} \mathcal{R}_c(\delta_{U, w}) - \inf_{w \in \mathbb{R}^r} \mathcal{R}_c(\delta_{U^\star, w})\\
        &\quad\leq 2\mathbb{E}_{\mathcal{E}}\qty[(2F(\omega^\star) - 1) |\Phi(\alpha_{\max, U^\star} \omega^\star / v^\star) - \Phi(\alpha_{\max, U} \omega^\star / v^\star)|\1\{\omega^\star > 0\}]\\
        &\quad\lesssim \frac{2}{v^\star} \int_0^\infty (2F(\omega^\star) - 1) |\alpha_{\max, U^\star} - \alpha_{\max, U}| \omega^\star \phi((\alpha_{\max, U^\star} \wedge \alpha_{\max, U}) \omega^\star / v^\star) \frac{\phi(\omega^\star/v^\star)}{v^\star} \dd{\omega^\star}\\
        &\quad\lesssim \frac{|\alpha_{\max, U^\star} - \alpha_{\max, U}|}{v^\star} \int_0^\infty (2F(\omega^\star) - 1) \phi((\alpha_{\max, U^\star} \wedge \alpha_{\max, U}) \omega^\star / v^\star) \dd{\omega^\star}\\
        &\quad= \frac{|\alpha_{\max, U^\star} - \alpha_{\max, U}|}{\alpha_{\max, U^\star} \wedge \alpha_{\max, U}} \int_0^\infty (2F(\omega^\star) - 1) \frac{\exp(-1/(2((\alpha_{\max, U^\star} \wedge \alpha_{\max, U})^{-2} {v^\star}^2)) {\omega^\star}^2)}{\sqrt{2\pi ((\alpha_{\max, U^\star} \wedge \alpha_{\max, U})^{-2} {v^\star}^2)}} \dd{\omega^\star}\\
        &\quad= \frac{|\alpha_{\max, U^\star} - \alpha_{\max, U}|}{\alpha_{\max, U^\star} \wedge \alpha_{\max, U}} (\Psi_F(((\alpha_{\max, U^\star} \wedge \alpha_{\max, U^\star})^{-2} {v^\star}^2)) - 1/2),
    \end{align*}
    where we used $\sup_{u > 0} u \phi(u) < \infty$.
    Since $(a - b) = (a^2 - b^2)/(a + b) \leq (a^2 - b^2)/(a \wedge b)$ for $a, b > 0$, and $\Psi_F \leq 1$, we obtain
    \begin{align*}
        \inf_{w \in \mathbb{R}^r} \mathcal{R}_c(\delta_{U, w}) - \inf_{w \in \mathbb{R}^r} \mathcal{R}_c(\delta_{U^\star, w}) \lesssim \frac{|\alpha_{\max, U^\star}^2 - \alpha_{\max, U}^2|}{\alpha_{\max, U^\star}^2 \wedge \alpha_{\max, U}^2}. 
    \end{align*}
    When $\tau_{\max, U^\star} \geq \tau_{\max, U}$, since $\tau \mapsto \tau^2/(1 - \tau^2)$ is increasing in $\tau > 0$,
    \begin{align}
        \inf_{w \in \mathbb{R}^r} \mathcal{R}_c(\delta_{U, w}) - \inf_{w \in \mathbb{R}^r} \mathcal{R}_c(\delta_{U^\star, w}) &\lesssim \frac{\alpha_{\max, U^\star}^2 - \alpha_{\max, U}^2}{\alpha_{\max, U}^2}\nonumber\\
        &= \frac{\tau_{\max, U^\star}^2 - \tau_{\max, U}^2}{(1 - \tau_{\max, U^\star}^2) \tau_{\max, U}^2}. \label{eq: minimum risk difference by tau}
    \end{align}
    From Lemma \ref{lem: auxiliary lemma 1 for risk matrix} and \ref{lem: auxiliary lemma 2 for risk matrix}, we have
    \begin{align}
        \frac{\nu^2}{\nu^2 + \sigma_{(1)}^2} (1 - \|\sin\Theta(U, U^\star)\|_2^2) \leq \tau_{\max, U}^2 &\leq \frac{\nu^2}{\nu^2 (1 - \|\sin\Theta(U, U^\star)\|_2^2) + \sigma_{(d)}^2},\nonumber\\
        \frac{\nu^2}{\nu^2 + \sigma_{(1)}^2} \leq \tau_{\max, U^\star}^2 &\leq \frac{\nu^2}{\nu^2 + \sigma_{(d)}^2}.\label{eq: tau max bound}
    \end{align}
    Then, Equation \eqref{eq: minimum risk difference by tau} becomes
    \begin{align*}
        &\inf_{w \in \mathbb{R}^r} \mathcal{R}_c(\delta_{U, w}) - \inf_{w \in \mathbb{R}^r} \mathcal{R}_c(\delta_{U^\star, w})\\
        &\quad\lesssim \frac{\nu^2 + \sigma_{(d)}^2}{\sigma_{(d)}^2} \frac{\nu^2 + \sigma_{(1)}^2}{\nu^2 (1 - \|\sin\Theta(U, U^\star)\|_2^2)} (\tau_{\max, U^\star}^2 - \tau_{\max, U}^2)\\
        &\quad\leq \frac{\nu^2 + \sigma_{(d)}^2}{\sigma_{(d)}^2} \frac{\nu^2 + \sigma_{(1)}^2}{\nu^2 (1 - \|\sin\Theta(U, U^\star)\|_2^2)} \|\nu^2 (U^{\star \top} \Sigma_x U^\star)^{-1} - \nu^2 U^{\star \top} U (U^\top \Sigma_x U)^{-1} U^\top U^\star\|_2\\
        &\quad\leq \frac{(\kappa \rho^2 + 1)(\rho^{-2} + 1)^2}{(1 + \kappa^{-1}\rho^{-2})(1 - \|\sin\Theta(U, U^\star)\|_2^2)^2} \|\sin\Theta(U, U^\star)\|_2\\
        &\quad= \frac{\kappa\rho^2 (\rho^{-2} + 1)^2}{(1 - \|\sin\Theta(U, U^\star)\|_2^2)^2} \|\sin\Theta(U, U^\star)\|_2.
    \end{align*}
    where the last inequality follows from Lemma \ref{lem: auxiliary lemma 3 for risk matrix}.
    
    On the event where $\|\sin\Theta(U, U^\star)\|_2^2 \leq 1/2$,
    \begin{align*}
        \inf_{w \in \mathbb{R}^r} \mathcal{R}_c(\delta_{U, w}) - \inf_{w \in \mathbb{R}^r} \mathcal{R}_c(\delta_{U^\star, w}) \lesssim \kappa \rho^2 (1 +\rho^{-2})^2 \|\sin\Theta(U, U^\star)\|_2.
    \end{align*}
    
    When $\tau_{\max, U^\star} < \tau_{\max, U}$, on the event where $\|\sin\Theta(U, U^\star)\|_2 \leq \kappa^{-1}\rho^{-2} / 2$,
    \begin{align*}
        &\inf_{w \in \mathbb{R}^r} \mathcal{R}_c(\delta_{U, w}) - \inf_{w \in \mathbb{R}^r} \mathcal{R}_c(\delta_{U^\star, w})\\
        &\quad\lesssim \frac{\nu^2 + \sigma_{(1)}^2}{\nu^2} \frac{\nu^2(1 - \|\sin\Theta(U, U^\star)\|_2^2) + \sigma_{(d)}^2}{-\nu^2 \|\sin\Theta(U, U^\star)\|_2^2 + \sigma_{(d)}^2} (\tau_{\max, U}^2 - \tau_{\max, U^\star}^2)\\
        &\quad\leq \frac{(\nu^2 + \sigma_{(1)}^2)^2}{\nu^2} \frac{1}{-\nu^2 \|\sin\Theta(U, U^\star)\|_2^2 + \sigma_{(d)}^2}\\
        &\quad\quad\times \|\nu^2 (U^{\star \top} \Sigma_x U^\star)^{-1} - \nu^2 U^{\star \top} U (U^\top \Sigma_x U)^{-1} U^\top U^\star\|_2\\
        &\quad\leq \frac{(1 + \rho^{-2})^3}{(- \|\sin\Theta(U, U^\star)\|_2^2 + \kappa^{-1} \rho^{-2})^3} \|\sin\Theta(U, U^\star)\|_2\\
        &\quad\lesssim (\kappa(1 + \rho^2))^3 \|\sin\Theta(U, U^\star)\|_2,
    \end{align*}
    where we used Lemma \ref{lem: auxiliary lemma 3 for risk matrix} again.
    
    In summary, on the event where $\|\sin\Theta(U, U^\star)\|_2 \leq \kappa^{-1}\rho^{-2} / 2 \wedge 1/2$,
    \begin{align*}
        &\inf_{w \in \mathbb{R}^r} \mathcal{R}_c(\delta_{U, w}) - \inf_{w \in \mathbb{R}^r} \mathcal{R}_c(\delta_{U^\star, w})\\
        &\quad\lesssim ((\kappa(1 + \rho^2))^3 + \kappa \rho^2 (1 +\rho^{-2})^2) \|\sin\Theta(U, U^\star)\|_2.
    \end{align*}
    On the other hand, on the event where $\|\sin\Theta(U, U^\star)\|_2 > \kappa^{-1}\rho^{-2} / 2 \wedge 1/2$, we have a trivial inequality $\inf_{w \in \mathbb{R}^r} \mathcal{R}_c(\delta_{U, w}) - \inf_{w \in \mathbb{R}^r} \mathcal{R}_c(\delta_{U^\star, w}) \leq 1$. This gives
    \begin{align*}
        &\mathbb{E}_{\mathcal{D}}[\inf_{w \in \mathbb{R}^r} \mathcal{R}_c(\delta_{U, w}) - \inf_{w \in \mathbb{R}^r} \mathcal{R}_c(\delta_{U^\star, w})]\\
        &\quad\lesssim ((\kappa(1 + \rho^2))^3 + \kappa \rho^2 (1 +\rho^{-2})^2) \mathbb{E}_{\mathcal{D}}[\|\sin\Theta(U, U^\star)\|_2]\\
        &\quad\quad+ \mathbb{P}_{\mathcal{D}}(\|\sin\Theta(U, U^\star)\|_2 > \kappa^{-1}\rho^{-2} / 2 \wedge 1/2)\\
        &\quad\lesssim ((\kappa(1 + \rho^2))^3 + \kappa \rho^2 (1 +\rho^{-2})^2 + (\kappa\rho^2 \vee 1)) \mathbb{E}_{\mathcal{D}}[\|\sin\Theta(U, U^\star)\|_2],
    \end{align*}
    where the last inequality follows from Markov's inequality.
\end{proof}

\begin{lemma}\label{lem: classification risk lower bound}
    Suppose $U \in \mathbb{O}_{d,r}$ satisfies $1/(1 + \rho^2) - \kappa (r - \|\sin\Theta(U, U^\star)\|_F^2) \geq 0$. Then,
    \begin{align*}
        &\inf_{w \in \mathbb{R}^r} \mathcal{R}_c(\delta_{U, w}) - \inf_{w \in \mathbb{R}^r} \mathcal{R}_c(\delta_{U^\star, w})\\
        &\quad\gtrsim \frac{(1 + \rho^2)^{3/2}}{(1 + \kappa \rho^2)^{3/2}} \rho^2 \qty(\frac{1}{1 + \rho^2} - \kappa (r - \|\sin\Theta(U, U^\star)\|_F^2)).
    \end{align*}
\end{lemma}

\begin{proof}
    We firstly bound the term $\tau_{\max, U^\star}^2 - \tau_{\max, U}^2$.
    From Lemma \ref{lem: auxiliary lemma 4 for risk matrix},
    \begin{align}
        \tau_{\max, U^\star}^2 - \tau_{\max, U}^2 
        &\geq \lambda_{\min}(\nu^2 (U^{\star \top} \Sigma_x U^\star)^{-1} - \nu^2 U^{\star \top} U (U^\top \Sigma_x U)^{-1} U^\top U^\star)\nonumber\\
        &\geq \frac{\nu^2}{\nu^2 + \sigma_{(1)}^2} - \frac{\nu^2}{\sigma_{(d)}^2} (r - \|\sin\Theta(U, U^\star)\|_F^2)\label{eq: tau max lower bound}.
    \end{align}
    From assumption, RHS of Equation \eqref{eq: tau max lower bound} is non-negative. Then using the inequality $a - b = (a^2 - b^2)/(a + b) \geq (a^2 - b^2)/(2a)$ for $a \geq b \geq 0$,
    \begin{align*}
        \alpha_{\max, U^\star} - \alpha_{\max, U} &\gtrsim \frac{1}{\alpha_{\max, U^\star}} (\alpha_{\max, U^\star}^2 - \alpha_{\max, U^\star}^2)\\
        &\geq \frac{(1 - \tau_{\max, U^\star}^2)^{1/2}}{\tau_{\max, U^\star}} \frac{\tau_{\max, U^\star}^2 - \tau_{\max, U}^2}{(1 - \tau_{\max, U^\star}^2)(1 - \tau_{\max, U}^2)}.
    \end{align*}
    From Equation \eqref{eq: tau max bound} and Equation \eqref{eq: tau max lower bound},
    \begin{align}
        &\alpha_{\max, U^\star} - \alpha_{\max, U}\nonumber\\
        &\quad\gtrsim \qty(\frac{\nu^2 + \sigma_{(d)}^2}{\nu^2})^{1/2} \qty(\frac{\nu^2 + \sigma_{(1)}^2}{\sigma_{(1)}^2})^{3/2} \qty(\frac{\nu^2}{\nu^2 + \sigma_{(1)}^2} - \frac{\nu^2}{\sigma_{(d)}^2} (r - \|\sin\Theta(U, U^\star)\|_F^2))\nonumber\\
        &\quad= (1 + \kappa^{-1}\rho^{-2})^{1/2} (1 + \rho^2)^{3/2} \qty(\frac{\nu^2}{\nu^2 + \sigma_{(1)}^2} - \frac{\nu^2}{\sigma_{(d)}^2} (r - \|\sin\Theta(U, U^\star)\|_F^2)).\label{eq: alpha max lower bound}
    \end{align}
    From the proof of Lemma \ref{lem: classification risk upper bound},
    \begin{align*}
        &\inf_{w \in \mathbb{R}^r} \mathcal{R}_c(\delta_{U, w}) - \inf_{w \in \mathbb{R}^r} \mathcal{R}_c(\delta_{U^\star, w})\\
        &\quad= 2\mathbb{E}_{\mathcal{E}}\qty[(2F(\omega^\star) - 1) (\Phi(\alpha_{\max, U^\star} \omega^\star / v^\star) - \Phi(\alpha_{\max, U} \omega^\star / v^\star))\1\{\omega^\star > 0\}].
    \end{align*}
    Note that for any $b \geq a \geq 0$, $\Phi(b) - \Phi(a) \geq \phi(b) (b - a)$.
    Since we assume RHS of Equation \eqref{eq: tau max lower bound} is positive, $\alpha_{\max, U^\star} \geq \alpha_{\max, U}$. Thus on the event where $\omega^\star > 0$, $\alpha_{\max, U^\star} \omega^\star / v^\star \geq \alpha_{\max, U} \omega^\star / v^\star$.
    Observe
    \begin{align*}
        &\inf_{w \in \mathbb{R}^r} \mathcal{R}_c(\delta_{U, w}) - \inf_{w \in \mathbb{R}^r} \mathcal{R}_c(\delta_{U^\star, w})\\
        &\quad\geq 2\mathbb{E}_{\mathcal{E}}\qty[(2F(\omega^\star) - 1) \phi(\alpha_{\max, U^\star} \omega^\star / v^\star) (\alpha_{\max, U^\star} \omega^\star / v^\star - \alpha_{\max, U} \omega^\star / v^\star) \1\{\omega^\star > 0\}]\\
        &\quad= \frac{2}{v^\star} (\alpha_{\max, U^\star} - \alpha_{\max, U}) \int_0^\infty (2F(\omega^\star) - 1) \omega^\star \frac{\phi(\omega^\star/v^\star)}{v^\star} \phi(\alpha_{\max, U^\star} \omega^\star / v^\star) \dd{\omega^\star}\\
        &\quad\simeq \frac{\alpha_{\max, U^\star} - \alpha_{\max, U}}{v^\star} \int_0^\infty (2F(\omega^\star) - 1)  \omega^\star \exp(-(1/2)(1 + \alpha_{\max, U^\star}^2) {\omega^\star}^2/ {v^\star}^2) \dd{\omega^\star}\\
        &\quad\simeq \frac{\alpha_{\max, U^\star} - \alpha_{\max, U}}{1 + \alpha_{\max, U^\star}^2} \int_0^\infty (2F((1 + \alpha_{\max, U^\star}^2)^{-1/2} v^\star \omega^\star) - 1)  \omega^\star \exp(-(1/2) {\omega^\star}^2) \dd{\omega^\star},
    \end{align*}
    where in the last equality we transformed $w^\star \to (1 + \alpha_{\max, U^\star}^2)^{1/2} w^\star / v^\star$.
    Since $F(u)$ is differentiable at $0$ and $F(0)=1/2$,
    \begin{align*}
        F(u) - 1/2 = F'(0) u + o(u).
    \end{align*}
    Thus there exists a constant $\epsilon > 0$ only depending on $F$ such that $2(F(u) - 1/2) \geq F'(0)u$ for all $u \in [0, \epsilon]$ since $F'(0) > 0$.
    This gives
    \begin{align*}
        &\inf_{w \in \mathbb{R}^r} \mathcal{R}_c(\delta_{U, w}) - \inf_{w \in \mathbb{R}^r} \mathcal{R}_c(\delta_{U^\star, w})\\
        &\quad\gtrsim \frac{\alpha_{\max, U^\star} - \alpha_{\max, U}}{1 + \alpha_{\max, U^\star}^2} F'(0)(1 + \alpha_{\max, U^\star}^2)^{-1/2} v^\star\\
        &\quad\quad\times \int_0^{\epsilon (1 + \alpha_{\max, U^\star}^2)^{1/2} v^\star} {\omega^\star}^2 \exp(-(1/2) {\omega^\star}^2) \dd{\omega^\star}\\
        &\quad\gtrsim \frac{\alpha_{\max, U^\star} - \alpha_{\max, U}}{1 + \alpha_{\max, U^\star}^2} (1 + \alpha_{\max, U^\star}^2)^{-1/2} v^\star \int_0^{\epsilon v^\star} {\omega^\star}^2 \exp(-(1/2) {\omega^\star}^2) \dd{\omega^\star}\\
        &\quad\gtrsim \frac{\alpha_{\max, U^\star} - \alpha_{\max, U}}{1 + \alpha_{\max, U^\star}^2} (1 + \alpha_{\max, U^\star}^2)^{-1/2}.
    \end{align*}
    The last inequality follows since $v^\star = \|w^\star\| = 1$ by assumption.
    It is noted that $\alpha_{\max, U^\star}^2 \leq \nu^2/\sigma_{(d)}^2$ from Equation \eqref{eq: tau max bound}.
    Therefore with Equation \eqref{eq: alpha max lower bound},
    \begin{align*}
        &\inf_{w \in \mathbb{R}^r} \mathcal{R}_c(\delta_{U, w}) - \inf_{w \in \mathbb{R}^r} \mathcal{R}_c(\delta_{U^\star, w})\\
        &\quad\gtrsim \frac{1}{(1 + \kappa \rho^2)^{3/2}} (1 + \kappa^{-1}\rho^{-2})^{1/2} (1 + \rho^2)^{3/2} \qty(\frac{1}{1 + \rho^{-2}} - \kappa \rho^2 (r - \|\sin\Theta(U, U^\star)\|_F^2))\\
        &\quad\gtrsim \frac{(1 + \rho^2)^{3/2}}{(1 + \kappa \rho^2)^{3/2}} \rho^2 \qty(\frac{1}{1 + \rho^2} - \kappa (r - \|\sin\Theta(U, U^\star)\|_F^2)).
    \end{align*}
\end{proof}

\begin{proposition}\label{prop: excess_risk}
    For any $U \in \mathbb{O}_{d,r}$,
    \begin{align*}
        \inf_{w\in\mathbb{R}^r} \mathcal{R}_r(\delta_{U, w}) &= \nu^2 {w^\star}^\top (I - \nu^2 U^{\star \top} U (\nu^2 U^\top U^\star U^{\star \top} U + U^\top \Sigma U)^{-1} U^\top U^\star) w^\star + \sigma_\epsilon^2.
    \end{align*}
\end{proposition}

\begin{proof}[Proof of Proposition \ref{prop: excess_risk}]
    Generate random variables $(\check x, \check z, \check \xi, \check \epsilon)$ following the model \eqref{model: spiked covariance regression}.
    We calculate the prediction risk of $\delta_{U, w}$ as:
    \begin{align*}
        \mathcal{R}_r(\delta_{U, w}) &:= \mathbb{E}_{\mathcal{E}}(\check y - \check x^\top U w)^2\\
        &= \Var_{\mathcal{E}}(\nu^{-1}\check z^\top w^\star + \check \epsilon)^2 - 2 \Cov_{\mathcal{E}}(\nu^{-1}\check z^\top w^\star + \check \epsilon, U^\star \check z + \check \xi) U w\\
        &\quad+ w^\top U^\top \Var_{\mathcal{E}}(U^\star \check z + \check \xi) U w\\
        &= \|w^\star\|^2 + \sigma_\epsilon^2 - 2 \nu {w^\star}^\top U^{\star \top} U w + w^\top (\nu^2 U^\top U^\star U^{\star \top} U + U^\top \Sigma U) w\\
        &= (w - A^{-1} b)^\top A (w - A^{-1} b) - b^\top A^{-1} b + \|w^\star\|^2 + \sigma_\epsilon^2,
    \end{align*}
    where $A := \nu^2 U^\top U^\star U^{\star \top} U + U^\top \Sigma U$ and $b := \nu U^\top U^\star w^\star$.
    From this, we obtain
    \begin{align*}
        \inf_{w\in\mathbb{R}^r} \mathcal{R}_r(\delta_{U, w}) &= {w^\star}^\top \qty(I - U^{\star \top} U (U^\top U^\star U^{\star \top} U + (1/\nu^2) U^\top \Sigma U)^{-1} U^\top U^\star) w^\star + \sigma_\epsilon^2.
    \end{align*}
\end{proof}

\begin{lemma}\label{lem: excess risk}
    For any $U \in \mathbb{O}_{d,r}$,
    \begin{align*}
        \inf_{w \in \mathbb{R}^r} \mathcal{R}_r(\delta_{U, w}) - \inf_{w \in \mathbb{R}^r} \mathcal{R}_r(\delta_{U^\star, w}) = O\qty( (1 + \rho^{-2}) \mathbb{E}_{\mathcal{D}} [\|\sin\Theta(U, U^\star)\|_2] \|w^\star\|^2).
    \end{align*}    
\end{lemma}

\begin{proof}[Proof of Lemma \ref{lem: excess risk}]
    From proposition \ref{prop: excess_risk}, we have 
    \begin{align*}
        &\inf_{w \in \mathbb{R}^r} \mathcal{R}_r(\delta_{U, w}) - \inf_{w \in \mathbb{R}^r} \mathcal{R}_r(\delta_{U^\star, w})\\
        &\quad= {w^\star}^\top \qty( (I + (1/\nu^2) U^{\star \top} \Sigma U^\star)^{-1} - U^{\star \top} U (U^\top U^\star U^{\star \top} U + (1/\nu^2) U^\top \Sigma U)^{-1} U^\top {U^\star}) w^\star.
    \end{align*}
    Note that $\inf_{w \in \mathbb{R}^r} \mathcal{R}_r(\delta_{U, w}) - \inf_{w \in \mathbb{R}^r} \mathcal{R}_r(\delta_{U^\star, w}) \equiv \inf_{w \in \mathbb{R}^r} \mathcal{R}_r(\delta_{U O, w}) - \inf_{w \in \mathbb{R}^r} \mathcal{R}_r(\delta_{U^\star, w})$ for any orthogonal matrix $O \in \mathbb{O}_{r,r}$.
    Take $\tilde O \in \mathbb{O}_{r,r}$ such that $\|U \tilde O - U^\star\|_2 \leq \sqrt{2} \|\sin\Theta(U, O)\|_2$ without loss of generality, since we can always take a sequence $(\tilde O_m)_{m \geq 1}$ such that 
    $\|U O_m - U^\star\|_2 \leq \sqrt{2} \|\sin\Theta(U, O)\|_2 + 1/m$ from Lemma \ref{lem: bound 2 norm by sin distance}. 

    Lemma \ref{lem: auxiliary lemma 3 for risk matrix} gives
    \begin{align*}
        &\inf_{w \in \mathbb{R}^r} \mathcal{R}_r(\delta_{U, w}) - \inf_{w \in \mathbb{R}^r} \mathcal{R}_r(\delta_{U^\star, w})\\
        &\quad= O\qty( \frac{1}{1 - \|\sin\Theta(U, U^\star)\|_2^2 + \kappa^{-1}\rho^{-2}} \frac{1 + \rho^{-2}}{1 + \kappa^{-1}\rho^{-2}} \|\sin\Theta(U, U^\star)\|_2 \|w^\star\|^2 ).
    \end{align*}
    On the event where $\|\sin\Theta(U, U^\star)\|_2^2 < 1/2$,
    \begin{align*}
        \inf_{w \in \mathbb{R}^r} \mathcal{R}_r(\delta_{U, w}) - \inf_{w \in \mathbb{R}^r} \mathcal{R}_r(\delta_{U^\star, w}) = O\qty( \frac{1 + \rho^{-2}}{(1 + \kappa^{-1}\rho^{-2})^2} \|\sin\Theta(U, U^\star)\|_2 \|w^\star\|^2).
    \end{align*}
    On the event where $\|\sin\Theta(U, U^\star)\|_2^2 \geq 1/2$, we utilize the trivial upper bound
    \begin{align*}
        \inf_{w \in \mathbb{R}^r} \mathcal{R}_r(\delta_{U, w}) - \inf_{w \in \mathbb{R}^r} \mathcal{R}_r(\delta_{U^\star, w}) &\leq \|(I + \nu^{-2} U^{\star \top} \Sigma U^\star)^{-1}\|_2 \|w^\star\|^2 \leq \frac{\nu^2}{\nu^2 + \sigma_{(d)}^2} \|w^\star\|^2.
    \end{align*}
    Combining these results, we have
    \begin{align*}
        &\mathbb{E}_{\mathcal{D}} [ \inf_{w \in \mathbb{R}^r} \mathcal{R}_r(\delta_{U, w}) - \inf_{w \in \mathbb{R}^r} \mathcal{R}_r(\delta_{U^\star, w}) ]\\
        &\quad\lesssim \frac{1 + \rho^{-2}}{(1 + \kappa^{-1}\rho^{-2})^2} \mathbb{E}_{\mathcal{D}} [\|\sin\Theta(U, U^\star)\|_2] \|w^\star\|^2\\
        &\quad\quad+ \frac{1}{1 + \kappa^{-1}\rho^{-2}} \|w^\star\|^2 \mathbb{P}_{\mathcal{D}}(\|\sin\Theta(U, U^\star)\|_2 \geq 1/\sqrt{2})\\
        &\quad\lesssim \frac{1 + \rho^{-2}}{(1 + \kappa^{-1}\rho^{-2})^2} \mathbb{E}_{\mathcal{D}} [\|\sin\Theta(U, U^\star)\|_2] \|w^\star\|^2,
    \end{align*}
    where the last inequality follows by Markov's inequality.
\end{proof}

\section{Discussion about Autoencoders and random masking augmentation}
\label{sec: MAE}
In the following, we show that our results do not change if we applied the same augmentation (\ref{aug: random masking}) for autoencoders. As discussed in Section \ref{sec: PCA and autoencoder}, we can ignore the bias term in autoencoders for simplicity, which only serves as centralization of the data matrix. In that case, we applied random augmentation $g_1(x)=Ax$ and $g_2(x)=(I-A)x$ to the original data $\{x_i\}_{i=1}^n$, and the optimization problem can be formulated as follows:
\begin{equation}
\begin{aligned}
\label{Augmented autoencoder}
    \min_{W_{AE}, W_{DE}}\frac{1}{2n}\mathbb{E}_{A}[\|AX-W_{DE}W_{AE}AX\|_F^2+\|(I-A)X-W_{DE}W_{AE}(I-A)X\|_F^2].
\end{aligned}
\end{equation}
Then, similar to Theorem \ref{prop: augment} for contrastive learning, we can also obtain an explicit solution for this optimization problem.
\begin{theorem}
\label{thm: augmented autoencoder}
    The optimal solution of autoencoders with random masking augmentation (\ref{Augmented autoencoder}) is given by:
	\begin{equation*}
		W_{AE}=W_{DE}^\top = C\left(\sum_{i=1}^{r}u_i\sigma_i v_i^\top\right)^\top,
	\end{equation*}
	where $C>0$ is a positive constant, $\sigma_i$ is the $i$-th largest eigenvalue of the following matrix:
	\begin{equation}
		\frac{1}{2}\Delta(XX^\top)+D(XX^\top),
	\end{equation}
	$u_i$ is the corresponding eigenvector and $V=[v_1,\cdots,v_r]\in\mathbb{R}^{r\times r}$ can be any orthonormal matrix.
\end{theorem}
\begin{proof}
We first derive the equivalent form for this objective function:
\begin{equation}
    \begin{aligned}
        &\frac{1}{2n}\mathbb{E}_{A}[\|AX-W_{DE}W_{AE}AX\|_F^2+\|(I-A)X-W_{DE}W_{AE}(I-A)X\|_F^2]\\
        =&\frac{1}{2n}\mathbb{E}_{A}[\tr(X^\top A^\top AX)+\tr(X^\top A^\top W_{DE}W_{AE}AX)+\tr(X^\top A^\top W_{AE}^\top W_{DE}^\top W_{DE}W_{AE}AX)\\
        &+\tr(X^\top (I-A)^\top (I-A)X)+\tr(X^\top (I-A)^\top W_{DE}W_{AE}(I-A)X)\\
        &+\tr(X^\top (I-A)^\top W_{AE}^\top W_{DE}^\top W_{DE}W_{AE}(I-A)X)]\\
        =&\frac{1}{2n}\mathbb{E}_{A}[\tr(X^\top AX)+\tr(AXX^\top A^\top W_{DE}W_{AE})+\tr(AXX^\top A^\top W_{AE}^\top W_{DE}^\top W_{DE}W_{AE})\\
        &+\tr(X^\top (I-A)X)+\tr((I-A)XX^\top (I-A)^\top W_{DE}W_{AE})\\
        &+\tr((I-A)XX^\top (I-A)^\top W_{AE}^\top W_{DE}^\top W_{DE}W_{AE})]\\
        =&\frac{1}{2n}\mathbb{E}_{A}[\tr(X^\top X)+\tr(\hat M W_{DE}W_{AE})+\tr(\hat M W_{AE}^\top W_{DE}^\top W_{DE}W_{AE})],\\
    \end{aligned}
\end{equation}
where $\hat M:= AXX^\top A^\top+(I-A)XX^\top (I-A)^\top$. Note that by Definition \ref{aug: random masking} we have $A=\diag(a_1,\cdots,a_d)$ and $a_i$ follows the Bernoulli distribution, so we have:
\begin{equation}
    \mathbb{E}_A \hat M= \frac{1}{2}\Delta(XX^\top)+D(XX^\top)
\end{equation}   
Again, by Theorem 2.4.8 in \citet{van1996matrix}, the optimal solution of Eq.(\ref{Augmented autoencoder}) is given by the eigenvalue decomposition of $\mathbb{E}_A \hat M= \frac{1}{2}\Delta(XX^\top)+D(XX^T)$, up to an orthogonal transformation, which finishes the proof.
\end{proof}
With Theorem \ref{thm: augmented autoencoder} established, we can now derive the space distance for autoencoders with random masking augmentation. 
\begin{theorem}
\label{thm: random masking autoencoder}
	Consider the spiked covariance model Eq.(\ref{model: spiked covariance}), under Assumptions \ref{asm: regular}-\ref{asm: incoherent} and $n> d\gg r$, let $W_{AE}$ be the learned representation of augmented autoencoder with singular value decomposition $W_{AE}=(U_{AE}\Sigma_{AE}V_{AE}^\top)^\top$ (i.e., the optimal solution of optimization problem \ref{Augmented autoencoder}). If we further assume $\{\sigma_{i}^2\}_{i=1}^d$ are different from each other and $\sigma_{(1)}^2/(\sigma_{(r)}^2-\sigma_{(r+1)}^2)<C_{\sigma}$ for some universal constant $C_{\sigma}$. Then there exist two universal constants $C_\rho>0,c\in (0,1)$, such that when $\rho<C_\rho$, we have
	\begin{equation}
		\mathbb{E}\left\|\sin \Theta\left(U^\star, U_{AE}\right)\right\|_F \geq c\sqrt{r}.
	\end{equation}
\end{theorem}
\begin{proof}
    Step1, similar to the proof of Theorem \ref{thm: recover PCA ap}, we first bound the difference between $\hat{M}:=\Delta(XX^\top)+2D(XX^\top)$ and $\Sigma:=\Cov(\xi\xi^\top)$. Note that:
    \begin{equation}
        \|\hat{M}-\Sigma\|_2=\|XX^\top-\Sigma-\frac{1}{2}\Delta(XX^\top)\|_2\leq\|XX^\top-\Sigma\|_2+\frac{1}{2}\|\Delta(XX^\top-\Sigma)\|_2+\frac{1}{2}\|\Delta(\Sigma)\|_2
    \end{equation}
    Since $\Sigma$ is a diagonal matrix, then by Lemma \ref{Lemma: Delta} we have:
    \begin{equation}
        \|\hat{M}-\Sigma\|_2\leq 2 \|XX^\top-\Sigma\|_2
    \end{equation}
    Now, directly apply equation (\ref{PCA step1 term1 b})(\ref{PCA step1 term2})(\ref{PCA step1 term3}) we can obtain that:
	\begin{equation}
		\mathbb{E}\|\hat{M}-\Sigma\|_2\lesssim \nu^2\qty(1+\sqrt{\frac{r}{n}}+\frac{r}{n})+\sigma_{(1)}^2\qty(\sqrt{\frac{d}{n}}+\frac{d}{n})+\sqrt{\frac{d}{n}}\sigma_{(1)}\nu.
	\end{equation}
	
	Step 2, bound the $\sin\Theta$ distance between eigenspaces. As we have shown in step 1, the target matrix of the autoencoder is close to the covariance matrix of random noise, i.e., $\Sigma$. Note that $\Sigma$ is assumed to be a diagonal matrix with different elements, hence its eigenspace only consists of canonical basis $e_i$. Denote $U_{\Sigma}$ to be the top-$r$ eigenspace of $\Sigma$ and $\{e_i\}_{i\in C}$ to be its corresponding basis vectors, apply the Davis-Kahan Theorem \ref{Lemma: DK} we can conclude that:
	\begin{equation*}
		\begin{aligned}
			&\mathbb{E}\|\sin\Theta(U_{AE},U_{\Sigma})\|_F\leq\frac{2\sqrt{r}\mathbb{E}\|\hat{M}-\Sigma\|_2}{\sigma_{(r)}^2-\sigma_{(r+1)}^2}\\
			\lesssim& \sqrt{r}\frac{1}{\sigma_{(1)}^2}\qty(\nu^2\qty(1+\sqrt{\frac{r}{n}}+\frac{r}{n})+\sigma_{(1)}^2\qty(\sqrt{\frac{d}{n}}+\frac{d}{n})+\sqrt{\frac{d}{n}}\sigma_{(1)}\nu)\\
			\lesssim& \sqrt{r} \qty(\rho^2+\sqrt{\frac{d}{n}}+\rho\sqrt{\frac{d}{n}}).\\
		\end{aligned}
	\end{equation*}
	
	Step 3, obtain the final result by triangular inequality. By Assumption \ref{asm: incoherent} we know that the distance between canonical basis and the eigenspace of core features can be large:
	\begin{equation}
		\begin{aligned}
			\|\sin\Theta(U^\star ,U_{\Sigma})\|_F^2&=\|U_{\Sigma\perp}^\top U^\star \|_F^2=\sum_{i\in[d]/ C}\|e_i^\top U^\star \|^2=\|U^\star \|_F^2-\sum_{i\in C}\|e_i^\top U^\star \|^2\\
			&\geq r-rI(U^\star )= r-O\qty(\frac{r^2}{d}\log d).
		\end{aligned}
	\end{equation}
	Then apply the triangular inequality of $\sin\Theta$ distance (Proposition \ref{triangular inequality of sin theta}) we can obtain the lower bound of the autoencoder.
	\begin{equation*}
		\begin{aligned}
			\mathbb{E}\|\sin\Theta(U_{AE},U^\star )\|_F&\geq\mathbb{E}\|\sin\Theta(U^\star ,U_\Sigma)\|_F-\mathbb{E}\|\sin\Theta(U_{AE},U_\Sigma)\|_F\\
			&\geq\sqrt{r}-O\qty(\frac{r}{\sqrt{d}}\sqrt{\log d})-O\qty(\sqrt{r} \qty(\rho^2+\sqrt{\frac{d}{n}}+\rho\sqrt{\frac{d}{n}})).
		\end{aligned}
	\end{equation*}
	By Assumption \ref{asm: SNR}, it implies that when n and d are sufficiently large and $\rho$ is sufficiently small (smaller than a given constant $C_\rho>0$), there exists a universal constant $c\in(0,1)$ such that:
	\begin{equation*}
	    \mathbb{E}\|\sin\Theta(U_{AE},U^\star )\|_F\geq c\sqrt{r}.
	\end{equation*}
\end{proof}
Compared with Theorem \ref{thm: recover PCA}, we can find that random masking augmentation makes no difference to autoencoders, which justifies the fairness of our comparison between contrastive learning and autoencoders.

However, contrary to the autoencoders with random-masking augmentation, we show that the representations obtained by DAEs behave as the representations obtained by contrastive learning.
\begin{proof}[Proof of Remark \ref{rem: masked ae}]\label{proof: masked 1}
    Let $\mathcal{L} := (1/n) \|X - W^\top W A X\|_F^2$ be the loss function of DAEs. Then,
    \begin{align}
        \mathbb{E}_A \mathcal{L} = \frac{1}{n}\tr( W^\top W \qty( \frac{1}{2} D(X X^\top) + \frac{1}{4} \Delta(X X^\top) ) W^\top W -  W^\top W X X^\top ) + (\text{const}.).
    \end{align}
    We minimize the loss over $W$ such that $W W^\top = 2 I_r$.
    Then, the loss becomes
    \begin{align*}
        \argmin_{W W^\top = 2 I_r} E_A \mathcal{L} &= \argmin_{W W^\top = 2 I_r} \frac{1}{n}\tr( W \qty(D(XX^\top) + \frac{1}{2} \Delta(XX^\top)) W^\top - W XX^\top W^\top )\\
        &= \argmax_{W W^\top = 2 I_r} \tr( W \frac{1}{n} \Delta(XX^\top) W^\top ).
    \end{align*}
    Thus, the solution to the (expected) loss minimization problem is the top-$r$ eigenvectors of $\Delta(n^{-1} XX^\top)$, i.e., $W^\top = \sqrt{2} O P_r(\Delta(n^{-1} XX^\top))$, where $O$ is any orthogonal matrix from $\mathcal{O}_{r,r}$.
    We use the same argument as in the proof of Theorem \ref{thm: recover CL ap}.
    First note that
    \begin{align*}
        \frac{1}{n}\Delta(XX^\top) = \frac{1}{n} \Delta(U^\star Z Z^\top U^{\star \top}) + \frac{1}{n} \Delta(U^\star Z E + E Z^\top U^{\star \top}) + \frac{1}{n} \Delta(E E^\top).
    \end{align*}
    By Lemmas \ref{Lemma: Delta}, \ref{Lemma: bound ZZT} and the incoherent condition $I(U)=O(\frac{r}{d}\log d)$, we have:
	\begin{align}
		&\mathbb{E}\norm{\Delta\qty(\frac{1}{n}U^\star ZZ^\top U^{\star \top})-\nu^2 U^\star U^{\star \top}}_2\nonumber\\
		&\leq 2\mathbb{E} \norm{\frac{1}{n}U^\star ZZ^\top U^{\star \top}-\nu^2 U^\star U^{\star \top}}_2 + \mathbb{E}\norm{\Delta\qty(\nu^2 U^\star U^{\star \top})-\nu^2 U^\star U^{\star \top} )}_2\nonumber\\
		&\lesssim 2(\sqrt{\frac{r}{n}}+\frac{r}{n})\nu^2 + \frac{r}{d}\log d\nu^2.\label{MAE step1 term1}
	\end{align}
	
	For the second term, applying equation (\ref{PCA step1 term2}) yields:
	\begin{equation}
		\begin{aligned}
			\label{MAE step1 term2}
			\frac{1}{n}\mathbb{E}\|\Delta(U^\star ZE^\top+EZ^\top U^{\star \top})\|_2\leq\frac{4}{n}\mathbb{E}\|EZ^\top U^{\star \top}\|_2\lesssim\frac{\sqrt{d}}{\sqrt{n}}\sigma_{(1)}\nu.
		\end{aligned}
	\end{equation} 
		
    For the third term, applying equation (\ref{PCA step1 term3}) yields:
	\begin{equation}
		\label{MAE step1 term3}
		\mathbb{E}\|\frac{1}{n}\Delta(EE^\top)\|_2=\mathbb{E}\|\Delta(\frac{1}{n}EE^\top-\Sigma)\|_2\leq 2 \|\frac{1}{n}EE^\top-\Sigma\|_2\lesssim(\sqrt{\frac{d}{n}}+\frac{d}{n})\sigma_{(1)}^2.
	\end{equation}
	
	Combining equations (\ref{MAE step1 term1})(\ref{MAE step1 term2})(\ref{MAE step1 term3}) gives
	\begin{equation*}
		\mathbb{E}\norm{\Delta\qty(\frac{1}{n} XX^\top - \nu^2 U^\star U^{\star \top})}_2 \lesssim \nu^2\qty(\frac{r}{d}\log d+\sqrt{\frac{r}{n}}+\frac{r}{n})+\sigma_{(1)}^2\qty(\sqrt{\frac{d}{n}}+\frac{d}{n})+\sigma_{(1)}\nu\sqrt{\frac{d}{n}}.
	\end{equation*}
	From Lemma \ref{Lemma: DK}, we obtain the desired bound:
	\begin{align*}
	    &\mathbb{E}\|\sin\Theta(U_{\DAE},U^\star )\|_F \leq \frac{2\sqrt{r}}{\nu^2} \mathbb{E}\norm{\Delta\qty(\frac{1}{n} XX^\top - \nu^2 U^\star U^{\star \top})}_2 \lesssim \frac{r^{3/2}}{d}\log d+\sqrt{\frac{dr}{n}}.
    \end{align*}
\end{proof}

Here we provide some experimental results about DAEs on synthetic datasets as analog to Figure \ref{fig: CL and AE} and \ref{fig: theta distance}, the settings are the same as described in Section \ref{sec: synthetic}. The results are summarized in Figure \ref{fig: MAE error}, as we can observe, the performance of DAEs is comparable with contrastive learning, which aligns with our theoretical results above.

\begin{figure}
	\centering
	\includegraphics[width=0.4\linewidth]{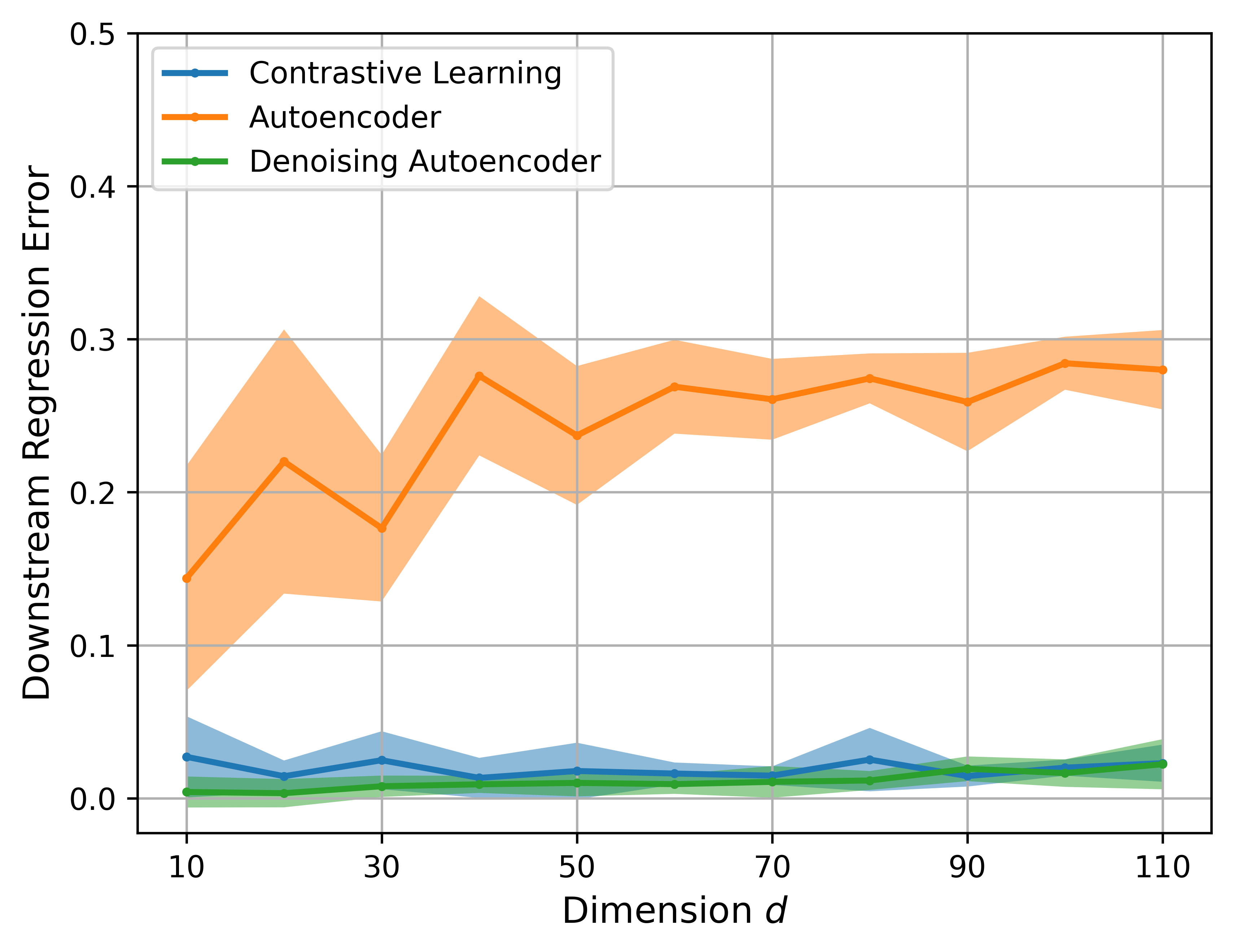}
    \includegraphics[width=0.39\linewidth]{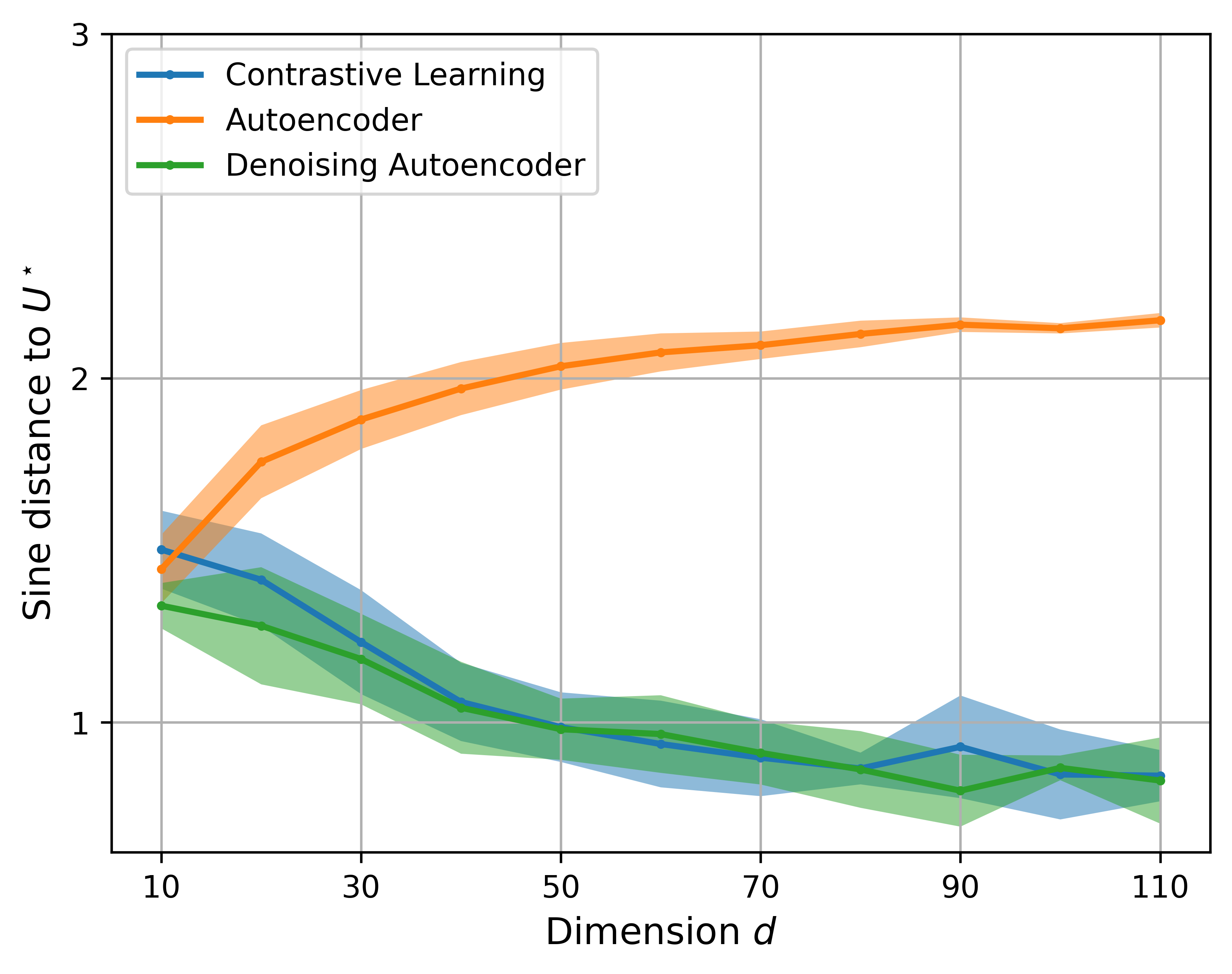}
    
    \includegraphics[width=0.4\linewidth]{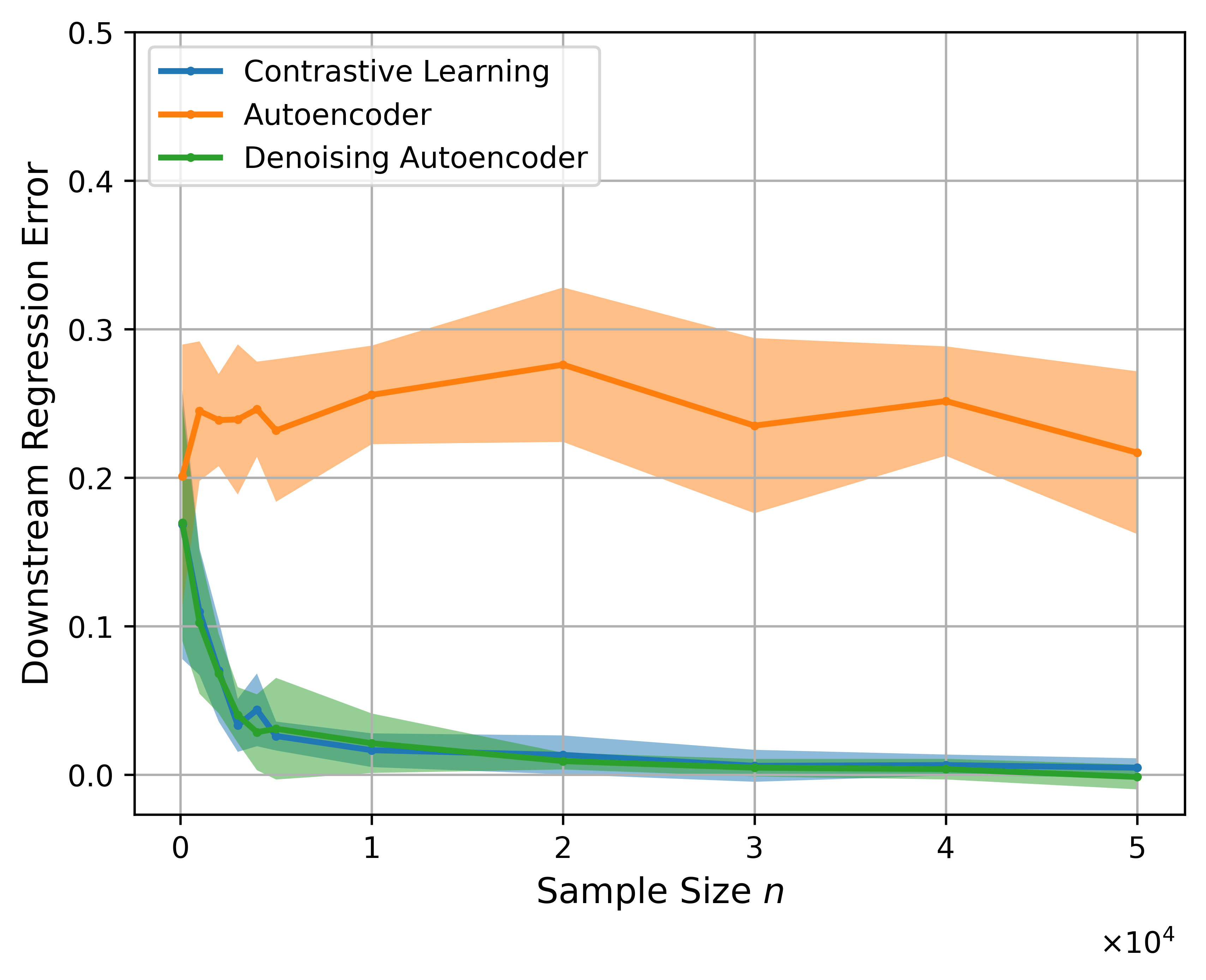}
    \includegraphics[width=0.39\linewidth]{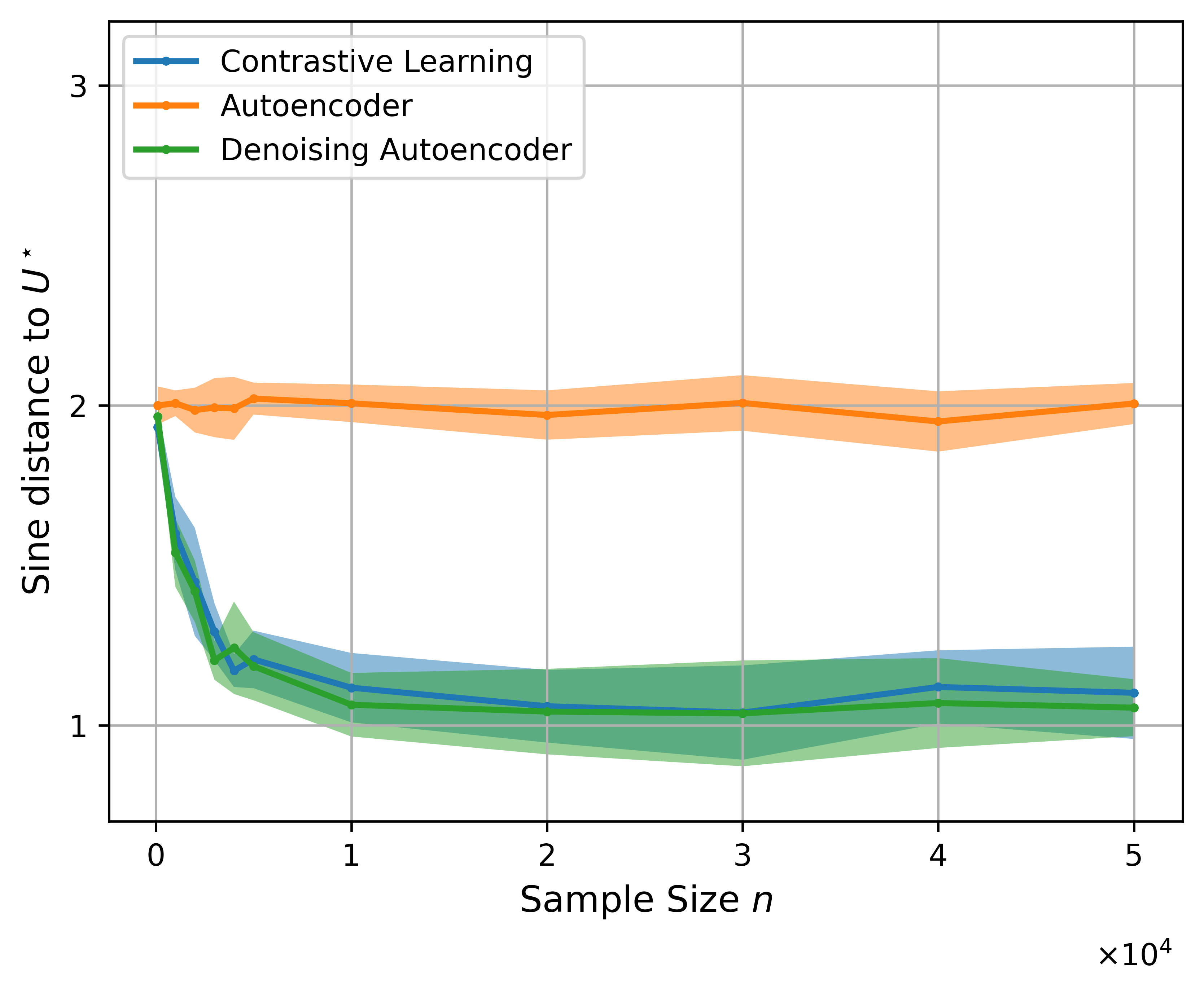}
    \caption{Comparison of denoising autoencoders, autoencoders, and contrastive learning on synthetic datasets. \textbf{Left Column:} The vertical axes indicate the downstream regression error. We subtract the regression error of the ground truth features to measure the excess error.
    \textbf{Top Row:} 
    Comparison of in-domain downstream task performance of  autoencoders, contrastive learning, and denoising autoencoders the dimension $d$. The sample size $n$ is set as $n=20000$. \textbf{Bottom Row:} Comparison of in-domain downstream task performance of autoencoders, contrastive learning, and denoising autoencoders the dimension $n$. The dimension $d$ is set as $d=40$. \vspace{-1.5em}}
	\label{fig: MAE error}
\end{figure}

\section{Omitted proofs for Section \ref{sec: labeled data}}

\subsection{Proofs for Section \ref{sec: classification}}
In this section, we will provide the proof of a generalized version of Theorem \ref{thm: SupCon} to cover the imbalanced setting, the statement and the detailed proof can be found in Theorem \ref{thm: general supcon}.

In the main body, we assume the unlabeled data and labeled data are both balanced for the sake of clarity and simplicity. Now we allow them to be imbalanced and provide a more general analysis. Suppose we have $n$ unlabeled data $X=[x_1,\cdots,x_n]\in\mathbb{R}^{d\times n}$ and $n_k$ labeled data $X_k=[x_k^1,\cdots,x_{k}^{n_k}]\in\mathbb{R}^{d\times n_k}$ for class $k$, the contrastive learning task can be formulated as:
\begin{equation}
	\label{sup contrastive task ap}
	\min_{W\in\mathbb{R}^{r\times d}}\mathcal{L}(W):=\min_{W\in\mathbb{R}^{r\times d}}\mathcal{L}_{\text{SelfCon}}(W)+\mathcal{L}_{\text{SupCon}}(W;\alpha).
\end{equation}
In addition, we write a generalized version of the supervised contrastive loss function to cover the imbalanced cases:
\begin{equation}
    \mathcal{L}_{\text{SupCon}}(W;\alpha)=-\frac{1}{r+1}\sum_{k=1}^{r+1}\frac{\alpha_k}{n_k}\sum_{i=1}^{n_k}[\sum_{j\neq i}\frac{\langle Wx_i^k,Wx_j^k\rangle}{n_k-1}-\frac{\sum_{j=1}^n\sum_{s\neq k}\langle Wx_i^k,Wx_j^s \rangle}{\sum_{s\neq k}n_s}]+\frac{\lambda}{2}\|WW^\top\|_F^2,
\end{equation}
where $\alpha_k>0$ is the weight for supervised loss of class $k$. Again we first provide a theorem to give the optimal solution to the contrastive learning problem.
\begin{theorem}
	\label{thm: supervised ap}
	The optimal solution of the supervised contrastive learning problem (\ref{sup contrastive task ap}) is given by :
	\begin{equation*}
		W_{\CL} = C\qty(\sum_{i=1}^{r}u_i\sigma_i v_i^\top)^\top,
	\end{equation*}
	where $C>0$ is a positive constant, $\sigma_i$ is the $i$-th largest eigenvalue of the following matrix:
	\begin{align*}
		&\frac{1}{4n}\qty(\Delta(XX^\top)-\frac{1}{n-1}X(1_n 1_n^\top-I_n)X^\top)\\
		&+\frac{1}{r+1}\sum_{k=1}^{r+1}\frac{\alpha_k}{n_k}\qty[\frac{1}{n_k-1}X_k(1_{n_k}1_{n_k}^\top-I_{n_k})X_k^\top-\frac{1}{\sum_{t\neq k}n_t}X_k1_{k}1_{s}^\top X_s^\top],
	\end{align*}
	$u_i$ is the corresponding eigenvector and $V=[v_1,\cdots,v_r]\in\mathbb{R}^{r\times r}$ can be any orthonormal matrix.
\end{theorem}
\begin{proof}
	Under this setting, combined with the result obtained in Corollary \ref{prop: diagonal contrast}, the contrastive loss can be rewritten as:
	\begin{equation*}
		\begin{aligned}
			\mathcal{L}(W) =& \frac{\lambda}{2}\|WW^\top\|_F^2-\frac{1}{2n}\tr(\qty(\frac{1}{2}\Delta(XX^\top)-\frac{1}{2(n-1)}X(1_n 1_n^\top-I_n)X^\top)W^\top W)\\
			&-\frac{1}{r+1}\sum_{k=1}^{r+1}\alpha_k\frac{1}{n_k}\sum_{i=1}^{n_k}\qty[\frac{1}{n_k-1}\sum_{j\neq i}\langle Wx_i^k,Wx_j^k \rangle-\frac{1}{\sum_{t\neq k}n_t}\sum_{s\neq k}\sum_{j=1}^{n_s}\langle  Wx_i^k,Wx_j^s \rangle].
		\end{aligned}
	\end{equation*}
	Then we deal with the last term independently, note that:
	\begin{equation*}
		\begin{aligned}
			&\sum_{i=1}^{n_k}\qty[\frac{1}{n_k-1}\sum_{j\neq i}\langle Wx_i^k,Wx_j^k \rangle-\frac{1}{\sum_{t\neq k}n_t}\sum_{s\neq k}\sum_{j=1}^{n_s}\langle Wx_i^k,Wx_j^s \rangle]\\
			=&\frac{1}{n_k-1}\sum_{i=1}^{n_k}\sum_{j\neq i}\langle Wx_i^k,Wx_j^k \rangle-\frac{1}{\sum_{t\neq k}n_t}\sum_{i=1}^{n_k}\sum_{s\neq k}\sum_{j=1}^{n_s}\langle Wx_i^k,Wx_j^s \rangle\\
			=&\frac{1}{n_k-1}\tr(X_k(1_{n_k}1_{n_k}^\top-I_{n_k})X_k^\top W^\top W)-\frac{1}{\sum_{t\neq k}n_t}\sum_{s\neq k}\tr(X_k1_{k}1_{s}^\top X_s^\top W^\top W).
		\end{aligned}
	\end{equation*} 
	Thus we have:
	\begin{equation*}
		\begin{aligned}
			\mathcal{L}(W) =& \frac{\lambda}{2}\|WW^\top\|_F^2-\frac{1}{4n}\tr((\Delta(XX^\top)-\frac{1}{n-1}X(1_n 1_n^\top-I_n)X^\top)W^\top W)\\
			&-\frac{1}{r+1}\sum_{k=1}^{r+1}\frac{\alpha_k}{n_k}[\frac{1}{n_k-1}\tr(X_k(1_{n_k}1_{n_k}^\top-I_{n_k})X_k^\top W^\top W)\\
			&-\frac{1}{\sum_{t\neq k}n_t}\sum_{s\neq k}\tr(X_k1_{k}1_{s}^\top X_{} W^\top W)].
		\end{aligned}
	\end{equation*}
	Then by a similar argument as in the proof of Proposition \ref{prop: augment}, we can conclude that the optimal solution $W_{\CL}$ must satisfy the desired conditions.
\end{proof}
With the optimal solution obtained in Theorem \ref{thm: supervised ap}, we can provide a generalized version of Theorem \ref{thm: SupCon} to cover the imbalance cases.
\begin{theorem}[Generalized version of Theorem \ref{thm: SupCon}]
\label{thm: general supcon}
If Assumptions \ref{asm: regular}-\ref{asm: incoherent} hold, $n> d\gg r$ and let $W_{\CL}$ be any solution that minimizes the supervised contrastive learning problem in Equation \eqref{sup contrastive task ap}, and denote its singular value decomposition as $W_{\CL}=(U_{\CL}\Sigma_{\CL}V_{\CL}^\top)^\top$, then we have
	\begin{equation*}
		\begin{aligned}
			\mathbb{E}\|\sin\Theta(U_{\CL},U)\|_F\lesssim&\frac{\nu^2}{\lambda_r(T)}\biggl(\frac{r^{3/2}}{d}\log d+\sqrt{\frac{dr}{n}}\\
			&+\frac{1}{r+1}\sum_{k=1}^{r+1}\alpha_k\biggl[\sum_{s\neq k}\frac{\sqrt{n_sd}}{\sum_{t\neq k}n_t}(\sqrt{\frac{d}{n_k}}+\sqrt{r})+\sqrt{\frac{dr}{n_k}}\biggr]\biggr),
		\end{aligned}
	\end{equation*}
	where $T\triangleq  \frac{1}{4}\sum_{k=1}^{r+1}p_i\mu^k\mu^{k\top}+\frac{1}{r+1}\sum_{k=1}^{r+1}\alpha_k(\mu^k\mu^{k\top}-\sum_{s\neq k}\frac{n_s}{\sum_{t\neq k}n_t}\frac{1}{2}(\mu^k\mu^{s\top}+\mu^s\mu^{k\top}))$.
\end{theorem}
\begin{proof}[Proof of Theorem \ref{thm: general supcon}]\label{proof: general supcon}
	For labeled data $X=[x_1,\cdots,x_n]$, we write it to be $X=M+E$, where $M=[\mu_1,\cdots,\mu_n]$ and $E=[\xi_1,\cdots,\xi_n]$ are two matrices consisting of class mean and random noise. To be more specific, if $x_i$ subject to the $k$-th cluster, then $\mu_i=\mu^k$ and $\xi_i\sim\mathcal{N}(0,\Sigma^k)$. Since the data is randomly drawn from each class, $\mu_i$ follows the multinomial distribution over $\mu^1,\cdots,\mu^r$ with probability $p_1,\cdots,p_{r+1}$. Thus $\mu_i$ follows a subgaussian distribution with covariance matrix $N=\sum_{k=1}^{r+1}p_k\mu^k\mu^{k\top}$. \\
	
	As shown in Theorem \ref{thm: supervised ap}, the optimal solution of contrastive learning is equivalent to PCA of the following matrix:
	\begin{equation*}
		\begin{aligned}
			\hat{T}\triangleq&\frac{1}{4n}(\Delta(XX^\top)-\frac{1}{n-1}X(1_n 1_n^\top-I_n)X^\top)\\
			&+\frac{1}{r+1}\sum_{k=1}^{r+1}\frac{\alpha_k}{n_k}[\frac{1}{n_k-1}X_k(1_{n_k}1_{n_k}^\top-I_{n_k})X_k^\top\\
			&-\frac{1}{\sum_{t\neq k}n_t}\sum_{s\neq k}\frac{1}{2}(X_k1_{k}1_{s}^\top X_s^\top+X_s1_{s}1_{k}^\top X_k^\top)].
		\end{aligned}
	\end{equation*}
	Again we will deal with these terms separately,
	\begin{enumerate}
		\item For the first term, as we have discussed, $X$ can be divided into two matrices $M$ and $E$, each of them consisting of sub-gaussian columns. Again we can obtain the result as in (\ref{CL difference bound}) (the proof is totally the same): 
		\begin{equation}
			\label{sup term1}
			\mathbb{E}\|\frac{1}{n}(\Delta(XX^\top)-\frac{1}{n-1}X(1_n 1_n^\top-I_n)X^\top)-N\|_2\lesssim \nu^2(\frac{r}{d}\log d+\sqrt{\frac{r}{n}})+\sigma_{(1)}^2\sqrt{\frac{d}{n}}.
		\end{equation}
		\item For the second term, notice that:
		\begin{equation}
			\begin{aligned}
				\label{sup mat 1}
				&X_k(1_{n_k}1_{n_k}^\top-I_{n_k})X_k^\top=\sum_{i=1}^{n_k}\sum_{j\neq i}(\mu^k+\xi_i^k)(\mu^k+\xi_j^k)^\top\\
				=&n_k(n_k-1)\mu^k\mu^{k\top}+(n_k-1)\mu^k(\sum_{i=1}^{n_k}\xi_i^k)^\top+(n_k-1)(\sum_{i=1}^{n_k}\xi_i^k)\mu^{k\top}+\sum_{i=1}^{n_k}\sum_{j\neq i}\xi_i^k\xi_j^{kT},
			\end{aligned}
		\end{equation}
		and that:
		\begin{equation}
			\begin{aligned}
				\label{sup mat 2}
				&\frac{1}{\sum_{t\neq k}n_t}\sum_{s\neq k}X_k1_{k}1_{s}^\top X_s^\top=\frac{1}{\sum_{t\neq k}n_t}\sum_{s\neq k}\sum_{i=1}^{n_k}(\mu^k+\xi_i^k)\sum_{j=1}^{n_s}(\mu^s+\xi_j^s)^\top\\
				=&\frac{1}{\sum_{t\neq k}n_t}\sum_{s\neq k}[n_kn_s\mu^k\mu^{s\top}+n_k\mu^k(\sum_{j=1}^{n_s}\xi_j^s)^\top+n_s\sum_{i=1}^{n_k}\xi_i^k\mu^{s\top}+\sum_{i=1}^{n_k}\xi_i^k\sum_{j=1}^{n_s}\xi_j^{sT}].
			\end{aligned}
		\end{equation}
		Since $\xi_i^k\sim\mathcal{N}(0,\Sigma^k)$, we can conclude that:
		\begin{equation}
			\label{sup mat eq 1}
			\mathbb{E}\|\frac{1}{n_k}\sum_{i=1}^{n_k}\xi_i^k\|_2\leq\sqrt{\mathbb{E}\|\frac{1}{n_k}\sum_{i=1}^{n_k}\xi_i^k\|_2^2}=\sqrt{\frac{d}{n_k}}\sigma_{(1)}.
		\end{equation}
		Moreover, we have 
		\begin{equation}
			\label{sup mat eq 2}
			\begin{aligned}
				\frac{1}{n_k(n_k-1)}\mathbb{E}\|\sum_{i=1}^{n_k}\sum_{j\neq i}\xi_i^k\xi_j^{kT}\|_2
				&\leq\frac{1}{n_k(n_k-1)}\mathbb{E}\|E_kE_k^\top\|_2+\frac{n_k}{n_k-1}\mathbb{E}\|\bar{\xi^k}\bar{\xi^k}^\top\|_2\\
				&\lesssim\frac{d}{n_k}\sigma_{(1)}^2.
			\end{aligned}
		\end{equation}
		Take equation (\ref{sup mat eq 1}) and (\ref{sup mat eq 2}) back into (\ref{sup mat 1}) we can conclude:
		\begin{equation}
			\label{sup term2}
			\mathbb{E}\|\frac{1}{n_k(n_k-1)}X_k(1_{n_k}1_{n_k}^\top-I_{n_k})X_k^\top-\mu^k\mu^{k\top}\|_2\lesssim\sqrt{\frac{d}{n_k}}\sigma_{(1)}\sqrt{r}\nu+\frac{d}{n_k}\sigma_{(1)}^2.
		\end{equation}
		On the other hand, by equation (\ref{sup mat eq 1}) we know:
		\begin{equation}
			\label{sup mat eq 3}
			\mathbb{E}\|\frac{1}{\sum_{t\neq k}n_t}\sum_{s\neq k}\sum_{j=1}^{n_s}\xi_j^s\|_2\leq\sum_{s\neq k}\frac{n_s}{\sum_{t\neq k}n_t}\mathbb{E}\|\frac{1}{n_s}\sum_{i=1}^{n_s}\xi_i^s\|_2\lesssim\sum_{s\neq k}\frac{n_s}{\sum_{t\neq k}n_t}\sqrt{\frac{d}{n_s}}\sigma_{(1)}.
		\end{equation}
		Notice that:
		\begin{equation}
			\begin{aligned}
				\label{sup mat eq 4}
				&\mathbb{E}\|\frac{1}{\sum_{t\neq k}n_t}\frac{1}{n_k}\sum_{s\neq k}\sum_{i=1}^{n_k}\xi_i^k\sum_{j=1}^{n_s}\xi_j^{sT}\|_2\leq\mathbb{E}\|\sum_{s\neq k}\frac{n_s}{\sum_{t\neq k}n_t}\bar{\xi^k}\bar{\xi^s}^\top\|_2\\
				\leq&\sum_{s\neq k}\frac{n_s}{\sum_{t\neq k}n_t}\mathbb{E}\|\bar{\xi^k}\bar{\xi^s}^\top\|_2\lesssim\sum_{s\neq k}\frac{n_s}{\sum_{t\neq k}n_t}\frac{d}{\sqrt{n_kn_s}}\sigma_{(1)}^2.
			\end{aligned}
		\end{equation}
		Thus take equations (\ref{sup mat eq 3}) and (\ref{sup mat eq 4}) back into equation (\ref{sup mat 2}) we have:
		\begin{align}
			\label{sup term3}
			&\mathbb{E}\|\frac{1}{n_k}\frac{1}{\sum_{t\neq k}n_t}\sum_{s\neq k}X_k1_{k}1_{s}^\top X_s^\top-\sum_{s\neq k}\frac{n_s}{\sum_{t\neq k}n_t}\mu^k\mu^{s\top}\|_2\\
			&\lesssim\sum_{s\neq k}\frac{\sqrt{n_sd}}{\sum_{t\neq k}n_t}(\sqrt{\frac{d}{n_k}}\sigma_{(1)}^2+\sigma_{(1)}\sqrt{r}\nu).
		\end{align}
	\end{enumerate}
	Then combine equations (\ref{sup term1})(\ref{sup term2})(\ref{sup term3}) together, we can obtain the following result:
	\begin{equation*}
		\begin{aligned}
			&\mathbb{E}\|\hat{T}-\frac{1}{4}N-\frac{1}{r+1}\sum_{k=1}^{r+1}\alpha_k(\mu^k\mu^{k\top}-\sum_{s\neq k}\frac{n_s}{\sum_{t\neq k}n_t}\frac{1}{2}(\mu^k\mu^{s\top}+\mu^s\mu^{k\top}))\|_2\\
			\lesssim&\nu^2\qty(\frac{r}{d}\log d+\sqrt{\frac{r}{n}})+\sigma_{(1)}^2\sqrt{\frac{d}{n}}\\
			&+\frac{1}{r+1}\sum_{k=1}^{r+1}\alpha_k\qty[\sum_{s\neq k}\frac{\sqrt{n_sd}}{\sum_{t\neq k}n_t}\qty(\sqrt{\frac{d}{n_k}}\sigma_{(1)}^2+\sqrt{r}\sigma_{(1)}\nu)+\sqrt{\frac{d}{n_k}}\sigma_{(1)}\sqrt{r}\nu+\frac{d}{n_k}\sigma_{(1)}^2].
		\end{aligned}
	\end{equation*}
	Since we have assumed that $\operatorname{rank}(\sum_{k=1}^{r+1}p_k\mu^k\mu^{k\top})=r$  we can find that the top-$r$ eigenspace of matrix:
	\begin{equation*}
		T= \frac{1}{4}\sum_{k=1}^{r+1}p_i\mu^k\mu^{k\top}+\frac{1}{r+1}\sum_{k=1}^{r+1}\alpha_k\qty(\mu^k\mu^{k\top}-\sum_{s\neq k}\frac{n_s}{\sum_{t\neq k}n_t}\frac{1}{2}(\mu^k\mu^{s\top}+\mu^s\mu_{k\top}))
	\end{equation*}
	is spanned by $U^\star$, then apply Lemma \ref{Lemma: DK} again we have:
	\begin{equation*}
		\begin{aligned}
			&\mathbb{E}\|\sin\Theta(U_{\SCL},U)\|_F\leq\frac{2\sqrt{r}\mathbb{E}\|\hat{N}-N\|_2}{\lambda_r(N)}\\
			\lesssim& \frac{\sqrt{r}}{\lambda_r(T)}\Biggl[\nu^2\qty(\frac{r}{d}\log d+\sqrt{\frac{r}{n}})+\sigma_{(1)}^2\sqrt{\frac{d}{n}}\\
			&+\frac{1}{r+1}\sum_{k=1}^{r+1}\alpha_k\Biggl[\sum_{s\neq k}\frac{\sqrt{n_sd}}{\sum_{t\neq k}n_t}\qty(\sqrt{\frac{d}{n_k}}\sigma_{(1)}^2+\sqrt{r}\sigma_{(1)}\nu)+\sqrt{\frac{d}{n_k}}\sqrt{r}\sigma_{(1)}\nu+\frac{d}{n_k}\sigma_{(1)}^2\Biggr]\Biggr]\\
			\lesssim&\frac{\nu^2}{\lambda_r(T)}\qty(\frac{r^{3/2}}{d}\log d+\sqrt{\frac{dr}{n}}+\frac{1}{r+1}\sum_{k=1}^{r+1}\alpha_k\qty[\sum_{s\neq k}\frac{\sqrt{n_sd}}{\sum_{t\neq k}n_t}\qty(\sqrt{\frac{d}{n_k}}+\sqrt{r})+\sqrt{\frac{dr}{n_k}}]).
		\end{aligned}
	\end{equation*}
\end{proof}
Now we use this result to derive Theorem~\ref{thm: SupCon}.  Since $\|\mu^k\|=O(\sqrt{r}\nu)$ and $\sum_{k=1}^{r+1}p_k\mu^k=0$, approximately we have $\frac{\nu^2}{\lambda_r(N)}\approx \frac{1}{\min_{k\in [r]} [1+\alpha_k]}$. Although we can not obtain the closed-form eigenvalue in general, in a special case, where $\alpha=\alpha_1=\cdots=\alpha_{r+1}$, $m=n_1=n_2=\cdots=n_{r+1}$ and $\frac{1}{r+1}=p_1=p_2=\cdots=p_{r+1}$, it is easy to find that:
\begin{equation*}
    \sum_{s\neq k}\frac{1}{2}(\mu^k\mu^{s\top}+\mu^s\mu^{k\top})=-\mu^k\mu^{k\top},
\end{equation*}
which further implies that:
\begin{equation*}
    T=\frac{1}{4}\sum_{k=1}^{r+1}p_k\mu^k\mu^{k\top}+\frac{1}{r+1}\sum_{k=1}^{r+1}\alpha(1+\frac{1}{r})\mu^k\mu^{k\top},\quad \lambda_r(T)=[\frac{1}{4}+\alpha(1+\frac{1}{r})]\lambda(N).
\end{equation*}
and we can obtain the result in Theorem \ref{thm: SupCon}.
\subsection{Proofs for Section \ref{sec: transfer}}
\label{sec: Information filtering proof}
In this section, we will provide the proof of the generalized version of Theorems \ref{thm: transfer t<r} and \ref{thm: transfer t>r} to cover the imbalanced setting, the statement and detailed proof can be found in Theorems \ref{thm: transfer t<r ap} and \ref{thm: transfer t>r ap}. With the two generalized theorems proven, Theorems \ref{thm: transfer t<r}, \ref{thm: downstream risk transfer t<r}, \ref{thm: transfer t>r}, \ref{thm: downstream risk transfer t>r} holds immediately.

First, we prove a useful lemma to illustrate that the supervised loss function only yields estimation along a 1-dimensional space. Consider a single source task, where the data $x=U^\star z+\xi$ is generated by the spiked covariance model and the label is generated by
$$
y = \langle w^\star,z\rangle/\nu
$$
suppose we have collect $n$ labeled data from this task, denote the data as $X=[x_1,x_2,\cdots,x_n]\in\mathbb{R}^{d\times n}$ and the label $y=[y_1,y_2,\cdots,y_n]\in\mathbb{R}^n$, then we have the following result.
\begin{lemma}
	\label{Thm: SCL}
	Under the conditions similar to Theorem \ref{thm: recover CL}, we can find an event $A$ such that $\mathbb{P}(A^C)=O(\sqrt{d/n})$ and:
	\begin{equation}
	\label{SCL eq}
		\mathbb{E}\qty[\norm{\frac{1}{(n-1)^2} XHyy^\top HX^\top-\nu^2U^\star w^\star w^{\star \top} U^{\star \top}}_F\1\{A\}] \lesssim \sqrt{\frac{d}{n}}\sigma_{(1)}\nu.
	\end{equation}
\end{lemma}
The proof strategy is to estimate the difference between the two rank-1 matrices via bounding the difference of the corresponding vector component. We first provide a simple lemma to illustrate the technique:
\begin{lemma}
	\label{Lemma: alphabeta}
	Suppose $\alpha,\beta\in\mathbb{R}^d$ are two vectors, then we have:
	\begin{equation*}
		\|\alpha\alpha^\top-\beta\beta^\top\|_F\leq \sqrt{2}(\|\alpha\|_2+\|\beta\|_2)\|\alpha-\beta\|_2.
	\end{equation*}
\end{lemma}
\begin{proof}
	Denote $\alpha=(\alpha_1,\cdots,\alpha_d),\beta=(\beta_1,\cdots,\beta_d)$, then we have:
	\begin{equation*}
		\begin{aligned}
			&\|\alpha\alpha^\top-\beta\beta^\top\|_F^2\leq\sum_{i=1}^{d}\sum_{j=1}^{d}|\alpha_i\alpha_j-\beta_i\beta_j|^2\leq2\sum_{i=1}^{d}\sum_{j=1}^{d}|\alpha_i\alpha_j-\alpha_i\beta_j|^2+|\alpha_i\beta_j-\beta_i\beta_j|^2\\
			\leq&2\sum_{i=1}^{d}\sum_{j=1}^{d}|\alpha_i|^2|\alpha_j-\beta_j|^2+|\beta_j|^2|\alpha_i-\beta_i|^2\leq2(\|\alpha\|_2^2+\|\beta\|_2^2)\|\alpha-\beta\|_2^2\\
			\leq&2(\|\alpha\|_2+\|\beta\|_2)^2\|\alpha-\beta\|_2^2.
		\end{aligned}
	\end{equation*}
	Take square root on both sides we can finish the proof.
\end{proof}
Now we can prove the Lemma \ref{Thm: SCL}.
\begin{proof}[Proof of Lemma \ref{Thm: SCL}]
	Clearly, we have:
	\begin{equation}
	\label{eq: replace n}
	\begin{aligned}
	    &\|\frac{1}{(n-1)^2} XHyy^\top HX^\top-\nu^2U^\star w^\star w^{\star \top} U^{\star \top}\|_F\\
	    \leq&\frac{n^2}{(n-1)^2}\|\frac{1}{n^2} XHyy^\top HX^\top -\nu^2U^\star w^\star w^{\star \top } U^{\star \top }\|_F+\frac{2n+1}{(n-1)^2}\|\nu^2U^\star w^\star w^{\star \top } U^{\star \top }\|_F\\
	    \lesssim&\|\frac{1}{n^2} XHyy^\top HX^\top -\nu^2U^\star w^\star w^{\star \top } U^{\star \top }\|_F+\frac{r}{n}\nu^2,
	\end{aligned}
	\end{equation}
	thus we can replace the $\frac{1}{(n-1)^2}$ with $\frac{1}{n}$ in equation (\ref{SCL eq}) and conclude the proof. Denote $\hat{N}\triangleq\frac{1}{n^2} XHyy^\top HX^\top$, note that both of $\hat{N}$ and $Uw^\star w^{\star \top} U^\top$ are rank-1 matrices. We first bound the difference between $\frac{1}{n}XHy$ and $Uw^\star$:
	\begin{equation}
		\begin{aligned}
			\label{SCL: eq}
			\|\frac{1}{n}XHy-\nu U^\star w^\star\|=&\|\frac{1}{n\nu}(U^\star Z+E)HZ^\top w^{\star}-\nu U^\star w^\star\|\\
			\leq&\|\frac{1}{n\nu}(U^\star Z+E)HZ^\top -\nu U^\star \|_2\\
			\leq& \frac{1}{\nu}(\|\frac{1}{n}U^\star ZZ^\top -\nu^2U^\star \|_2+\frac{1}{n}\|EZ^\top \|_2+\frac{1}{n}\|U^\star Z\bar{Z}^\top \|_2+\frac{1}{n}\|E\bar{Z}^\top \|_2).
		\end{aligned}
	\end{equation}
	We deal with the four terms in (\ref{SCL: eq}) separately:
	\begin{enumerate}
		\item For the first term, apply Lemma \ref{Lemma: bound ZZT} we have:
		\begin{equation}
			\label{SCL term1}
			\mathbb{E}\|\frac{1}{n}U^\star ZZ^\top -\nu^2U^\star \|_2\leq\mathbb{E}\|\frac{1}{n}ZZ^\top -\nu^2I_r\|_2\leq\qty(\frac{r}{n}+\sqrt{\frac{r}{n}})\nu^2.
		\end{equation}
		\item For the second term, apply Lemma \ref{Lemma: bound EV} twice we have:
		\begin{equation}
			\begin{aligned}
				\label{SCL term2}
				\frac{1}{n}\mathbb{E}\|EZ^\top \|_2=&\frac{1}{n}\mathbb{E}_{Z}[\mathbb{E}_{E}[\|EZ^\top \|_2|Z]]\\
				\lesssim& \frac{1}{n}\mathbb{E}_Z[\|Z\|_2(\sigma_{\text{sum}}+r^{1/4}\sqrt{\sigma_{\text{sum}}\sigma_{(1)}}+\sqrt{r}\sigma_{(1)})]\\
				\lesssim& \frac{1}{n}\mathbb{E}_Z[\|Z\|_2]\sqrt{d}\sigma_{(1)}\\
				\lesssim&\frac{1}{n}\sqrt{d}\sigma_{(1)}(r^{1/2}\nu+(nr)^{1/4}\nu+n^{1/2}\nu)\\
				\lesssim&\frac{\sqrt{d}}{\sqrt{n}}\sigma_{(1)}\nu.
			\end{aligned}
		\end{equation} 
		\item For the third term and fourth term, from equation (\ref{PCA step1 term4}) we know:
		\begin{equation}
			\label{SCL term3}
			\mathbb{E}\frac{1}{n}\|U^\star Z\bar{Z}^\top \|_2+\mathbb{E}\frac{1}{n}\|E\bar{Z}^\top\|_2\leq\mathbb{E}\|\bar{z}\bar{z}^\top\|_2+\mathbb{E}\|\bar{\xi}\bar{z}^\top\|_2\leq\frac{r}{n}\nu^2+\sqrt{\frac{d}{n}}\nu\sigma_{(1)}.
		\end{equation}
	\end{enumerate}
	Combine these three equations (\ref{SCL term1})(\ref{SCL term2})(\ref{SCL term3}) together we have:
	\begin{equation}
		\label{SCL: bound}
		\mathbb{E}\|\frac{1}{n}XHy-\nu U^\star w^\star\|\lesssim \sqrt{\frac{d}{n}}\sigma_{(1)}.
	\end{equation}
	With equation (\ref{SCL: bound}), we can now turn to the difference between $\hat{N}$ and $Uw^\star w^{\star \top} U^\top$. By Lemma \ref{Lemma: alphabeta} we know that:
	\begin{equation*}
		\|\hat{N}-\nu^2U^\star w^\star w^{\star \top} U^{\star \top}\|_F\lesssim(\|\frac{1}{n}XHy\|+\|\nu U^\star w^\star\|)\|\frac{1}{n}XHy-\nu U^\star w^\star|\|.
	\end{equation*}
	Using Markov's inequality, we can conclude from (\ref{SCL: bound}) that:
	\begin{equation*}
		\mathbb{P}(\|\frac{1}{n}XHy-\nu U^\star w^\star\|\geq \nu)\leq\frac{\mathbb{E}\|\frac{1}{n}XHy-\nu U^\star w^\star\|}{\nu}\lesssim\sqrt{\frac{d}{n}}.
	\end{equation*}
	Then denote $A=\{\omega: \|\frac{1}{n}XHy-\nu^2U^\star w^\star\|_2< \nu\}$ we have:
	\begin{equation*}
		\begin{aligned}
			\mathbb{E}\|\hat{N}-\nu^2U^\star w^\star w^{\star \top} U^{\star\top}\|_F\1\{A\}\lesssim&\mathbb{E}(\|\frac{1}{n}XHy\|+\|\nu U^\star w^\star\|)\|\frac{1}{n}XHy-\nu U^\star w^\star|\|\1\{A\}\\
			\lesssim&\nu(\mathbb{E}\|\frac{1}{n}XHy-\nu U^\star w^\star\|)\lesssim\sqrt{\frac{d}{n}}\sigma_{(1)}\nu.
		\end{aligned}
	\end{equation*}
	which finished the proof.
\end{proof}
In the main body, we assume the number of labeled data and the ratio of the loss function is both balanced. Now we will provide a more general result to cover the imbalance occasions.
Formally, suppose we have $n$ unlabeled data $X=[x_1,\cdots,x_n]\in\mathbb{R}^{d\times n}$ and $n_i$ labeled data $\mathcal{S}_i$ $X_i=[x_i^1,\cdots,x_i^{n_i}],y_i=[y_i^1,\cdots,y_i^{n_1}],\forall i=1,\cdots T$ for source task , we learn the linear representation via joint optimization:
\begin{equation}
	\label{trans contrastive task ap}
	\min_{W\in\mathbb{R}^{r\times d}}\mathcal{L}(W):=\min_{W\in\mathbb{R}^{r\times d}}\mathcal{L}_{\text{SelfCon}}(W)-\sum_{t=1}^{T}\alpha_i\operatorname{HSIC}(\hat{X}^t,y^t;W),
\end{equation}
To investigate its feature recovery ability, we first give the following result.
\begin{theorem}
	\label{thm: trans ap}
	For the optimization problem (\ref{trans contrastive task ap}), if we apply augmented pairs generation in Definition \ref{pair: augmented} with random masking augmentation \ref{aug: random masking} for unlabeled data, then the optimal solution is given by:
	\begin{equation*}
		W_{\CL} = C\qty(\sum_{i=1}^{r}u_i\sigma_i v_i^\top)^\top,
	\end{equation*}
	where $C>0$ is a constant, $\sigma_i$ is the $i$-th largest eigenvalue of the following matrix:
	\begin{equation*}
		\frac{1}{4n}\qty(\Delta(XX^\top)-\frac{1}{n-1}X(1_n 1_n^\top-I_n)X^\top)+\sum_{i=1}^T\frac{\alpha_i}{(n_i-1)^2}X_iH_{n_i}y_iy_i^\top H_{n_i}X_i^\top),
	\end{equation*}
	$u_i$ is the corresponding eigenvector, $V=[v_1,\cdots,v_r]\in\mathbb{R}^{r\times r}$ can be any orthogonal matrix and $H_{n_i}=I_{n_i}-\frac{1}{n_i}1_{n_i}1_{n_i}^\top$ is the centering matrix.
\end{theorem}
\begin{proof}
	Under this setting, combined with the result obtained in Corollary \ref{prop: diagonal contrast}, the loss function can be rewritten as:
	\begin{equation*}
		\begin{aligned}
			\mathcal{L}(W) =& \frac{\lambda}{2}\|WW^\top\|_F^2-\frac{1}{2n}\tr(\qty(\frac{1}{2}\Delta(XX^\top)-\frac{1}{2(n-1)}X(1_n 1_n^\top-I_n)X^\top)W^\top W)\\
			&-\sum_{t=1}^{T}\alpha_i\frac{1}{(n_i-1)^2}\tr(X_i^\top W^\top WX_iHy_iy_i^\top H)\\
			=&\frac{\lambda}{2}\biggl\|WW^\top-\frac{1}{4n\lambda}\qty(\Delta(XX^\top)-\frac{1}{n-1}X(1_n 1_n^\top-I_n)X^\top)\\
			&-\sum_{i=1}^T\frac{\alpha_i}{\lambda(n_i-1)^2}X_iH_{n_i}y_iy_i^\top H_{n_i}X_i^\top)\biggr\|_F^2\\
			&-\frac{\lambda}{2}\biggl\|\frac{1}{4n\lambda}\qty(\Delta(XX^\top)-\frac{1}{n-1}X(1_n 1_n^\top-I_n)X^\top)\\
			&+\sum_{i=1}^T\frac{\alpha_i}{\lambda(n_i-1)^2}X_iH_{n_i}y_iy_i^\top H_{n_i}X_i^\top\biggr\|_F^2.
		\end{aligned}
	\end{equation*}
	Then by a similar argument as in the proof of Proposition \ref{prop: augment}, we can conclude that the optimal solution $W_{\CL}$ must satisfy the desired conditions.
\end{proof}
Then we can give the proofs of Theorem \ref{thm: transfer t<r} and Theorem \ref{thm: transfer t>r} under our generalized setting, one can easily obtain those under balanced setting by simply setting $\alpha=\alpha_1=\cdots=\alpha_T$ and $m=n_1=\cdots=n_T$, which is consistent with Theorem \ref{thm: transfer t<r} and Theorem \ref{thm: transfer t>r} in the main body. 
\begin{theorem}[Generalized version of Theorem \ref{thm: transfer t<r}]
	\label{thm: transfer t<r ap}
	In the regression setting where $y^t = \langle w_t, z \rangle/\nu$ , suppose Assumptions \ref{asm: regular}-\ref{asm: incoherent} hold for spiked covariance model (\ref{model: spiked covariance}) and $n> d\gg r$, if we further assume that $T<r$ and $w_t$'s are orthogonal to each other, and let $W^{\CL}$ be any solution that optimizes the problem in Equation \eqref{trans contrastive task ap}, and denote its singular value decomposition as $W_{\CL}=(U_{\CL}\Sigma_{\CL}V_{\CL}^\top)^\top$, then we have:
	
	\begin{equation*}
		\begin{aligned}
			\mathbb{E}\|\sin(\Theta(U_{\CL},U^\star ))\|_F\lesssim&\qty(\frac{\sqrt{r-T}}{\min_{i\in [T]}\{\alpha_i,1\}}+\frac{\sqrt{T}}{\min_{i\in[T]}\alpha_i})\qty(\frac{r}{d}\log d+\sqrt{\frac{d}{n}})\\
			&+\sum_{i=1}^T\qty(\sqrt{r-T}\frac{\alpha_i+\min_{i\in [T]}\{\alpha_i,1\}}{\min_{i\in [T]}\{\alpha_i,1\}}+\sqrt{T}\frac{\alpha_i+\min_{i\in[T]}\alpha_i}{\min_{i\in[T]}\alpha_i})\sqrt{\frac{d}{n_i}}.
		\end{aligned}
	\end{equation*}
\end{theorem} 
\begin{proof}[Proof of Theorem \ref{thm: transfer t<r ap}]\label{proof: transfer t<r ap}
	As shown in Theorem \ref{thm: trans ap}, optimizing loss function (\ref{trans contrastive task ap}) is equivalent to find the top-$r$ eigenspace of matrix
	\begin{equation*}
		\frac{1}{4n}\qty(\Delta(XX^\top)-\frac{1}{n-1}X(1_n 1_n^\top-I_n)X^\top)+\sum_{i=1}^T \frac{\alpha_i}{(n_i-1)^2}X_iH_{n_i}y_iy_i^\top H_{n_i}X_i^\top.
	\end{equation*}
	Again denote $\hat{M}_2 \triangleq \frac{1}{n}(\Delta(XX^\top)-\frac{1}{n-1}X(1_n 1_n^\top-I_n)X^\top)$ and $\hat{N}_i\triangleq\frac{1}{(n_i-1)^2}X_iHy_iy_i^\top HX_i^\top$.
	By equation (\ref{CL difference bound}) we know that:
	\begin{equation*}
		\mathbb{E}\|\hat{M}_2-M\|_2\lesssim \nu^2\qty(\frac{r}{d}\log d+\sqrt{\frac{r}{n}}+\frac{r}{n})+\sigma_{(1)}^2\qty(\sqrt{\frac{d}{n}}+\frac{d}{n})+\sigma_{(1)}\nu\sqrt{\frac{d}{n}}.
	\end{equation*}
	By Theorem \ref{Thm: SCL} we know that for each task $\mathcal{S}_i$, we can find an event $A_i$ such that $\mathbb{P}(A_i) =O(\sqrt{\frac{d}{n}})$:
	\begin{equation*}
		\mathbb{E}\|\hat{N}_i-\nu^2U^\star w_i w_i^{\top} U^{\star\top}\|_F\1\{A_i\}\lesssim \sqrt{\frac{d}{n_i}}\sigma_{(1)}\nu.
	\end{equation*}
	The target matrix is $N = \nu^2U^\star U^{\star \top}+\sum_{i=1}^{T}\alpha_i\nu^2U^\star w_i w_i^{ T} U^{\star \top}$, and we can obtain the upper bound for the difference between $N$ and $\hat{N}$:
	\begin{equation}
		\label{eq: transfer upper bound}
		\begin{aligned}
			&\mathbb{E}\|\hat{N}-N\|_2\1\{\cap_{i=1}^T A_i\}\leq\frac{1}{4}\mathbb{E}\|\hat{M}_2-M\|_2+\sum_{i=1}^{T}\alpha_i\mathbb{E}\|\hat{N}_i-\nu^2U^\star w_i w_i^{\top} U^{\star \top}\|_F\1\{A_i\}\\
			\lesssim&\nu^2(\frac{r}{d}\log d+\sqrt{\frac{r}{n}}+\frac{r}{n})+\sigma_{(1)}^2\qty(\sqrt{\frac{d}{n}}+\frac{d}{n})+\sigma_{(1)}\nu\sqrt{\frac{d}{n}}+\sum_{i=1}^{T}\qty[\alpha_i\sqrt{\frac{d}{n_i}}\sigma_{(1)}\nu].
		\end{aligned}
	\end{equation}
	We divide the top-$r$ eigenspace $U_{\CL}$ of $W_{\CL}W_{\CL}^\top$ into two parts: the top-$T$ eigenspace $U_{\CL}^{(1)}$ and top-$(T+1)$ to top-$r$ eigenspace $U_{\CL}^{(2)}$. Similarly, we also divide the top-$r$ eigenspace $U^\star$ of $N$ into two parts: $U^{\star(1)}$ and $U^{\star(2)}$.
	Then applying Lemma \ref{Lemma: DK} we can bound the sine distance for each part: on the one hand,
	\begin{equation*}
		\begin{aligned}
			&\mathbb{E}\|\sin(\Theta(U_{\CL}^{(1)},U^{\star(1)}))\|_F\\
			=&\mathbb{E}\|\sin(\Theta(U_{\CL}^{(1)},U^{\star(1)}))\|_F\1\{\cap_{i=1}^T A_i\}+\mathbb{E}\|\sin(\Theta(U_{\CL}^{(1)},U^{\star(1)}))\|_F\1\{\cup_{i=1}^T A_i^C\}\\
			\leq&\frac{\sqrt{T}\mathbb{E}\|\hat{N}-N\|_2\1\{\cap_{i=1}^T A_i\}}{\lambda_{(T)}(N)-\lambda_{(T+1)}(N)}+\sqrt{T}\mathbb{P}(\cup_{i=1}^T A_i^C)\\
			\lesssim&\frac{\sqrt{T}}{\min_{i\in[T]}\alpha_i\nu^2}\qty(\nu^2\frac{r}{d}\log d+\sigma_{(1)}^2\sqrt{\frac{d}{n}}+\sum_{i=1}^{T}\alpha_i\sqrt{\frac{d}{n_i}}\sigma_{(1)}\nu )+\sqrt{T}\sum_{i=1}^T\sqrt{\frac{d}{n_i}}\\
			\lesssim&\frac{\sqrt{T}}{\min_{i\in[T]}\alpha_i}\qty(\frac{r}{d}\log d+\sqrt{\frac{d}{n}})+\sqrt{T}\sum_{i=1}^{T}\frac{\alpha_i+\min_{i\in[T]}\alpha_i}{\min_{i\in[T]}\alpha_i}\sqrt{\frac{d}{n_i}}.
		\end{aligned}
	\end{equation*}
	On the other hand,
	\begin{equation*}
		\begin{aligned}
			&\mathbb{E}\|\sin(\Theta(U_{\CL}^{(2)},U^{\star(2)}))\|_F\\
			&=\mathbb{E}\|\sin(\Theta(U_{\CL}^{(2)},U^{\star(2)}))\|_F\1\{\cap_{i=1}^T A_i\}+\mathbb{E}\|\sin(\Theta(U_{\CL}^{(2)},U^{\star(2)}))\|_F\1\{\cup_{i=1}^T A_i^C\}\\
			\leq&\frac{\sqrt{r-T}\mathbb{E}\|\hat{N}-N\|_2\1\{\cap_{i=1}^T A_i\}}{\min\{\lambda_{(T)}(N)-\lambda_{(T+1)}(N),\lambda_{(r)}(N)\}}+\sqrt{r-T}\mathbb{P}(\cup_{i=1}^T A_i^C)\\
			\lesssim&\frac{\sqrt{r-T}}{\min_{i\in [T]}\{\alpha_i,1\}\nu^2}\qty(\nu^2\frac{r}{d}\log d+\sigma_{(1)}^2\sqrt{\frac{d}{n}}+\sum_{i=1}^{T}\alpha_i\sqrt{\frac{d}{n_i}}\sigma_{(1)}\nu )+\sqrt{r-T}\sum_{i=1}^T\sqrt{\frac{d}{n_i}}\\
			\lesssim&\frac{\sqrt{r-T}}{\min_{i\in [T]}\{\alpha_i,1\}}\qty(\frac{r}{d}\log d+\sqrt{\frac{d}{n}})+\sqrt{r-T}\sum_{i=1}^T\qty(\frac{\alpha_i}{\min_{i\in [T]}\{\alpha_i,1\}}+1)\sqrt{\frac{d}{n_i}}.
		\end{aligned}
	\end{equation*}
	Note that:
	\begin{equation*}
		\begin{aligned}
			&\|\sin(\Theta(U_{\CL},U^\star))\|_F^2\\
			&= r-\|U_{\CL}^\top U^\star\|_F^2\\
			&\leq r-\|U_{\CL}^{(1)\top}U^{\star(1)}\|_F^2-\|U_{\CL}^{(2)T}U^{\star(2)}\|_F^2\\
			&\leq T-\|U_{\CL}^{(1)\top}U^{\star(1)}\|_F^2+(r-T)-\|U_{\CL}^{(2)\top}U^{\star(2)}\|_F^2\\
			&\leq \|\sin\Theta(U_{\CL}^{(1)},U^{\star(1)})\|_F^2+\|\sin\Theta(U_{\CL}^{(1)},U^{\star(1)})\|_F^2,
		\end{aligned}
	\end{equation*}
	and the sine distance has trivial upper bounds:
	\begin{equation*}
	    \|\sin \Theta(U_{\CL}^{(1)},U^{\star(1)})\|_F^2\leq T,\quad \|\sin \Theta(U_{\CL}^{(2)},U^{\star(2)})\|_F^2\leq r-T
	\end{equation*}
	Thus we can conclude:
	\begin{equation*}
		\begin{aligned}
			&\mathbb{E}\|\sin(\Theta(U_{\CL},U^\star))\|_F\\
			&\leq\mathbb{E}\|\sin(\Theta(U_{\CL}^{(1)},U^{\star(1)}))\|_F+\mathbb{E}\|\sin(\Theta(U_{\CL}^{(2)},U^{\star(2)}))\|_F\\
			&\lesssim\qty(\frac{\sqrt{r-T}}{\min_{i\in [T]}\{\alpha_i,1\}}\qty(\frac{r}{d}\log d+\sqrt{\frac{d}{n}})+\sum_{i=1}^T\sqrt{r-T}\frac{\alpha_i+\min_{i\in [T]}\{\alpha_i,1\}}{\min_{i\in [T]}\{\alpha_i,1\}}\sqrt{\frac{d}{n_i}})\wedge\sqrt{r-T}\\
			&\quad+\qty(\frac{\sqrt{T}}{\min_{i\in[T]}\alpha_i}\qty(\frac{r}{d}\log d+\sqrt{\frac{d}{n}})
			+\sum_{i=1}^T\sqrt{T}\frac{\alpha_i+\min_{i\in[T]}\alpha_i}{\min_{i\in[T]}\alpha_i}\sqrt{\frac{d}{n_i}})\wedge\sqrt{T}.
		\end{aligned}
	\end{equation*}
\end{proof}

\begin{theorem}[Restatement of Theorem \ref{thm: downstream risk transfer t<r}]\label{thm: downstream risk transfer t<r ap}
    Suppose the conditions in Theorem \ref{thm: transfer t<r} hold. Then, 
    \begin{align}
		&\mathbb{E}_{\mathcal{D}}[\inf_{w\in\mathbb{R}^r} \mathbb{E}_{\mathcal{E}}[\ell_r(\delta_{W_{\CL}, w})]-\inf_{w\in\mathbb{R}^r}\mathbb{E}_{\mathcal{E}}[\ell_r(\delta_{U^{\star \top}, w})]\\
		&\quad\lesssim \sqrt{r-T} \left(\frac{r\log d}{d}+\sqrt{\frac{d}{n}}+\alpha T\sqrt{\frac{ d}{m}} \wedge 1 \right) +\sqrt{T}\left(\frac{r\log d}{\alpha d}+\frac{1}{\alpha}\sqrt{\frac{d}{n}}+T\sqrt{\frac{d}{m}}\right).
	\end{align}
\end{theorem}

\begin{proof}[Proof of Theorem \ref{thm: downstream risk transfer t<r ap}]\label{proof: downstream risk transfer t<r ap}
    Theorem \ref{thm: downstream risk transfer t<r} follows directly from Lemma \ref{lem: excess risk} and Theorem \ref{thm: transfer t<r}.
\end{proof}

\begin{theorem}[Generalized version of Theorem \ref{thm: transfer t>r}]
	\label{thm: transfer t>r ap}
	In the regression setting where $y^t = \langle w_t, z \rangle/\nu$ , suppose Assumptions \ref{asm: regular}-\ref{asm: incoherent} hold for spiked covariance model (\ref{model: spiked covariance}) and $n> d\gg r$, if we further assume that $T\geq r$ and $\sum_{i=1}^{T}\alpha_iw_iw_i^\top$ is full rank, suppose $W^{\CL}$ is the optimal solution of optimization problem Equation \eqref{trans contrastive task ap}, and denote its singular value decomposition as $W_{\CL}=(U_{\CL}\Sigma_{\CL}V_{\CL}^\top)^\top$, then we have:
	\begin{equation*}
		\begin{aligned}
			\mathbb{E}\|\sin(\Theta(U_{\CL},U^\star))\|_F\lesssim&\frac{\sqrt{r}}{1+\nu^2\lambda_{(r)}(\sum_{i=1}^{T}\alpha_i w_iw_i^\top)}\qty(\frac{r}{d}\log d
			+\sqrt{\frac{d}{n}})\\&+\sqrt{r}\sum_{i=1}^T\qty(\frac{\alpha_i}{1+\nu^2\lambda_{(r)}(\sum_{i=1}^{T}\alpha_i w_iw_i^\top)}+1)\sqrt{\frac{d}{n_i}}.
		\end{aligned}
	\end{equation*}
\end{theorem}

\begin{proof}[Proof of Theorem \ref{thm: transfer t>r ap}]\label{proof: transfer t>r ap}
	The proof strategy is similar to that of Theorem \ref{thm: transfer t<r}, here the difference is that each direction can be accurately estimated by the labeled data and we do not need to separate the eigenspace. Directly applying Lemma \ref{Lemma: DK} and equation (\ref{eq: transfer upper bound}) we have:
	\begin{equation*}
		\begin{aligned}
			&\mathbb{E}\|\sin(\Theta(U_{\CL},U^\star))\|_F\\
			&=\mathbb{E}\|\sin(\Theta(U_{\CL},U^\star))\|_F\1\{\cap_{i=1}^T A_i\}+\mathbb{E}\|\sin(\Theta(U_{\CL},U^\star))\|_F\1\{\cup_{i=1}^T A_i^C\}\\
			&\lesssim\frac{\sqrt{r}\mathbb{E}\|\hat{N}-N\|_2\1\{\cap_{i=1}^T A_i\}}{\lambda_{(r)}(N)}+\sqrt{r}\mathbb{P}(\cup_{i=1}^T A_i^C)\\
			&\lesssim\frac{\sqrt{r}}{\nu^2+\nu^2\lambda_{(r)}(\sum_{i=1}^{T}\alpha_i w_iw_i^\top)}\qty(\nu^2\frac{r}{d}\log d+\sigma_{(1)}^2\sqrt{\frac{d}{n}}+\sum_{i=1}^{T}\alpha_i\sqrt{\frac{d}{n_i}}\sigma_{(1)}\nu)+\sqrt{r}\sum_{i=1}^T\sqrt{\frac{d}{n_i}}\\
			&\lesssim\frac{\sqrt{r}}{1+\lambda_{(r)}(\sum_{i=1}^{T}\alpha_i w_iw_i^\top)}\qty(\frac{r}{d}\log d+\sqrt{\frac{d}{n}})+\sqrt{r}\sum_{i=1}^T\qty(\frac{\alpha_i}{1+\lambda_{(r)}(\sum_{i=1}^{T}\alpha_i w_iw_i^\top)}+1)\sqrt{\frac{d}{n_i}}.
		\end{aligned}
	\end{equation*}
\end{proof}

\begin{theorem}[Restatement of Theorem \ref{thm: downstream risk transfer t>r}]\label{thm: downstream risk transfer t>r ap}
    Suppose the conditions in Theorem \ref{thm: transfer t>r} hold. Then, 
    \begin{align}
		\mathbb{E}_{\mathcal{D}}[\inf_{w\in\mathbb{R}^r} \mathbb{E}_{\mathcal{E}}[\ell_r(\delta_{W_{\CL}, w})]-\inf_{w\in\mathbb{R}^r}\mathbb{E}_{\mathcal{E}}[\ell_r(\delta_{U^{\star \top}, w})] \lesssim \frac{\sqrt{r}}{\alpha+1}\qty(\frac{r}{d}\log d+\sqrt{\frac{d}{n}})+T\sqrt{\frac{dr}{m}}.
	\end{align}
\end{theorem}

\begin{proof}[Proof of Theorem \ref{thm: downstream risk transfer t>r}]\label{proof: downstream risk transfer t>r}
    Theorem \ref{thm: downstream risk transfer t>r} follows directly from Lemma \ref{lem: excess risk} and Theorem \ref{thm: transfer t>r}.
\end{proof}

Now we move to a binary classification setting, where labels $y$ are generated by $y=\sign(\langle w^\star, z\rangle)$ instead of $y=\langle w^\star, z\rangle/\nu$ in previous regression setting. We first give the corresponding generalized version of Theorem \ref{thm: transfer t<r classification} and Theorem \ref{thm: transfer t>r classification} to cover the general imbalanced settings.
\begin{theorem}[Generalized version of Theorem \ref{thm: transfer t<r classification}]
	\label{thm: transfer t<r ap classification}
	In the classification setting where $y^t = \sign(\langle w_t, z \rangle)$ , suppose Assumptions \ref{asm: regular}-\ref{asm: incoherent} hold for spiked covariance model (\ref{model: spiked covariance}), $z$ follows a Gaussian distribution, and $n> d\gg r$, if we further assume that $T<r$ and $w_t$'s are orthogonal to each other, and let $W^{\CL}$ be any solution that optimizes the problem in Equation \eqref{trans contrastive task ap}, and denote its singular value decomposition as $W_{\CL}=(U_{\CL}\Sigma_{\CL}V_{\CL}^\top)^\top$, then we have:
	
	\begin{equation*}
		\begin{aligned}
			\mathbb{E}\|\sin(\Theta(U_{\CL},U^\star ))\|_F\lesssim&\qty(\frac{\sqrt{r-T}}{\min_{i\in [T]}\{\alpha_i,1\}}+\frac{\sqrt{T}}{\min_{i\in[T]}\alpha_i})\qty(\frac{r}{d}\log d+\sqrt{\frac{d}{n}})\\
			&+\sum_{i=1}^T\qty(\sqrt{r-T}\frac{\alpha_i+\min_{i\in [T]}\{\alpha_i,1\}}{\min_{i\in [T]}\{\alpha_i,1\}}+\sqrt{T}\frac{\alpha_i+\min_{i\in[T]}\alpha_i}{\min_{i\in[T]}\alpha_i})\sqrt{\frac{d}{n_i}}.
		\end{aligned}
	\end{equation*}
\end{theorem} 
\begin{theorem}[Generalized version of Theorem \ref{thm: transfer t>r classification}]
	\label{thm: transfer t>r ap classification}
	In the classification setting where $y^t = \sign(\langle w_t, z \rangle)$ , suppose Assumptions \ref{asm: regular}-\ref{asm: incoherent} hold for spiked covariance model (\ref{model: spiked covariance}), $z$ follows a Gaussian distribution, and $n> d\gg r$, if we further assume that $T\geq r$ and $\sum_{i=1}^{T}\alpha_iw_iw_i^\top$ is full rank, suppose $W^{\CL}$ is the optimal solution of optimization problem Equation \eqref{trans contrastive task ap}, and denote its singular value decomposition as $W_{\CL}=(U_{\CL}\Sigma_{\CL}V_{\CL}^\top)^\top$, then we have:
	\begin{equation*}
		\begin{aligned}
			\mathbb{E}\|\sin(\Theta(U_{\CL},U^\star))\|_F\lesssim&\frac{\sqrt{r}}{1+\nu^2\lambda_{(r)}(\sum_{i=1}^{T}\alpha_i w_iw_i^\top)}\qty(\frac{r}{d}\log d
			+\sqrt{\frac{d}{n}})\\&+\sqrt{r}\sum_{i=1}^T\qty(\frac{\alpha_i}{1+\nu^2\lambda_{(r)}(\sum_{i=1}^{T}\alpha_i w_iw_i^\top)}+1)\sqrt{\frac{d}{n_i}}.
		\end{aligned}
	\end{equation*}
\end{theorem}
The only difference between these two settings is the distribution of labels $y$. Thus to prove Theorem \ref{thm: transfer t<r ap classification} and Theorem \ref{thm: transfer t>r ap classification}, we only need to recover Lemma \ref{Thm: SCL} in this binary classification setting. Since in the classification setting the labels are discrete and could be harder to analyze, we make the Gaussian assumption on $z$ to make problems mathematically tractable in these two Theorems. 

\begin{lemma}[Classification version of Lemma \ref{Thm: SCL}]
    \label{lem: label concentration}
    In the binary classification setting, under the conditions similar to Theorem \ref{thm: recover CL} and assume $z$ in the spiked covariance model \eqref{model: spiked covariance} follows a Gaussian distribution, we can find an event $A$ such that $\mathbb{P}(A^C)=O(\sqrt{d/n})$ and:
	\begin{equation}
	\label{SCL eq 2}
		\mathbb{E}\qty[\norm{\frac{1}{(n-1)^2} XHyy^\top HX^\top- \frac{2\nu^2}{\pi} U^\star w^\star w^{\star \top} U^{\star \top}}_F\1\{A\}] \lesssim \sqrt{\frac{d}{n}}\sigma_{(1)}\nu.
	\end{equation}
\end{lemma}
\begin{proof}
    Again, by \eqref{eq: replace n} we have:
	\begin{equation*}
	\begin{aligned}
	    &\|\frac{1}{(n-1)^2} XHyy^\top HX^\top-\frac{2\nu^2}{\pi}U^\star w^\star w^{\star \top} U^{\star \top}\|_F\\
	    \lesssim&\|\frac{1}{n^2} XHyy^\top HX^\top -\frac{2\nu^2}{\pi}U^\star w^\star w^{\star \top } U^{\star \top }\|_F+\frac{r}{n}\nu^2,
	\end{aligned}
	\end{equation*}
	thus we can replace the $\frac{1}{(n-1)^2}$ with $\frac{1}{n}$ in equation (\ref{SCL eq 2}) and conclude the proof. Denote $\hat{N}\triangleq\frac{1}{n^2} XHyy^\top HX^\top$, note that both of $\hat{N}$ and $Uw^\star w^{\star \top} U^\top$ are rank-1 matrices. We first bound the difference between $\frac{1}{n}XHy$ and $\sqrt{\frac{2\nu^2}{\pi}}Uw^\star$:
	\begin{equation}
		\begin{aligned}
			\label{eq: vec diff}
			\|\frac{1}{n}XHy-\sqrt{\frac{2\nu^2}{\pi}} U^\star w^\star\|=&\|\frac{1}{n}(U^\star Z+E)Hy-\sqrt{\frac{2\nu^2}{\pi}} U^\star w^\star\|\\
			\leq& \|\frac{1}{n}U^\star Zy -\sqrt{\frac{2\nu^2}{\pi}} U^\star w^\star \|+\frac{1}{n}\|Ey \|+\frac{1}{n}\|U^\star Z\bar{y} \|+\frac{1}{n}\|E\bar{y}\|.
		\end{aligned}
	\end{equation}
	We deal with the four terms in (\ref{eq: vec diff}) separately:
	\begin{enumerate}
		\item For the first term, note that:$
		    \frac{1}{n}Zy=\frac{1}{n}\sum_{i=1}^n z_i\sign(z_i^\top w^\star)$
		and $z_i\sim \mathcal{N}(0,\nu^2I_r)$, thus $z_i\sign(z_i^\top w^\star)$ follows a folded Gaussian distribution, which is a reflection of standard Gaussian distribution along the normal plane of $w^\star$, thus
		\begin{equation}
		\label{vecdiff term1}
		\begin{aligned}
		    \mathbb{E}\|\frac{1}{n}U^\star Zy -\sqrt{\frac{2\nu^2}{\pi}} U^\star w^\star \|&\leq\mathbb{E}\|\frac{1}{n} Zy -\sqrt{\frac{2\nu^2}{\pi}} w^\star \|\leq\sqrt{\mathbb{E}\|\frac{1}{n} Zy -\sqrt{\frac{2\nu^2}{\pi}} w^\star \|^2}\\
		    &\leq\sqrt{\frac{r}{n}}\nu
		\end{aligned}
		\end{equation}
		\item For the second term, note that $y$ and $E$ are independent and $|y|=1$ almost surely
		\begin{equation}
			\begin{aligned}
				\label{vecdiff term2}
				\frac{1}{n}\mathbb{E}\|Ey \|=\frac{1}{n}\mathbb{E}\|\sum_{i=1}^n\xi_i\|\leq\frac{1}{n}\sqrt{\mathbb{E}\|\sum_{i=1}^n\xi_i\|^2}\lesssim\sqrt{\frac{d}{n}}\sigma_{(1)}
			\end{aligned}
		\end{equation} 
		\item For the third term and fourth terms, we have:
		\begin{equation}
			\label{vecdiff term3}
			\mathbb{E}\frac{1}{n}\|U^\star Z\bar{y} \|+\mathbb{E}\frac{1}{n}\|E\bar{y}\|\leq\mathbb{E}\frac{1}{n}\|\sum_{i=1}^n z_i \|+\mathbb{E}\frac{1}{n}\|\sum_{i=1}^n \xi_i\|\lesssim\sqrt{\frac{r}{n}}\nu+\sqrt{\frac{d}{n}}\sigma_{(1)}.
		\end{equation}
	\end{enumerate}
	Combine these three equations (\ref{vecdiff term1})(\ref{vecdiff term2})(\ref{vecdiff term3}) together we have:
	\begin{equation}
		\label{vecdiff bound}
		\mathbb{E}\|\frac{1}{n}XHy-\sqrt{\frac{2\nu^2}{\pi}} U^\star w^\star\|\lesssim \sqrt{\frac{d}{n}}\sigma_{(1)}.
	\end{equation}
	With equation (\ref{vecdiff bound}), we can now turn to the difference between $\hat{N}$ and $\frac{2\nu^2}{\pi}Uw^\star w^{\star \top} U^\top$. By Lemma \ref{Lemma: alphabeta} we know that:
	\begin{equation*}
		\|\hat{N}-\frac{2\nu^2}{\pi}U^\star w^\star w^{\star \top} U^{\star \top}\|_F\lesssim(\|\frac{1}{n}XHy\|+\|\sqrt{\frac{2\nu^2}{\pi}} U^\star w^\star\|)\|\frac{1}{n}XHy-\sqrt{\frac{2\nu^2}{\pi}} U^\star w^\star\|.
	\end{equation*}
	Using Markov's inequality, we can conclude from (\ref{vecdiff bound}) that:
	\begin{equation*}
		\mathbb{P}(\|\frac{1}{n}XHy-\sqrt{\frac{2\nu^2}{\pi}} U^\star w^\star\|\geq \nu)\leq\frac{\mathbb{E}\|\frac{1}{n}XHy-\sqrt{\frac{2\nu^2}{\pi}} U^\star w^\star\|}{\nu}\lesssim\sqrt{\frac{d}{n}}.
	\end{equation*}
	Then denote $A=\{\omega: \|\frac{1}{n}XHy-\sqrt{\frac{2\nu^2}{\pi}}U^\star w^\star\|_2< \nu\}$ we have:
	\begin{equation*}
		\begin{aligned}
			\mathbb{E}\|\hat{N}-\frac{2\nu^2}{\pi}U^\star w^\star w^{\star \top} U^{\star\top}\|_F\1\{A\}\lesssim&\mathbb{E}(\|\frac{1}{n}XHy\|+\|\sqrt{\frac{2\nu^2}{\pi}} U^\star w^\star\|)\|\frac{1}{n}XHy-\sqrt{\frac{2\nu^2}{\pi}} U^\star w^\star|\|\1\{A\}\\
			\lesssim&\nu\mathbb{E}\|\frac{1}{n}XHy-\sqrt{\frac{2\nu^2}{\pi}} U^\star w^\star\|\lesssim\sqrt{\frac{d}{n}}\sigma_{(1)}\nu.
		\end{aligned}
	\end{equation*}
	which finished the proof.
\end{proof}
With Lemma \ref{lem: label concentration} established, it is straightforward to obtain the same results as in Theorem~\ref{thm: transfer t<r}, Theorem~\ref{thm: downstream risk transfer t<r}, Theorem~\ref{thm: transfer t>r} and Theorem~\ref{thm: downstream risk transfer t>r} for this binary classification setting.

\section{Useful lemmas}
In this section, we list some of the main techniques that have been used in the proof of the main results.
\begin{lemma} [Theorem 2 in \citet{yu2015useful}]
	\label{Lemma: DK}
	Let $\Sigma, \hat{\Sigma} \in \mathbb{R}^{p \times p}$ be symmetric, with eigenvalues $\lambda_{1} \geq \ldots \geq \lambda_{p}$ and $\hat{\lambda}_{1} \geq$
	$\ldots \geq \hat{\lambda}_{p}$ respectively. Fix $1 \leq r \leq s \leq p$ and assume that $\min \left(\lambda_{r-1}-\lambda_{r}, \lambda_{s}-\lambda_{s+1}\right)>0$
	where $\lambda_{0}:=\infty$ and $\lambda_{p+1}:=-\infty .$ Let $d:=s-r+1$, and let $V=\left(v_{r}, v_{r+1}, \ldots, v_{s}\right) \in \mathbb{R}^{p \times d}$
	and $\hat{V}=\left(\hat{v}_{r}, \hat{v}_{r+1}, \ldots, \hat{v}_{s}\right) \in \mathbb{R}^{p \times d}$ have orthonormal columns satisfying $\Sigma v_{j}=\lambda_{j} v_{j}$ and $\hat{\Sigma} \hat{v}_{j}=\hat{\lambda}_{j} \hat{v}_{j}$ for $j=r, r+1, \ldots, s .$ Then
	$$
	\|\sin \Theta(\hat{V}, V)\|_{\mathrm{F}} \leq \frac{2 \min \left(d^{1 / 2}\|\hat{\Sigma}-\Sigma\|_{\mathrm{2}},\|\hat{\Sigma}-\Sigma\|_{\mathrm{F}}\right)}{\min \left(\lambda_{r-1}-\lambda_{r}, \lambda_{s}-\lambda_{s+1}\right)}.
	$$
	Moreover, there exists an orthogonal matrix $\hat{O} \in \mathbb{R}^{d \times d}$ such that
	$$
	\|\hat{V} \hat{O}-V\|_{\mathrm{F}} \leq \frac{2^{3 / 2} \min \left(d^{1 / 2}\|\hat{\Sigma}-\Sigma\|_{\mathrm{2}},\|\hat{\Sigma}-\Sigma\|_{\mathrm{F}}\right)}{\min \left(\lambda_{r-1}-\lambda_{r}, \lambda_{s}-\lambda_{s+1}\right)}.
	$$
\end{lemma}
\begin{lemma} [Lemma 2 in \citet{zhang2018heteroskedastic}]
	\label{Lemma: bound EV}
	Assume that $E \in \mathbb{R}^{p_{1} \times p_{2}}$ has independent sub-Gaussian entries, $\operatorname{Var}\left(E_{i j}\right)=$ $\sigma_{i j}^{2}, \sigma_{C}^{2}=\max _{j} \sum_{i} \sigma_{i j}^{2}, \sigma_{R}^{2}=\max _{i} \sum_{j} \sigma_{i j}^{2}, \sigma_{(1)}^{2}=\max _{i, j} \sigma_{i j}^{2}.$ Assume that
	$$
	\left\|E_{i j} / \sigma_{i j}\right\|_{\psi_{2}}:=\max _{q \geq 1} q^{-1 / 2}\left\{\mathbb{E}\left(\left|E_{i j}\right| / \sigma_{i j}\right)^{q}\right\}^{1 / q} \leq \kappa.
	$$
	Let $V \in \mathbb{O}_{p_{2}, r}$ be a fixed orthogonal matrix. Then
	
	$$
	\mathbb{P}\left(\|E V\|_2 \geq 2\left(\sigma_{C}+x\right)\right) \leq 2 \exp \left(5 r-\min \left\{\frac{x^{4}}{\kappa^{4} \sigma_{(1)}^{2} \sigma_{C}^{2}}, \frac{x^{2}}{\kappa^{2} \sigma_{(1)}^{2}}\right\}\right),
	$$
	$$
	\mathbb{E}\|E V\|_2 \lesssim \sigma_{C}+\kappa r^{1 / 4}\left(\sigma_{(1)} \sigma_{C}\right)^{1 / 2}+\kappa r^{1 / 2}\sigma_{(1)}.
	$$
\end{lemma}
\begin{lemma} [Theorem 6 in \citet{cai2020non}]
	\label{Lemma: bound ZZT}
	Suppose $Z$ is a $p_{1}$-by- $p_{2}$ random matrix with independent mean-zero sub-Gaussian entries. If there exist $\sigma_{1}, \ldots, \sigma_{p} \geq 0$ such that $\left\|Z_{i j} / \sigma_{i}\right\|_{\psi_{2}} \leq C_{K}$ for constant $C_{K}>0$, then
	$$
	\mathbb{E}\left\|Z Z^\top-\mathbb{E} Z Z^\top\right\|_2 \lesssim \sum_{i} \sigma_{i}^{2}+\sqrt{p_{2} \sum_{i} \sigma_{i}^{2}} \cdot \max _{i} \sigma_{i}.
	$$
\end{lemma}

\begin{lemma}[The Eckart-Young-Mirsky Theorem \citep{Eckart1936TheAO}]
	\label{Lemma: best rank r approximation}
	Suppose that $A=U \Sigma V^{T}$ is the singular value decomposition of $A$. Then the best rank- $k$ approximation of the matrix $A$ w.r.t the Frobenius norm, $\|\cdot\|_{F}$, is given by
	$$
	A_{k}=\sum_{i=1}^{k} \sigma_{i} u_{i} v_{i}^{T}.
	$$
	that is, for any matrix $B$ of rank at most k
	$$
	\|A-A_k\|_F\leq\|A-B\|_F.
	$$
\end{lemma}

\vskip 0.2in
\bibliography{cite.bib}

\end{document}